\documentclass{article}

\usepackage[round]{natbib}

\usepackage{fullpage}
\usepackage{courier}
\usepackage{authblk}

\usepackage[utf8]{inputenc} 
\usepackage[T1]{fontenc}    
\usepackage{xcolor}
\definecolor{mydarkblue}{rgb}{0,0.08,0.45}
\usepackage[hidelinks,colorlinks=True,allcolors=mydarkblue]{hyperref}       

\usepackage{amsmath}
\usepackage{amssymb}
\usepackage{mathtools}
\usepackage{amsthm}

\usepackage{url}            
\usepackage{booktabs}       
\usepackage{amsfonts}       
\usepackage{amssymb}       
\usepackage{nicefrac}       
\usepackage{microtype}      
\usepackage{graphicx}
\usepackage{enumitem}
\usepackage{xspace}
\usepackage[export]{adjustbox}
\usepackage{subfig}

\usepackage{algorithm}
\usepackage{algorithmic}

\usepackage{multirow}
\usepackage{multicol}
\usepackage{floatpag}
\usepackage{color,array}

\usepackage{amsmath}
\usepackage{amssymb}
\usepackage{mathtools}
\usepackage{amsthm}

\usepackage[capitalize,noabbrev]{cleveref}

\theoremstyle{plain}
\newtheorem{theorem}{Theorem}[section]
\newtheorem{proposition}[theorem]{Proposition}
\newtheorem{lemma}[theorem]{Lemma}
\newtheorem{corollary}[theorem]{Corollary}
\theoremstyle{definition}
\newtheorem{definition}[theorem]{Definition}

\theoremstyle{remark}
\newtheorem{remark}[theorem]{Remark}

\DeclareMathOperator*{\argmax}{arg\,max}
\DeclareMathOperator*{\argmin}{arg\,min}
\DeclareMathOperator*{\mmd}{MMD}

\newcommand{\asxz}{\ \omega\text{-a.s.}}
\newcommand{\intxz}[1]{\int_{\mathcal{X} \times \mathcal{Z}} {#1}\  \omega(\mathrm{d}x \otimes \mathrm{d}z) }
\newcommand{\intx}[1]{\int_{\mathcal{X}} {#1} \omega(\mathrm{d}x) }

\newcommand{\avg}{\frac{1}{n} \sum_{i=1}^n}
\newcommand{\dd}{\mathrm{d}}

\title{Estimation Beyond Data Reweighting: \\ Kernel Method of Moments}

\author[$1$]{Heiner Kremer}
\author[$1$]{Yassine Nemmour}
\author[$1,2$]{Bernhard Sch{\"o}lkopf}
\author[$3$]{Jia-Jie Zhu}

\affil[$1$]{Max Planck Institute for Intelligent Systems, Tübingen, Germany}
\affil[$2$]{Eidgen\"ossische Technische Hochschule Zürich, Switzerland}
\affil[$3$]{Weierstrass Institute for Applied Analysis and Stochastics, Berlin, Germany}

\begin{document}
\floatpagestyle{plain}
\date{}
\maketitle

\begin{abstract}
    Moment restrictions and their conditional counterparts emerge in many areas of machine learning and statistics ranging from causal inference to reinforcement learning. 
    Estimators for these tasks, generally called \emph{methods of moments}, include the prominent \emph{generalized method of moments} (GMM) which has recently gained attention in causal inference.
    GMM is a special case of the broader family of \emph{empirical likelihood estimators} which are based on approximating a population distribution by means of minimizing a $\varphi$-divergence to an empirical distribution.
    However, the use of $\varphi$-divergences effectively limits the candidate distributions to reweightings of the data samples.
    We lift this long-standing limitation and provide a method of moments that goes beyond data reweighting. This is achieved by defining an empirical likelihood estimator based on maximum mean discrepancy which we term the \emph{kernel method of moments} (KMM). 
    We provide a variant of our estimator for conditional moment restrictions and show that it is asymptotically first-order optimal for such problems.
    Finally, we show that our method achieves competitive performance on several conditional moment restriction tasks.
\end{abstract}

\section{Introduction}
Many problems in machine learning, statistics, causal inference and economics can be formulated as (conditional) moment restrictions \citep{newey2003instrumental,angrist2008mostly}. Moment restrictions (MR) identify a parameter of interest by restricting the expectation over a so-called moment function to a fixed value. From a machine learning perspective, moment restrictions subsume empirical risk minimization since the corresponding first order conditions imply that the expectation of the gradient of the loss function vanishes.
A significantly harder problem is posed by conditional moment restrictions (CMR), which restrict the \emph{conditional} expectation of the moment function. In this case estimation effectively requires solving a continuum of unconditional moment restrictions \citep{BIERENS1982105}. 
A prominent CMR problem is instrumental variable (IV) regression \citep{newey2003instrumental}, where the expectation of the prediction residual conditioned on the instruments is required to be zero. The CMR formulation of IV regression is a powerful way to define estimators that avoid two-step procedures as, e.g., in the common two-stage least squares method \citep{angrist2008mostly}.
Other examples of CMR problems include variants of double machine learning \citep{chernozhukov2016double,chernozhukov2017double,chernozhukov2018double} and off-policy evaluation in reinforcement learning \citep{xu2021learning,chen2021instrumental,bennett2020efficient,bennett2021off}.
Perhaps the most popular approach to learning with moment restrictions is the generalized method of moments (GMM) of \citet{hansen}, which recently gained popularity in machine learning \citep{lewis2018adversarial,bennett2020variational}. 
GMM belongs to the wider family of generalized empirical likelihood (GEL) estimators of \citet{owen88,owen90,qin-lawless,smith1997alternative}. While moment restrictions are imposed with respect to a population distribution, in practice one usually only has access to an empirical sample from this distribution. GEL estimators are based on simultaneously finding the model parameters and an approximation of the population distribution by considering distributions with minimal distance to the empirical distribution for which the moment restrictions can be fulfilled exactly.
The various GEL estimators differ in the choice of $\varphi$-divergence used to define this distance. In this context, the continuous updating version of GMM \citep{hansen-finite-samples} can be interpreted as a GEL estimator with $\chi^2$-divergence.
However, the use of $\varphi$-divergences effectively restricts the set of candidate distributions to multinomial distributions on the empirical sample, i.e., reweightings of the data, which can be a crude approximation especially in the low sample regime.
In the present work, we define the first method of moments estimator that parts with this limitation by defining a GEL framework based on a fundamentally different notion of distributional distance, namely the maximum mean discrepancy (MMD). This allows us to consider arbitrary candidate distributions with support different from the empirical distribution. As in many cases the population distribution is continuous, this bears the potential to find better approximations thereof. 
In principle, our flexible framework even allows to evolve the class of candidate distributions over the course of the optimization and thus might benefit from developments in gradient flows and optimal transport. 
The practical benefit of our approach is demonstrated by competitive empirical performance. 
\paragraph{Our Contributions}
\begin{enumerate}
\vspace{-.5em}
  \setlength\itemsep{-.1em}
    \item 
    We propose the first method of moments estimator without the limitation to data reweightings by extending the GEL framework to MMD. We derive the dual problem of the resulting inner optimization problem which is a semi-infinitely constrained convex program. 
    \item
    To overcome computational challenges, we introduce entropy regularization and show that the dual of the inner problem gives rise to an unconstrained convex program, turning a semi-infinite formulation into either a soft-constraint or log-barrier setting.
    \item
    We provide the first order asymptotics and demonstrate that our estimator is asymptotically optimal for CMR estimation in the sense that it achieves the semi-parametric efficiency bound of \citet{CHAMBERLAIN1987305}. 
    \item
    We provide details on the practical implementation and empirically demonstrate state-of-the-art performance of our method on several CMR problems. 
    \item 
    We release an implementation of our method as part of a \href{https://github.com/HeinerKremer/conditional-moment-restrictions}{software package for (conditional) moment restriction estimation}.
\end{enumerate}
The remainder of the paper is structured as follows. 
Section~\ref{sec:learning-with-mom} gives an overview of method of moments estimation for conditional and unconditional moment restrictions. 
Section~\ref{sec:kel} introduces our estimator, and provides duality results as well as asymptotic properties and practical considerations. 
Section~\ref{sec:results} provides an empirical evaluation of our estimators on various conditional moment restriction tasks. Section~\ref{sec:related-work} discusses connections to related methods and Section~\ref{sec:conclusion} concludes.
\section{Background} \label{sec:learning-with-mom}
\paragraph{Method of Moments}
Let $X$ be a random variable taking values in $\mathcal{X} \subseteq \mathbb{R}^r$ distributed according to $P_0$. 
In the following we will denote the expectation with respect to a distribution $P$ by $E_{P}[\cdot]$ and drop the subscript whenever we refer to the population distribution $P_0$.
Moment restrictions identify a parameter of interest $\theta_0 \in \Theta \subseteq \mathbb{R}^p$ by restricting the expectation of a so-called moment function $\psi: \mathcal{X} \times \Theta \rightarrow \mathbb{R}^m$, such that
\begin{align*}
    E[\psi(X;\theta_0)] = 0 \in \mathbb{R}^m.
\end{align*}
In practice, the true distribution $P_0$ is generally unknown and one only has access to a sample $\{x_i \}_{i=1}^n$ with empirical distribution $\hat{P}_n = \sum_{i=1}^n \frac{1}{n} \delta_{x_i}$, where $\delta_{x_i}$ denotes a Dirac measure centered at $x_i$. 
The corresponding \emph{empirical} moment restrictions can be defined as
\begin{align*}
    E_{\hat{P}_n}[\psi(X;\theta)] = 0, \quad \theta \in \Theta.
\end{align*}
If the number of restrictions $m$ does not exceed the number of parameters $p$, these can often be solved exactly. For example, suppose we are interested in estimating the mean $\theta$ of a distribution. Then, solving the empirical moment restrictions for the moment function $\psi(X;\theta)= X-\theta$ yields the maximum likelihood estimate $\theta = \frac{1}{n} \sum_{i=1}^n x_i$. 
However, in the so-called \emph{overidentified} case with $m > p$, the system of equations is generally over-determined and the empirical moment restrictions cannot be satisfied exactly.
This is the domain of the celebrated generalized method of moments (GMM) of \citet{hansen}. Instead of trying to satisfy the moment restrictions exactly, GMM relaxes the problem into a minimization of a quadratic form,
\begin{align*}
    \theta^{\text{GMM}} \! = \! \argmin_{\theta \in \Theta} E_{\hat{P}_n}[\psi(X;\theta)]^T \! \left(\widehat{\Omega}(\tilde{\theta})\right)^{-1} \! E_{\hat{P}_n}[\psi(X;\theta)],
\end{align*}
where $\widehat{\Omega}(\tilde{\theta}) = E_{\hat{P}_n}[\psi(X;\tilde{\theta})\psi(X;\tilde{\theta})^T] \in \mathbb{R}^{m\times m}$ denotes the empirical covariance matrix evaluated at a first stage estimate $\tilde{\theta}$.
\paragraph{Empirical Likelihood Estimation}
GMM is a special case of the wider family of generalized empirical likelihood estimators \citep{owen88,owen90,qin-lawless}. In an attempt to improve the finite sample properties of GMM, alternative estimators from this family have been proposed~\cite{smith1997alternative,newey04}. GEL estimation is based on the idea that while it might not be possible to satisfy the moment restrictions with respect to the empirical distribution $\hat{P}_n$, the population distribution $P_0$, for which the moment restrictions hold at the true parameter $\theta_0$, will be in a shrinking neighbourhood of $\hat{P}_n$ as the number of samples $n$ grows. Therefore GEL seeks to find a parameter $\theta$ and a distribution $P$ for which the moment restrictions hold exactly while staying as close as possible to the empirical distribution.
For a convex function $\varphi: [0, \infty) \rightarrow (-\infty, \infty]$ define the $\varphi$-divergence between distributions $P$ and $Q$ as $D_{\varphi}(P||Q) = \int \!  \varphi \! \left( \frac{\dd P}{\dd Q} \right) \dd Q$, where $\frac{\dd P}{\dd Q}$ denotes the Radon-Nikodym derivative of $P$ with respect to $Q$. Define the \emph{profile divergence} with respect to a $\varphi$-divergence as
\begin{align}
    R(\theta) = & \inf_{P \ll \hat{P}_n} D_\varphi(P||\hat{P}_n) 
    \quad \mathrm{s.t.} \quad E_P[\psi(X;\theta)] = 0, \quad E_P[1] = 1, \label{eq:classical-profile}
\end{align}
where $P \ll \hat{P}_n$ is the set of positive measures $P$ that are absolutely continuous w.r.t.\ the empirical distribution $\hat{P}_n$. The GEL estimator then results from minimizing the profile divergence over $\theta \in \Theta$,
\begin{align*}
    \theta^{\mathrm{GEL}} = \argmin_{\theta \in \Theta} R(\theta).
\end{align*}
Due to the absolute continuity assumption, the distributions considered by GEL are reweightings of the empirical data. Being a special case of GEL, GMM therefore also implicitly corresponds to reweightings of the data as formalized by the following proposition which follows directly from the equivalence result of \citet{newey04} (Theorem~2.1).
\begin{proposition} 
    The first order optimality conditions for the continuous updating GMM estimator and the GEL estimator with $\chi^2$-divergence coincide. As the optimal distribution of the latter is given by $P^\ast = \sum_{i=1}^{n} p_i \delta_{x_i}$ for some $p \in \mathbb{R}^n$ with $\sum_{i=1}^n p_i =1 $, in consequence, GMM implicitly corresponds to a reweighting of the data.
\end{proposition}

\paragraph{Conditional Moment Restrictions}
In practice, many interesting problems can be formulated as \emph{conditional} moment restrictions, where the estimating equations are given by a conditional expectation over the moment function. Let $Z$ be an additional random variable taking values in $\mathcal{Z}$, then conditional moment restrictions take the form
\begin{align}
    E[\psi(X;\theta_0)|Z] = 0, \ P_Z\text{-a.s.}, \label{eq:conditional-mr}
\end{align}
where the restrictions need to hold almost surely (a.s.) with respect to the marginal distribution $P_Z$ over $Z$ corresponding to $P_0$.
As conditional moment restrictions are difficult to handle in practice, many proposed estimators rely on transforming them into a corresponding continuum of unconditional restrictions \citep{BIERENS1982105} of the form
\begin{align}
    E[\psi(X;\theta)^T h(Z)] = 0 \ \ \forall h \in \mathcal{H}, \label{eq:continuum-mr}
\end{align}
where the expectation is taken over the joint distribution of $X$ and $Z$ and $\mathcal{H}$ is a sufficiently rich function space. Examples of such spaces are the Hilbert space of square integrable functions or the reproducing kernel Hilbert space of a universal kernel \citep{universalkernel06}. Both the GMM and the GEL framework have been extended to conditional moment restrictions in multiple ways, building on basis function expansions of $\mathcal{H}$ \citep{Carrasco1,tripathi2003testing,ai2003efficient,chausse2012generalized,carrasco2017regularized} as well as modern machine learning models \citep{deepiv,lewis2018adversarial,bennett2020deep,kremer2022functional}.

\paragraph{Reproducing Kernel Hilbert Spaces}
A reproducing kernel Hilbert space (RKHS) $\mathcal{F}$ is a Hilbert space of functions $f: \mathcal{X} \rightarrow \mathbb{R}$ in which all point evaluation functionals are bounded. Let $\langle \cdot ,\cdot\rangle_\mathcal{F}$ denote the inner product on $\mathcal{F}$ and define the RKHS norm as the induced norm $\|f \|_\mathcal{F} = \sqrt{\langle f, f\rangle_\mathcal{F}}$.
With every RKHS one can associate a unique kernel $k: \mathcal{X} \times \mathcal{X} \rightarrow \mathbb{R}$ with the reproducing property $\langle k(x,\cdot), f\rangle_\mathcal{F} = f(x)$ for any $f \in \mathcal{F}$ and $x \in \mathcal{X}$. A kernel is called integrally strictly positive definite (ISPD) if for any $f \in \mathcal{F}$ with $0<\| f\|_2^2 < \infty$ we have $\int_{\mathcal{X}} f(x) k\left(x, x^{\prime}\right) f\left(x^{\prime}\right) \mathrm{d} x \mathrm{~d} x^{\prime}>0$.
Let $\mathcal{P}$ denote a space of probability distributions, then we define the kernel mean embedding of $P \in \mathcal{P}$ as $\mu_P = E_{P}[k(X,\cdot)] \in \mathcal{F}$, which has the property that $\langle \mu_P, f \rangle_\mathcal{F} = E_P[f(X)] \ \forall f \in \mathcal{F}$. 
This can be used to define a metric on a space of probability distributions $\mathcal{P}$. For $P, Q \in \mathcal{P}$ the maximum mean discrepancy (MMD) \citep{gretton2012kernel} is defined as $\operatorname{MMD}(P,Q;\mathcal{F}) := \|\mu_P  - \mu_Q\|_\mathcal{F}$. Refer to, e.g., \citet{scholkopf2002learning,berlinet2011reproducing,steinwart2008support} for comprehensive introductions to kernel methods for machine learning.

\section{Kernel Method of Moments} \label{sec:kel}
In this section we derive the KMM estimator for unconditional and conditional moment restrictions and explore its properties.
We first derive an exact MMD-based GEL estimator that leads to a difficult semi-infinitely constrained optimization problem for the \emph{MMD profile} $R(\theta)$.
We show that an entropy regularized version of our estimator leads to an unconstrained convex dual program which can be readily solved with, e.g., first order optimization methods. We show that our estimator is consistent and optimal for (conditional) moment restriction problems in the sense that it achieves the lowest possible asymptotic variance among all estimators based solely on the CMR. Finally we provide details on the computational procedure. All proofs are deferred to Section~\ref{sec:proofs}.
\subsection{Our Method}
Our goal is to define a profile function $R(\theta)$ based on maximum mean discrepancy instead of $\varphi$-divergences, such that the KMM estimator can be obtained as $\hat{\theta} = \argmin_{\theta \in \Theta}R(\theta)$. Let $\psi: \mathcal{X} \times \Theta \rightarrow \mathbb{R}^m$ denote a moment function and let $\mathcal{P}$ denote the space of positive measures. Further let $\mathcal{F}$ be an RKHS corresponding to a universal kernel.
Then we can define the \emph{MMD-profile} as 
\begin{align}
    R(\theta) =  \inf_{P \in \mathcal{P}} \frac{1}{2} \mmd(P, \hat P_n ;\mathcal{F})^2 
    \quad \mathrm{s.t.}\quad E_P[\psi(X;\theta)] =0, \quad E_P[1] =1, \label{eq:primal-mmd-profile-original} 
\end{align}
where we set $R(\theta) = \infty$ whenever the optimization problem is infeasible.
The restriction to RKHS $\mathcal{F}$ corresponding to universal kernels ensures that $\operatorname{MMD}(P, Q ;\mathcal{F}) = 0$ if and only if $P = Q$ and thus ensures uniqueness of the infimum in \eqref{eq:primal-mmd-profile-original} as $n \rightarrow \infty$.
Using Lagrange duality we can derive the corresponding dual problem as formalized in the following.
\begin{theorem} \label{th:unreg-duality}
The \emph{MMD profile} \eqref{eq:primal-mmd-profile-original} has the strongly dual form
\begin{align}
    R(\theta) = &\sup_{\substack{\eta \in \mathbb{R},f \in \mathcal{F} , \\  h \in \mathbb{R}^m}}  \frac{1}{n} \sum_{i=1}^n f(x_i) + \eta -\frac{1}{2} \|f \|_\mathcal{F}^2  \label{eq:unreg-mmd-profile} \\ & \mathrm{s.t.} \quad f(x) + \eta  \leq \psi(x;\theta)^T h \quad \forall x \in \mathcal{X}.  \nonumber
\end{align}
\end{theorem}
Note that the structure and derivation of \eqref{eq:unreg-mmd-profile} resembles recent reformulation techniques for MMD-based distributionally robust optimization (DRO) by \citet{zhu2021kernel} and in fact GEL estimation can be seen as a dual problem to DRO \citep{lamRecoveringBestStatistical2019}.
While MMD enjoys many favorable properties, 
the MMD-profile~\eqref{eq:unreg-mmd-profile}, involves a (semi-)infinite constraint which is difficult to handle in practice, especially when combined with stochastic-gradient-type algorithms in machine learning.
In the following section, we will show that these limitations can be lifted by introducing entropy regularization.
\subsection{Entropy Regularization} \label{sec:regularization}
Inspired by the interior point method for convex optimization in finite-dimensions \cite{nesterov1994interior}, we define an entropy-regularized version of the MMD profile \eqref{eq:unreg-mmd-profile}. This allows us to translate the semi-infinite constraint in \eqref{eq:unreg-mmd-profile} into an additional term in the objective of an unconstrained optimization problem.
Note that entropy regularization has been used before in the context of the computation of optimal transport distances~\citep{cuturi2013sinkhorn}. Our use here is different as we do not regularize a distance computation but instead regularize the duality structure to handle the semi-infinite constraint. To the best of our knowledge entropy regularization has not been combined with MMD in this context.
For a convex function $\varphi: [0, \infty) \rightarrow (-\infty, \infty]$ define the relative entropy (or $\varphi$-divergence) between a reference distribution $\omega$ and a distribution $P$, which admits a density $p$ w.r.t.\ $\omega$, as
\begin{align}
    D_\varphi(P || \omega) = \intx{\varphi(p(x))}. \label{eq:entropy}
\end{align}
Then for any $\epsilon >0$ we define the \emph{entropy-regularized MMD profile} for a moment restriction of the form $E[\psi(X;\theta_0)]=0$ as  
\begin{align}
     R^\varphi_\epsilon(\theta) = & \inf_{P \ll \omega} \frac{1}{2}  \mmd(P,\hat{P}_n;\mathcal{F})^2 + \epsilon D_\varphi(P||\omega) \quad \mathrm{s.t.} \quad E_P[\psi(X;\theta)] =0, \quad E_P[1] =1. \label{eq:mmd-profile} 
\end{align}
In contrast to the classical $\varphi$-divergence-based profile divergence \eqref{eq:classical-profile}, the regularized MMD profile does not require $P \ll \hat{P}_n$. Instead, absolute continuity is only imposed with respect to an arbitrary (potentially continuous) reference distribution $\omega$ which can be constructed in a data-driven way (cf.\ Section~\ref{app:reference-measure}). In practice, by sampling from $\omega$, this allows us to approximate the population distribution $P_0$ in an arbitrarily fine-grained way instead of using mere reweightings of the training data as in GEL/GMM, which can be a rough approximation especially in the low data regime.
Using Lagrangian duality we can derive the dual problem of \eqref{eq:mmd-profile} as formalized in the following theorem.
\begin{theorem}[KMM Duality] \label{th:kel-duality}
The entropy-regularized MMD profile \eqref{eq:mmd-profile} has the strongly dual form
\begin{align}
    R^\varphi_\epsilon(\theta) = \sup_{\substack{\eta \in \mathbb{R},f \in \mathcal{F} , \\  h \in \mathbb{R}^m}} \frac{1}{n} \sum_{i=1}^n f(x_i) + \eta -\frac{1}{2} \|f \|_{\mathcal{F}}^2
    - \epsilon \intx{ \varphi^\ast \left( \frac{f(x) + \eta -  \psi(x;\theta)^T h }{\epsilon} \right) }. \label{eq:dual-profile} 
\end{align}
where $\varphi^\ast(t):=\sup_s \langle t, s \rangle - \varphi(s)$ denotes the convex conjugate of $\varphi$.
The optimization problem in \eqref{eq:dual-profile} is  jointly convex in the dual variables $(\eta, f, h)$.
\end{theorem}
As opposed to the unregularized version \eqref{eq:unreg-mmd-profile} the entropy-regularized MMD profile \eqref{eq:mmd-profile} is a jointly convex, unconstrained optimization problem over the dual variables. 
Finally the KMM estimator can be obtained as the minimizer of the entropy-regularized MMD profile
\begin{align*}
    \hat{\theta} = \argmin_{\theta \in \Theta} R^{\varphi}_\epsilon(\theta).
\end{align*}
\subsection{Choices of Entropy Regularizers} \label{sec:choices}
\begin{figure}
    \centering
    \includegraphics[width=.5\linewidth]{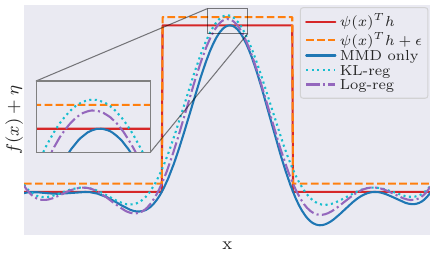}
    \caption{Effect of Entropy Regularization. The red and orange lines correspond to an exemplary function $\psi(x;\theta)^T h$ and its relaxation $\psi(x;\theta)^T h + \epsilon$. The blue line shows the strictly minorizing RKHS function resulting from enforcing the constraint in \eqref{eq:unreg-mmd-profile} exactly. The cyan and purple lines correspond to the $\varphi$-divergence regularized problem. The log-divergence works as a barrier-function which allows to violate the constraint in \eqref{eq:unreg-mmd-profile} by up to $\epsilon$. The KL-divergence yields a soft constraint by penalizing violations exponentially.}
    \label{fig:f-divergences}
\end{figure}
Different choices of $\varphi$-divergences in \eqref{eq:entropy} correspond to different relaxations of the generally intractable semi-infinite constraint in \eqref{eq:unreg-mmd-profile}.
Choosing the $\varphi$-divergence as the Kullback-Leibler divergence, i.e., $\varphi(p) = p \log(p) - p + 1$ we obtain
\begin{align*}
    R^\mathrm{KL}_\epsilon(\theta) = \sup_{\substack{\eta \in \mathbb{R},f \in \mathcal{F} , \\  h \in \mathbb{R}^m}} \frac{1}{n} \sum_{i=1}^n f(x_i) + \eta -\frac{1}{2} \|f \|_{\mathcal{F}}^2  - \epsilon \intx{ \exp \left( \frac{f(x) + \eta - \psi(x;\theta)^T h }{\epsilon} \right) }.
\end{align*}
This corresponds to relaxing the constraint in \eqref{eq:unreg-mmd-profile} into a soft version, such that violations can occur but are exponentially penalized.
Another particularly interesting example is obtained by choosing the $\varphi$-divergence to be the backward KL-divergence or Burg's entropy $\varphi(p) = - \log p +p -1$, which leads to a dual regularized MMD profile with a log-barrier
\begin{align*}
    R^{\mathrm{log}}_\epsilon(\theta) =& \sup_{\substack{\eta \in \mathbb{R},f \in \mathcal{F} , \\  h \in \mathbb{R}^m}} \frac{1}{n} \sum_{i=1}^n f(x_i) + \eta -\frac{1}{2} \|f \|_{\mathcal{F}}^2  
    + \epsilon \intx{ \log \left( 1- \frac{f(x) + \eta -  \psi(x;\theta)^T h }{\epsilon} \right) }. 
\end{align*}
Numerically, the log-barrier enforces the solution to lie in the interior of the constraint set, i.e.,
\begin{align*}
    {f(x_i) + \eta - \psi(x_i;\theta)^T h } < {\epsilon}.
\end{align*}
Therefore, this can be seen as an interior-point method for handling the infinite constraint in \eqref{eq:unreg-mmd-profile}. 
Figure~\ref{fig:f-divergences} provides a visualization of the different regularization schemes. Refer to Section~\ref{sec:appendix-entropy} for additional details.
\subsection{KMM for Conditional Moment Restrictions}
The KMM estimator can be generalized to conditional moment restrictions via a functional formulation following the approach of \citet{kremer2022functional}.

Suppose we have a sufficiently rich Hilbert space $\mathcal{H}$ of functions $h: \mathcal{Z} \rightarrow \mathbb{R}^m$, such that we can write conditional moment restrictions of the form~\eqref{eq:conditional-mr} as a continuum of unconditional restrictions~\eqref{eq:continuum-mr}. Let $\mathcal{H}^\ast$ denote the dual space of functionals $\Psi: \mathcal{H} \rightarrow \mathbb{R}$, equipped with the dual norm $\| \Psi \|_{\mathcal{H}^\ast} = \sup_{\|h \|_\mathcal{H}=1} \Psi(h)$.
Then for any $(x,z,\theta) \in \mathcal{X} \times \mathcal{Y} \times \Theta$ we define the \emph{moment functional} $\Psi(x,z;\theta) \in \mathcal{H}^\ast$ such that 
\begin{align*}
    \Psi(x,z;\theta)(h) = \psi(x;\theta)^T h(z).
\end{align*}
The continuum of moment restrictions \eqref{eq:continuum-mr} can thus be written as
\begin{align}
    E[\Psi(X,Z;\theta)] = 0 \in \mathcal{H}^\ast, \label{eq:fmr}
\end{align}
where $0 \in \mathcal{H}^\ast$ describes the functional that maps each element in $\mathcal{H}$ to zero, which is equivalent to requiring $\| E[\Psi(X,Z;\theta)]\|_{\mathcal{H}^\ast} = 0 $. 
By the Riesz representation theorem \citep{zeidler2012applied}, we can identify each element $\Psi \in \mathcal{H}^\ast$ with an element $\phi(\Psi) \in \mathcal{H}$ such that $\Psi(h) = \langle \phi(\Psi), h \rangle_\mathcal{H}$ $\forall h \in \mathcal{H}$ and $\|\Psi \|_{\mathcal{H}^\ast} = \| \phi(\Psi) \|_\mathcal{H}$ . 
%
Generalizing the KMM estimator to conditional moment restrictions then just becomes a matter of substituting 
\begin{align*}
    h \in \mathbb{R}^m &\rightarrow h \in \mathcal{H} \\
    \psi(x;\theta)^T h &\rightarrow \Psi(x,z;\theta)(h) 
\end{align*}
and adding a regularization term $-\frac{1}{2} \| h\|_\mathcal{H}^2$ for the Lagrange parameter $h \in \mathcal{H}$, which regularizes the first order conditions for $h$ as argued by \citet{kremer2022functional}.
With this at hand, we can define the functional version of the entropy-regularized MMD profile for conditional moment restrictions (refer to Section~\ref{sec:primal-cmr} for details on the duality relationship).
\begin{definition}[Functional KMM] \label{def:kmm-cmr}
Let $\mathcal{H} \subseteq L^2(\mathcal{Z},\mathbb{R}^m,P_Z)$ be a sufficiently rich Hilbert space of functions $\mathcal{Z} \rightarrow \mathbb{R}^m$ such that equivalence between \eqref{eq:conditional-mr} and \eqref{eq:continuum-mr} holds. Then the entropy-regularized MMD profile for conditional moment restrictions is given as
\begin{align}
    R^{\varphi}_{\epsilon,\lambda_n} \! (\theta) 
    =& \! \! \sup_{\substack{\eta \in \mathbb{R}, f \in \mathcal{F}, \\ h \in \mathcal{H}}} \frac{1}{n} \sum_{i=1}^n f(x_i,z_i) + \eta -\frac{1}{2} \|f \|_{\mathcal{F}}^2 - \frac{\lambda_n}{2} \|h \|_{\mathcal{H}}^2 \label{eq:functional-kgel} \\
    &- \epsilon \intxz{ \varphi^\ast \! \left( \frac{f(x,z) + \eta - \Psi(x,z;\theta)(h) 
    }{\epsilon} \right)}. \nonumber
\end{align}
\end{definition}
From the proof of Theorem~\ref{th:kel-duality} it directly follows that the optimization problem in \eqref{eq:functional-kgel} is jointly convex in the dual variables.
The space $\mathcal{H}$ can be chosen, e.g., as the RKHS of a universal integrally strictly positive definite (ISPD) kernel~\citep{JMLR:v19:16-291} to guarantee the consistency of the solution for the conditional moment restrictions~\citep{kremer2022functional}. In practice, alternative choices, e.g., neural networks which lack these theoretical guarantees have proven successful and often preferable \citep{bennett2020variational,kremer2022functional}.

\subsection{Asymptotic Properties}
Consider the regularized KMM estimator with any $\varphi$-divergence such that $\frac{d}{dt}\varphi^\ast(t)|_{t=0}=: \varphi_1^\ast(0)= 1$ and $\frac{d^2}{(dt)^2}\varphi^\ast(t)|_{t=0}=: \varphi_2^\ast(0)= 1$, which is fulfilled for, e.g., the forward and backward KL divergence.

For space reasons, here we focus on the estimator for \emph{conditional} moment restrictions by combining the theory for the functional KMM estimator (see Section~\ref{sec:KMM-FMR}) with a sufficiently rich space of locally Lipschitz functions $\mathcal{H}$. The properties of the unconditional/finite-dimensional KMM estimator are provided in Section~\ref{sec:KMMfinite}. 
\begin{theorem}[Consistency]\label{th:consistency-cmr}
    Let $\mathcal{H} \subseteq L^2(\mathcal{Z}, \mathbb{R}^m, P_Z)$ be a Hilbert space of locally Lipschitz functions which is sufficiently rich such that equivalence between \eqref{eq:conditional-mr} and \eqref{eq:continuum-mr} holds.
    Further assume that 
    a) $\theta_0 \in \Theta$ is the unique solution to $E[\psi(X;\theta)|Z] = 0 \ P_Z\text{-a.s.}$; 
    b) $\Theta \subset \mathbb{R}^p$ is compact; 
    c) $\psi(X;\theta)$ is continuous at each $\theta \in \Theta$ w.p.1; 
    d) $E[\sup_{\theta \in \Theta} \| \psi(X;\theta) \|^2_2|Z] < \infty$ w.p.1; 
    e)~$V_0(Z) := E[\psi(X;\theta_0) \psi(X;\theta_0) |Z]$ is non-singular w.p.1;
    f)~$\omega = (1-\alpha) \hat{P}_n + \alpha Q$ for $\alpha= O_p(n^{-1})$ and any distribution $Q$ such that $E_Q[\sup_{\theta \in \Theta} \| \psi(X;\theta) \|_2^2 |Z ] < \infty$ w.p.1;
    and 
    g)~$\lambda_n = O_p(n^{-\xi})$ with $0 < \xi < 1/2$. 
Then for the KMM estimator $\hat{\theta}$ we have $\hat{\theta} \overset{p}{\rightarrow} \theta_0$.

If additionally 
    h)~$\theta_0 \in \operatorname{int}(\Theta)$;
    i)~$\psi(x;\theta)$ is continuously differentiable in a neighborhood $\bar{\Theta}$ of $\theta_0$ and $E[\sup_{\theta \in \bar{\Theta}} \| \nabla_\theta \psi(X;\theta) \|^2 | Z] < \infty$ w.p.1; as well as
    j)~$\operatorname{rank}\left(E[\nabla_\theta \psi(X;\theta_0) |Z] \right) = p$ w.p.1,
we have $\|\hat{\theta} - \theta_0 \| = O_{p}(n^{-1/2})$.
\end{theorem}

\begin{remark}
    Assumption~f) implies that asymptotically the reference distribution $\omega$ is required to converge weakly to the population distribution $P_0$. However, as $Q$ can be chosen as an arbitrary (continuous) distribution, as long as $\operatorname{supp}(\hat{P}_n) \subseteq \operatorname{supp}(Q)$, the form of $\omega = (1-\alpha) \hat{P}_n + \alpha Q$ does not restrict the set of candidate distributions $P$ in \eqref{eq:functional-kgel} further than to distributions that admit a density w.r.t.\ $Q$.
\end{remark}

\begin{remark}
    A sufficiently rich function space for Theorem~\ref{th:consistency-cmr} is for example given by the RKHS of an integrally strictly positive definite universal kernel (e.g., Gaussian kernel; see Theorem~3.9 of \citet{kremer2022functional}). Moreover, based on universal approximation theorems~\citep{hornik1989multilayer}, $\mathcal{H}$ can be represented by neural networks of asymptotically growing width/depth. In this case the local Lipschitz property can be ensured by restricting the weights to a compact domain (e.g., via weight clipping).
\end{remark}
\begin{theorem}[Asymptotic Normality ]\label{th:asymptotic-normality-cmr}
Let the assumptions of Theorem~\ref{th:consistency-cmr} be satisfied. 
Then,
\begin{align*}
    \sqrt{n} (\hat{\theta} - \theta_0)  \overset{d}{\rightarrow}  N(0, \Xi_0)
\end{align*}
where 
    $$\Xi_0 = E\left[  E[\nabla_\theta \psi(X;\theta_0)|Z] \ V_0^{-1}(Z) \ E[\nabla_\theta \psi(X;\theta_0)|Z] \right]^{-1}.$$
\end{theorem}
The asymptotic variance in Theorem~\ref{th:asymptotic-normality-cmr} agrees with the semi-parametric efficiency bound of \citet{CHAMBERLAIN1987305}. This implies that the KMM estimator achieves the lowest possible asymptotic variance among any estimators based on the CMR \eqref{eq:conditional-mr} as formalized in the following corollary.
\begin{corollary}[Efficiency]\label{efficiency-cmr}
    Let the assumptions of Theorem~\ref{th:consistency-cmr} be satisfied. Then the KMM estimator $\hat{\theta}$ is efficient for $\theta_0$, i.e., it has the smallest asymptotic variance among all  estimators based solely on the conditional moment restrictions $E[\psi(X;\theta_0)|Z]= 0 \ P_Z\text{-a.s.}$.
\end{corollary}
This is a particularly strong result, setting our estimator apart from a number of recently proposed modern mini-max approaches to CMR estimation \citep{lewis2018adversarial,bennett2020deep,zhang2021maximum} and which is matched only by the kernel VMM estimator of \citet{bennett2020variational} and the more traditional sieve-based approaches of \citet{ai2003efficient} and \citet{chen2009efficient}. 
Note that while this shows that our estimator is asymptotically first-order equivalent to these methods, the finite sample properties can be vastly different.
\subsection{Computing KMM Estimators}
In the following we restrict the discussion to the conditional version of the KMM estimator. The version for unconditional MR follows directly by setting $\lambda_n=0$, $\mathcal{H}=\mathbb{R}^m$ and $\Psi(x,z;\theta)= \psi(x;\theta)$.
The functional entropy-regularized MMD profile \eqref{eq:functional-kgel} is a convex optimization problem over function-valued dual parameters $f\in\mathcal{F}$ and $h \in \mathcal{H}$ as well as $\eta \in \mathbb{R}$. The dual formulation \eqref{eq:functional-kgel} allows us to base our method on a dual functional gradient ascent algorithm.
This is in contrast to particle gradient descent methods that address the primal problem by relying on discretizing measures which is commonly used in the gradient flow literature.
Define the saddle point objective from \eqref{eq:functional-kgel} as
\begin{align*}
     \widehat{G}_{\epsilon, \lambda_n}(\theta, \eta, f, h)
     =& \frac{1}{n} \sum_{i=1}^n f(x_i, z_i) + \eta -\frac{1}{2} \|f \|_{\mathcal{F}}^2 - \frac{\lambda_n}{2} \|h \|_{\mathcal{H}}^2  \\
    &- \epsilon  \intxz{ \varphi^\ast \left( \frac{f(x, z) + \eta - \Psi(x,z;\theta)(h) }{\epsilon} \right) }.
\end{align*}
Then the KMM estimator is given as the solution to the problem
\begin{align}
    \hat{\theta} = \argmin_{\theta \in \Theta} \sup_{\beta \in \mathcal{M}} \widehat{G}_{\epsilon,\lambda_n}(\theta, \beta), \label{eq:estimation}
\end{align}
where we defined $\beta := (\eta, f, h) \in \mathcal{M}= \mathbb{R} \times \mathcal{F} \times \mathcal{H}$.
\begin{algorithm}[t]
   \caption{Gradient Descent Ascent for KMM}
   \label{alg:training}
\begin{algorithmic}
  \STATE {\bfseries Input:} empirical distribution $\hat{P}_n$, reference distribution $\omega$, hyperparameters $\epsilon$, $\lambda$, 
  batchsizes $n_1$, $n_2$
  \WHILE{ not converged}
        \STATE Sample $\{ (x_i, z_i)\}_{i=1}^{n_1} \sim \hat{P}_n $, \ \  $\{ (x_j^\omega, z_j^\omega) \}_{j=1}^{n_2} \sim \omega$
        \STATE $G \gets \frac{1}{n_1 n_2} \sum_{i=1}^{n_1} \sum_{j=1}^{n_2} G_{\epsilon, \lambda}(\theta,\beta; (x_i, z_i), (x^\omega_j, z_j^\omega))$
       \STATE $\beta \gets \operatorname{AscentStep}(\beta,  \nabla_\beta G)$
       \STATE $\theta \gets \operatorname{DescentStep}(\theta, \nabla_\theta G)$
  \ENDWHILE
  \STATE {\bfseries Output:} Parameter estimate $\theta$
\end{algorithmic}
\end{algorithm}
\paragraph{Stochastic Approximation}
In order to evaluate the KMM objective we use a stochastic approximation of the integral term in \eqref{eq:estimation}.
Let the random variables $(X,Z)$ and $(X^\omega, Z^\omega)$ be distributed according to $P_0$ and $\omega$ respectively and define the random variable 
\begin{align*}
     G_{\epsilon, \lambda_n}(\theta, \beta; (X,Z), (X^\omega, Z^\omega))
     \ =& \  f(X,Z) + \eta -\frac{1}{2} \|f \|_{\mathcal{F}}^2 - \frac{\lambda_n}{2} \|h \|_{\mathcal{H}}^2 \\
    &- \epsilon \varphi^\ast \left( \frac{f(X^\omega, Z^\omega) + \eta - \Psi(X^\omega, Z^\omega;\theta)(h) }{\epsilon} \right). 
\end{align*}
Then we can express the KMM objective as an expectation with respect to the empirical and reference distributions $\hat{P}_n$ and $\omega$ respectively, 
\begin{align*}
    \widehat{G}_{\epsilon,\lambda_n}(\theta, \beta) = E_{\hat{P}_n} E_\omega [G_{\epsilon, \lambda_n}(\theta, \beta; (X,Z), (X^\omega, Z^\omega)) ].
\end{align*}
This has the form required for mini-batch stochastic gradient descent (SGD) optimization. 
We solve problem \eqref{eq:estimation} by alternating between functional SGD steps in the dual variables  $\beta$ and SGD steps in $\theta$. Our approach is detailed in Algorithm~\ref{alg:training}.

\paragraph{Random Feature Approximation}
The supremum over $\beta$ in \eqref{eq:estimation} involves the optimization over the function space $\mathcal{F}$ which is generally intractable. To provide a scalable estimator that can be optimized with variants of stochastic gradient descent (SGD), we resort to a random Fourier feature approximation of the RKHS function $f$ \citep{NIPS2007_013a006f}.
For $\alpha \in \mathbb{R}^d$, let $f(x,z) = \alpha^T \phi(x,z)$ define the random feature approximation of $f$ with random Fourier features $\phi(x,z) \in \mathbb{R}^d$, where $d$ is the number of random features.
The KMM objective then becomes
\begin{align*}
    \widehat{G}_{\epsilon,\lambda_n}(\theta, \eta, \alpha, h) 
    =& \frac{1}{n} \sum_{i=1}^n \alpha^T \phi(x_i, z_i) + \eta -\frac{1}{2} \| \alpha\|^2_2 - \frac{\lambda_n}{2} \|h \|_{\mathcal{H}}^2  \\
    &- \epsilon \intxz{\varphi^\ast  \left( \frac{\alpha^T \phi(x,z) + \eta - \Psi(x,z;\theta)(h)}{\epsilon} \right) }.  
\end{align*}
Combined with the stochastic approximation and a scalable instrument function $h$ (e.g., neural network or RKHS function with RF approximation) this allows our method to scale to large sample sizes. 
\section{Empirical Results} \label{sec:results}
We benchmark our estimator on two different tasks against state-of-the-art estimators for conditional moment restrictions, including maximum moment restrictions (MMR) \citep{zhang2021maximum}, sieve minimum distance (SMD) \citep{ai2003efficient}, DeepIV \citep{deepiv}, DeepGMM \citep{bennett2020deep} and the neural network version of functional GEL (FGEL) \citep{kremer2022functional}. As an additional baseline we compare to ordinary least squares (OLS) which ignores the conditioning variable and minimizes $\avg \|\psi(x_i;\theta)\|_2^2$.
For all methods involving kernels we use a radial basis function (RBF) kernel, whose bandwidth is set via the median heuristic \citep{garreau2018large}. The hyperparameters of all methods are set by evaluating the Hilbert-Schmidt independence criterion (HSIC)~\citep{gretton2005measuring} between the residues and the conditioning variable $\operatorname{HSIC}(Y-g(T), Z)$ over a validation set of the size of the training set. HSIC has been proposed as an objective for CMR problems by \citet{mooij09} and \citet{saengkyongam2022exploiting} and we empirically find it to yield better estimates than MMR \citep{muandet2020kernel,zhang2021maximum} which has been used for as a validation metric in other works (see Section~\ref{sec:validation}).
For the variational approaches we use a batchsize of $n_1=200$.
Additionally, for our KMM esimator we represent the reference distribution $\omega$ by a kernel density estimator (KDE) trained on the empirical sample (see Section~\ref{app:reference-measure}) from which we sample mini-batches of size $n_2=200$. Refer to Section~\ref{sec:additional-details} for additional details.
An implementation of our estimator and code to reproduce our results is available at \href{https://github.com/HeinerKremer/conditional-moment-restrictions}{https://github.com/HeinerKremer/conditional-moment-restrictions}.
\paragraph{Heteroskedastic Instrumental Variable Regression}
We adopt the HeteroskedasticIV experiment of \citet{bennett2020variational}.
Let the data-generating process be given by
\begin{align*}
    &Z \sim \operatorname{Uniform}([-5, 5]^2), \quad \quad  H, \eta, \varepsilon \sim \mathcal{N}(0,1) \\ 
    &T_{\text{exo}} = Z_1 + |Z_2|, \quad \quad  T_{\text{endo}} = 5 H + 0.2 \eta \\
    &T = 0.75 T_{\text{exo}} + 0.25 T_{\text{endo}}, \quad  S = 0.1 \log( 1 + \exp(T_{\text{exo}})) \\
    &Y = g(T;\theta_0) + 5 H + S \varepsilon,
\end{align*}
where the parameter of interest $\theta_0 = [2.0, 3.0, -0.5, 3.0] \in \mathbb{R}^4$ enters the process via the function
\begin{align*}
    g(t;\theta) = \theta_2 + \theta_3 (t - \theta_1) + \frac{\theta_4 - \theta_3}{2} \log( 1+ e^{2(t-\theta_1)}).
\end{align*}
This task is particularly challenging as it involves heteroskedastic noise on the instruments. The true parameter $\theta_0$ is identified by imposing the CMR $E[Y - g(T;\theta)|Z] = 0 \ P_Z\text{-a.s.}$.
Table~\ref{tab:bennet_hetero} shows the mean squared error (MSE) of the parameter estimate for different methods and sample sizes. Our method provides a significantly lower MSE for small sample sizes and approaches the results of DeepGMM and FGEL for larger samples.
\begin{table*}[t]
    \centering
    \caption{Instrumental Variable Regression with Heteroskedastic Instrument Noise. Mean of the parameter MSE $\|\theta - \theta_0 \|^2$  and its standard error are computed over 20 random runs.}
    \label{tab:bennet_hetero}

    \begin{tabular}{lllllll}
    \hline
             & OLS           & MMR           & DeepIV        & DeepGMM    & FGEL   & KMM           \\
     n=500   & $1.78\pm0.21$ & $1.73\pm0.22$ & $2.57\pm0.06$ & $1.03\pm0.17$ & $1.02\pm0.19$ & $0.40\pm0.13$ \\
     n=1000  & $2.27\pm0.18$ & $1.97\pm0.23$ & $2.53\pm0.08$ & $0.70\pm0.16$ & $0.45\pm0.11$ & $0.16\pm0.04$ \\
     n=2000  & $1.79\pm0.10$ & $2.11\pm0.20$ & $2.43\pm0.06$ & $0.15\pm0.04$ & $0.14\pm0.03$ & $0.10\pm0.02$ \\
     n=4000  & $1.92\pm0.06$ & $1.65\pm0.15$ & $2.41\pm0.04$ & $0.07\pm0.02$ & $0.05\pm0.01$ & $0.07\pm0.02$ \\
     n=10000 & $1.99\pm0.04$ & N/A   & $1.99\pm0.04$ & $0.02\pm0.01$ & $0.03\pm0.01$ & $0.01\pm0.00$ \\
    \hline
    \end{tabular}
\end{table*}
\paragraph{Neural Network Instrumental Variable Regression}
To explore the viability of our estimator in the non-uniquely identified setting, we adopt the non-parametric instrumental variable regression experiment of \citet{lewis2018adversarial} which has also been used by \citet{bennett2020deep}, \citet{zhang2021maximum} and \citet{kremer2022functional}. 
Consider a data generating process given by
\begin{center}
\begin{tabular}{ l l l }
 $y = g_0(t) + e + \delta,$ & $t = z + e + \gamma,$ & $z \sim \operatorname{Uniform}([-3, 3]),$\\ 
 $e \sim N(0,1),$ & $\gamma, \delta \sim N(0, 0.1),$ &
\end{tabular}
\end{center}
where the function $g_0$ is chosen from
\begin{center}
\begin{tabular}{ l l }
 $\operatorname{sin:} g_0(t) = \sin(t)$, & $\operatorname{abs:} g_0(t) = |t| \vspace{.5em}$,\\ 
 $\operatorname{linear:} g_0(t) = t$, & $\operatorname{step:} g_0(t) = I_{\{t \geq 0 \}}$.
\end{tabular}
\end{center}
We try to learn an approximation of $g_0$ represented by a shallow neural network $g_\theta$ with 2 layers of $[20, 3]$ units and leaky ReLU activation functions.
We identify $g_\theta$ by imposing the conditional moment restrictions $E[Y-g_\theta(T)|Z] = 0 \  P_{Z}\text{-a.s.}$. We use training
and validation sets of size $n = 1000$ and evaluate the prediction error on a test set of size $20000$.
Table~\ref{tab:network-iv} shows the MSE of the predicted models trained with different CMR estimation methods.
We observe that our estimator consistently shows competitive performance and slightly outperforms the baselines on three out of four tasks.
\begin{table*}[t]
    \centering
    \caption{Neural Network Instrumental Variable Regression. Mean of the prediction MSE $E[\|g_\theta(T) - g_0(T) \|^2]$ and its standard error are computed over 30 random runs and scaled by a factor of ten for ease of presentation.}
    \begin{tabular}{llllllll}
    \hline
            & OLS           & SMD           & MMR           & DeepIV        & DeepGMM    & FGEL   & KMM           \\
     abs    & $3.21\pm0.14$ & $1.15\pm0.53$ & $1.41\pm0.48$ & $2.25\pm0.68$ & $0.42\pm0.04$ & $0.37\pm0.05$ & $0.32\pm0.06$ \\
     step   & $3.16\pm0.05$ & $0.54\pm0.06$ & $0.58\pm0.03$ & $0.74\pm0.04$ & $0.43\pm0.04$ & $0.40\pm0.04$ & $0.35\pm0.02$ \\
     sin    & $3.33\pm0.06$ & $1.31\pm0.08$ & $2.67\pm0.13$ & $3.75\pm0.15$ & $0.64\pm0.05$ & $0.62\pm0.04$ & $0.88\pm0.10$ \\
     linear & $2.95\pm0.08$ & $0.47\pm0.11$ & $0.96\pm0.20$ & $1.66\pm0.50$ & $0.49\pm0.05$ & $0.95\pm0.27$ & $0.43\pm0.11$ \\
    \hline
    \end{tabular}
    \label{tab:network-iv}
\end{table*}
\section{Related Work} \label{sec:related-work}
Conditional moment restrictions have been addressed in multiple ways by extending the GMM to continua of moment restrictions building on the equivalence between the conditional \eqref{eq:conditional-mr} and continuum \eqref{eq:continuum-mr} formulations. Seminal work in this direction has been carried out by \citep{Carrasco1,Carrasco2} and \citet{ai2003efficient}, which approximate the continuum of MR by a basis function expansion. Recently, the problem gained popularity in the machine learning community as many problems in causal inference can be formulated as CMR, most prominently instrumental variable regression. These modern approaches represent the continuum of MR via machine learning models, i.e., RKHS functions or neural networks and solve a mini-max formulation \citep{deepiv,lewis2018adversarial,bennett2020deep,bennett2020variational,Dikkala20:Minimax}.
Other GEL methods have historically played a less prominent role for CMR estimation, most likely due to their more complex mini-max structure compared to the simple minimization of traditional GMM-based methods. 
However, generalizations of GEL to continua of MR have been developed by \citet{tripathi2003testing,kitamura2004,chausse2012generalized,carrasco2017regularized} building on basis function expansions, which empirically have been competitive with their GMM-counterparts. Recently the problem has been addressed via modern machine learning models \citep{kremer2022functional}.  
All the aforementioned methods have in common that they either explicitly (GEL) or implicitly (GMM) optimize a $\varphi$-divergence between the candidate distributions and the empirical distribution and thus only allow for reweightings of the data. To the best of our knowledge, we provide the first method of moments estimator that lifts this restriction and allows for arbitrary candidate distributions.
The GEL framework bears a close duality relation to distributionally robust optimization (DRO)~\citep{lamRecoveringBestStatistical2019}.
In this context, it has been used to investigate the statistical properties of DRO \citep{duchi2018statistics,lamRecoveringBestStatistical2019} and to calibrate the size of the distributional ambiguity set used in the DRO framework \citep{lam2017optimization,lam2017empirical,blanchet2019robust,he2021higher}.
With the notable exception of \citet{blanchet2019robust} these works build on the standard $\varphi$-divergence-based GEL framework.
While \citet{blanchet2019robust} provide a GEL framework based on optimal transport distances, their goal is to calibrate an ambiguity set for DRO and they do not provide an estimator for moment restriction problems.
Computationally, an important contribution of this paper is handling the (semi)-infinite constraint in~\eqref{eq:unreg-mmd-profile}. 
Classical approaches to handling such constraints using polynomial sum-of-squares (SOS) \cite{lasserre2001global} do not apply here since we have a general moment function class outside the polynomials.
Furthermore, both classical and infinite-dimensional SOS techniques \cite{marteau2020non} suffer from scalability issues in high dimensions and large data sizes.
Compared to those, our entropy-regularization approach can be implemented with general nonlinear problems and stochastic-gradient-type algorithms.
\sloppy The objective of our inner optimization is a variational problem in the measure of the form 
$
\min_P \left\{\mathrm{F}(P) + \epsilon \mathrm{H}(P, Q)\right\}
$,
where $F$ is some energy functional and $H$ is some metric or divergence measure.
This was notably studied in the seminal work of \citet{jordan_variational_1998} as a time-discretization scheme of PDEs.
In recent literature related to machine learning, \citet{arbel_maximum_2019} studied the variational structure of MMD as energy in the Wasserstein geometry.
\citet{chizat_mean-field_2022} applied noisy particle gradient descent to an energy objective similar to ours, i.e., $\mmd + \epsilon D_{\mathrm{KL}}$.
Compared with that work, our goal is to train a model $\theta$ by minimizing this objective over $\theta$.
We also do not rely on gradient descent on the particles obtained from the discretization of the measure but adopt a dual functional gradient ascent scheme.
Our reformulation technique for the MMD-profile is similar to that of
\citet{zhu2021kernel}, who solved a similar variational problem involving MMD for DRO. Different from their method, our goal is to provide an estimator for CMR problems and we introduce entropy regularization as an interior point method for handling the constraint.

\section{Conclusion} \label{sec:conclusion}
The emergence of conditional moment restrictions in areas such as causal inference and robust machine learning has created the need for effective and robust estimation methods. 
Existing method of moments estimators (implicitly) rely on approximating the population distribution by reweighting a discrete empirical distribution. Our KMM estimator parts with this restrictive assumption and allows considering arbitrary (continuous) distributions as candidates for the population distribution.
As in many cases the population distribution is in fact continuous, this has the potential to find more accurate estimates especially in the low sample regime where reweightings can provide crude approximations. 
Our estimator comes with strong theoretical guarantees showing that it is first order efficient with respect to any estimator based on CMR and its competitive practical performance is demonstrated on several CMR tasks.
This paper laid the foundation of the KMM framework, which can inspire future work in multiple ways. Such work could include the development of more sophisticated and adaptive reference measures for the regularization scheme, e.g., by evolving the reference measure over the course of the optimization.
Another important direction would be a statistical learning theory analysis to provide theoretical properties of our estimator in the non-uniquely identified case. Other possibilities are extensions of the framework beyond estimation to construct confidence intervals for the estimates. 
Lastly, more efficient and tailored optimization methods can be developed to facilitate the application at larger scales.

\bibliography{refs}

\begin{thebibliography}{68}
\providecommand{\natexlab}[1]{#1}
\providecommand{\url}[1]{\texttt{#1}}
\expandafter\ifx\csname urlstyle\endcsname\relax
  \providecommand{\doi}[1]{doi: #1}\else
  \providecommand{\doi}{doi: \begingroup \urlstyle{rm}\Url}\fi

\bibitem[Ai and Chen(2003)]{ai2003efficient}
C.~Ai and X.~Chen.
\newblock Efficient estimation of models with conditional moment restrictions
  containing unknown functions.
\newblock \emph{Econometrica}, 71\penalty0 (6):\penalty0 1795--1843, 2003.

\bibitem[Angrist and Pischke(2008)]{angrist2008mostly}
J.~D. Angrist and J.-S. Pischke.
\newblock \emph{Mostly harmless econometrics}.
\newblock Princeton university press, 2008.

\bibitem[Arbel et~al.(2019)Arbel, Korba, Salim, and
  Gretton]{arbel_maximum_2019}
M.~Arbel, A.~Korba, A.~Salim, and A.~Gretton.
\newblock Maximum {Mean} {Discrepancy} {Gradient} {Flow}.
\newblock \emph{arXiv:1906.04370 [cs, stat]}, Dec. 2019.
\newblock arXiv: 1906.04370.

\bibitem[Bennett and Kallus(2020{\natexlab{a}})]{bennett2020efficient}
A.~Bennett and N.~Kallus.
\newblock Efficient policy learning from surrogate-loss classification
  reductions.
\newblock In \emph{International Conference on Machine Learning}, pages
  788--798. PMLR, 2020{\natexlab{a}}.

\bibitem[Bennett and Kallus(2020{\natexlab{b}})]{bennett2020variational}
A.~Bennett and N.~Kallus.
\newblock The variational method of moments, 2020{\natexlab{b}}.

\bibitem[Bennett et~al.(2019)Bennett, Kallus, and Schnabel]{bennett2020deep}
A.~Bennett, N.~Kallus, and T.~Schnabel.
\newblock Deep generalized method of moments for instrumental variable
  analysis.
\newblock \emph{Advances in neural information processing systems}, 32, 2019.

\bibitem[Bennett et~al.(2021)Bennett, Kallus, Li, and Mousavi]{bennett2021off}
A.~Bennett, N.~Kallus, L.~Li, and A.~Mousavi.
\newblock Off-policy evaluation in infinite-horizon reinforcement learning with
  latent confounders.
\newblock In \emph{International Conference on Artificial Intelligence and
  Statistics}, pages 1999--2007. PMLR, 2021.

\bibitem[Berlinet and Thomas-Agnan(2011)]{berlinet2011reproducing}
A.~Berlinet and C.~Thomas-Agnan.
\newblock \emph{Reproducing kernel Hilbert spaces in probability and
  statistics}.
\newblock Springer Science \& Business Media, 2011.

\bibitem[Bernstein(2009)]{bernstein2009matrix}
D.~S. Bernstein.
\newblock Matrix mathematics.
\newblock In \emph{Matrix Mathematics}. Princeton university press, 2009.

\bibitem[Bierens(1982)]{BIERENS1982105}
H.~J. Bierens.
\newblock Consistent model specification tests.
\newblock \emph{Journal of Econometrics}, 20\penalty0 (1):\penalty0 105--134,
  1982.

\bibitem[Blanchet et~al.(2019)Blanchet, Kang, and Murthy]{blanchet2019robust}
J.~Blanchet, Y.~Kang, and K.~Murthy.
\newblock Robust wasserstein profile inference and applications to machine
  learning.
\newblock \emph{Journal of Applied Probability}, 56\penalty0 (3):\penalty0
  830--857, 2019.

\bibitem[Carrasco and Florens(2000)]{Carrasco1}
M.~Carrasco and J.-P. Florens.
\newblock Generalization of gmm to a continuum of moment conditions.
\newblock \emph{Econometric Theory}, 16\penalty0 (6):\penalty0 797--834, 2000.
\newblock ISSN 02664666, 14694360.

\bibitem[Carrasco and Kotchoni(2017)]{carrasco2017regularized}
M.~Carrasco and R.~Kotchoni.
\newblock Regularized generalized empirical likelihood estimators.
\newblock Technical report, Technical report, 2017.

\bibitem[Carrasco et~al.(2007)Carrasco, Chernov, Florens, and
  Ghysels]{Carrasco2}
M.~Carrasco, M.~Chernov, J.-P. Florens, and E.~Ghysels.
\newblock Efficient estimation of general dynamic models with a continuum of
  moment conditions.
\newblock \emph{Journal of econometrics}, 140\penalty0 (2):\penalty0 529--573,
  2007.

\bibitem[Chamberlain(1987)]{CHAMBERLAIN1987305}
G.~Chamberlain.
\newblock Asymptotic efficiency in estimation with conditional moment
  restrictions.
\newblock \emph{Journal of Econometrics}, 34\penalty0 (3):\penalty0 305--334,
  1987.
\newblock ISSN 0304-4076.
\newblock \doi{https://doi.org/10.1016/0304-4076(87)90015-7}.

\bibitem[Chauss{\'e}(2012)]{chausse2012generalized}
P.~Chauss{\'e}.
\newblock Generalized empirical likelihood for a continuum of moment
  conditions.
\newblock 2012.

\bibitem[Chen and Pouzo(2009)]{chen2009efficient}
X.~Chen and D.~Pouzo.
\newblock Efficient estimation of semiparametric conditional moment models with
  possibly nonsmooth residuals.
\newblock \emph{Journal of Econometrics}, 152\penalty0 (1):\penalty0 46--60,
  2009.

\bibitem[Chen et~al.(2021)Chen, Xu, Gulcehre, Paine, Gretton, de~Freitas, and
  Doucet]{chen2021instrumental}
Y.~Chen, L.~Xu, C.~Gulcehre, T.~L. Paine, A.~Gretton, N.~de~Freitas, and
  A.~Doucet.
\newblock On instrumental variable regression for deep offline policy
  evaluation.
\newblock \emph{arXiv preprint arXiv:2105.10148}, 2021.

\bibitem[Chernozhukov et~al.(2016)Chernozhukov, Chetverikov, Demirer, Duflo,
  Hansen, Newey, and Robins]{chernozhukov2016double}
V.~Chernozhukov, D.~Chetverikov, M.~Demirer, E.~Duflo, C.~Hansen, W.~Newey, and
  J.~Robins.
\newblock Double/debiased machine learning for treatment and causal parameters.
\newblock \emph{arXiv preprint arXiv:1608.00060}, 2016.

\bibitem[Chernozhukov et~al.(2017)Chernozhukov, Chetverikov, Demirer, Duflo,
  Hansen, and Newey]{chernozhukov2017double}
V.~Chernozhukov, D.~Chetverikov, M.~Demirer, E.~Duflo, C.~Hansen, and W.~Newey.
\newblock Double/debiased/neyman machine learning of treatment effects.
\newblock \emph{American Economic Review}, 107\penalty0 (5):\penalty0 261--65,
  2017.

\bibitem[Chernozhukov et~al.(2018)Chernozhukov, Chetverikov, Demirer, Duflo,
  Hansen, Newey, and Robins]{chernozhukov2018double}
V.~Chernozhukov, D.~Chetverikov, M.~Demirer, E.~Duflo, C.~Hansen, W.~Newey, and
  J.~Robins.
\newblock Double/debiased machine learning for treatment and structural
  parameters, 2018.

\bibitem[Chizat(2022)]{chizat_mean-field_2022}
L.~Chizat.
\newblock Mean-{Field} {Langevin} {Dynamics}: {Exponential} {Convergence} and
  {Annealing}, Aug. 2022.
\newblock arXiv:2202.01009 [math].

\bibitem[Cuturi(2013)]{cuturi2013sinkhorn}
M.~Cuturi.
\newblock Sinkhorn distances: Lightspeed computation of optimal transport.
\newblock \emph{Advances in neural information processing systems}, 26, 2013.

\bibitem[Daskalakis et~al.(2018)Daskalakis, Ilyas, Syrgkanis, and
  Zeng]{daskalakis2018training}
C.~Daskalakis, A.~Ilyas, V.~Syrgkanis, and H.~Zeng.
\newblock Training gans with optimism, 2018.

\bibitem[Dikkala et~al.(2020)Dikkala, Lewis, Mackey, and
  Syrgkanis]{Dikkala20:Minimax}
N.~Dikkala, G.~Lewis, L.~Mackey, and V.~Syrgkanis.
\newblock Minimax estimation of conditional moment models.
\newblock In \emph{Advances in Neural Information Processing Systems},
  volume~33, pages 12248--12262. Curran Associates, Inc., 2020.

\bibitem[Duchi et~al.(2018)Duchi, Glynn, and Namkoong]{duchi2018statistics}
J.~Duchi, P.~Glynn, and H.~Namkoong.
\newblock Statistics of robust optimization: A generalized empirical likelihood
  approach, 2018.

\bibitem[Garreau et~al.(2018)Garreau, Jitkrittum, and
  Kanagawa]{garreau2018large}
D.~Garreau, W.~Jitkrittum, and M.~Kanagawa.
\newblock Large sample analysis of the median heuristic, 2018.

\bibitem[Goodfellow et~al.(2014)Goodfellow, Pouget-Abadie, Mirza, Xu,
  Warde-Farley, Ozair, Courville, and Bengio]{goodfellow2014generative}
I.~Goodfellow, J.~Pouget-Abadie, M.~Mirza, B.~Xu, D.~Warde-Farley, S.~Ozair,
  A.~Courville, and Y.~Bengio.
\newblock Generative adversarial nets.
\newblock \emph{Advances in neural information processing systems}, 27, 2014.

\bibitem[Gretton et~al.(2005)Gretton, Bousquet, Smola, and
  Sch{\"o}lkopf]{gretton2005measuring}
A.~Gretton, O.~Bousquet, A.~Smola, and B.~Sch{\"o}lkopf.
\newblock Measuring statistical dependence with hilbert-schmidt norms.
\newblock In \emph{Algorithmic Learning Theory: 16th International Conference,
  ALT 2005, Singapore, October 8-11, 2005. Proceedings 16}, pages 63--77.
  Springer, 2005.

\bibitem[Gretton et~al.(2012)Gretton, Borgwardt, Rasch, Sch{\"o}lkopf, and
  Smola]{gretton2012kernel}
A.~Gretton, K.~M. Borgwardt, M.~J. Rasch, B.~Sch{\"o}lkopf, and A.~Smola.
\newblock A kernel two-sample test.
\newblock \emph{The Journal of Machine Learning Research}, 13\penalty0
  (1):\penalty0 723--773, 2012.

\bibitem[Hansen(1982)]{hansen}
L.~P. Hansen.
\newblock Large sample properties of generalized method of moments estimators.
\newblock \emph{Econometrica}, 50\penalty0 (4):\penalty0 1029--1054, 1982.
\newblock ISSN 00129682, 14680262.

\bibitem[Hansen et~al.(1996)Hansen, Heaton, and Yaron]{hansen-finite-samples}
L.~P. Hansen, J.~Heaton, and A.~Yaron.
\newblock Finite-sample properties of some alternative gmm estimators.
\newblock \emph{Journal of Business \& Economic Statistics}, 14\penalty0
  (3):\penalty0 262--280, 1996.
\newblock ISSN 07350015.

\bibitem[Hartford et~al.(2017)Hartford, Lewis, Leyton-Brown, and Taddy]{deepiv}
J.~Hartford, G.~Lewis, K.~Leyton-Brown, and M.~Taddy.
\newblock Deep iv: A flexible approach for counterfactual prediction.
\newblock In \emph{International Conference on Machine Learning}, pages
  1414--1423. PMLR, 2017.

\bibitem[He and Lam(2021)]{he2021higher}
S.~He and H.~Lam.
\newblock Higher-order expansion and bartlett correctability of
  distributionally robust optimization.
\newblock \emph{arXiv preprint arXiv:2108.05908}, 2021.

\bibitem[Ho et~al.(2020)Ho, Jain, and Abbeel]{ho2020denoising}
J.~Ho, A.~Jain, and P.~Abbeel.
\newblock Denoising diffusion probabilistic models, 2020.

\bibitem[Hornik et~al.(1989)Hornik, Stinchcombe, and
  White]{hornik1989multilayer}
K.~Hornik, M.~Stinchcombe, and H.~White.
\newblock Multilayer feedforward networks are universal approximators.
\newblock \emph{Neural networks}, 2\penalty0 (5):\penalty0 359--366, 1989.

\bibitem[Jordan et~al.(1998)Jordan, Kinderlehrer, and
  Otto]{jordan_variational_1998}
R.~Jordan, D.~Kinderlehrer, and F.~Otto.
\newblock The variational formulation of the {Fokker}–{Planck} equation.
\newblock \emph{SIAM journal on mathematical analysis}, 29\penalty0
  (1):\penalty0 1--17, 1998.
\newblock Publisher: SIAM.

\bibitem[Kingma and Welling(2022)]{kingma2022autoencoding}
D.~P. Kingma and M.~Welling.
\newblock Auto-encoding variational bayes, 2022.

\bibitem[Kitamura et~al.(2004)Kitamura, Tripathi, and Ahn]{kitamura2004}
Y.~Kitamura, G.~Tripathi, and H.~Ahn.
\newblock Empirical likelihood-based inference in conditional moment
  restriction models.
\newblock \emph{Econometrica}, 72\penalty0 (6):\penalty0 1667--1714, 2004.
\newblock ISSN 00129682, 14680262.

\bibitem[Kremer et~al.(2022)Kremer, Zhu, Muandet, and
  Sch{\"o}lkopf]{kremer2022functional}
H.~Kremer, J.-J. Zhu, K.~Muandet, and B.~Sch{\"o}lkopf.
\newblock Functional generalized empirical likelihood estimation for
  conditional moment restrictions.
\newblock In \emph{International Conference on Machine Learning}, pages
  11665--11682. PMLR, 2022.

\bibitem[Lam(2019)]{lamRecoveringBestStatistical2019}
H.~Lam.
\newblock Recovering best statistical guarantees via the empirical
  divergence-based distributionally robust optimization.
\newblock \emph{Operations Research}, 67\penalty0 (4):\penalty0 1090--1105,
  2019.

\bibitem[Lam and Qian(2017)]{lam2017optimization}
H.~Lam and H.~Qian.
\newblock Optimization-based quantification of simulation input uncertainty via
  empirical likelihood.
\newblock \emph{arXiv preprint arXiv:1707.05917}, 2017.

\bibitem[Lam and Zhou(2017)]{lam2017empirical}
H.~Lam and E.~Zhou.
\newblock The empirical likelihood approach to quantifying uncertainty in
  sample average approximation.
\newblock \emph{Operations Research Letters}, 45\penalty0 (4):\penalty0
  301--307, 2017.

\bibitem[Lasserre(2001)]{lasserre2001global}
J.~B. Lasserre.
\newblock Global optimization with polynomials and the problem of moments.
\newblock \emph{SIAM Journal on optimization}, 11\penalty0 (3):\penalty0
  796--817, 2001.

\bibitem[Lewis and Syrgkanis(2018)]{lewis2018adversarial}
G.~Lewis and V.~Syrgkanis.
\newblock Adversarial generalized method of moments, 2018.

\bibitem[Marteau-Ferey et~al.(2020)Marteau-Ferey, Bach, and
  Rudi]{marteau2020non}
U.~Marteau-Ferey, F.~Bach, and A.~Rudi.
\newblock Non-parametric models for non-negative functions.
\newblock \emph{Advances in neural information processing systems},
  33:\penalty0 12816--12826, 2020.

\bibitem[Micchelli et~al.(2006)Micchelli, Xu, and Zhang]{universalkernel06}
C.~Micchelli, Y.~Xu, and H.~Zhang.
\newblock Universal kernels.
\newblock \emph{Mathematics}, 7, 12 2006.

\bibitem[Mooij et~al.(2009)Mooij, Janzing, Peters, and Schölkopf]{mooij09}
J.~Mooij, D.~Janzing, J.~Peters, and B.~Schölkopf.
\newblock Regression by dependence minimization and its application to causal
  inference.
\newblock page~94, 06 2009.
\newblock \doi{10.1145/1553374.1553470}.

\bibitem[Muandet et~al.(2020)Muandet, Jitkrittum, and
  Kübler]{muandet2020kernel}
K.~Muandet, W.~Jitkrittum, and J.~Kübler.
\newblock Kernel conditional moment test via maximum moment restriction, 2020.

\bibitem[Nesterov and Nemirovskii(1994)]{nesterov1994interior}
Y.~Nesterov and A.~Nemirovskii.
\newblock \emph{Interior-point polynomial algorithms in convex programming}.
\newblock SIAM, 1994.

\bibitem[Newey and Powell(2003)]{newey2003instrumental}
W.~K. Newey and J.~L. Powell.
\newblock Instrumental variable estimation of nonparametric models.
\newblock \emph{Econometrica}, 71\penalty0 (5):\penalty0 1565--1578, 2003.

\bibitem[Newey and Smith(2004)]{newey04}
W.~K. Newey and R.~J. Smith.
\newblock Higher order properties of gmm and generalized empirical likelihood
  estimators.
\newblock \emph{Econometrica}, 72\penalty0 (1):\penalty0 219--255, 2004.
\newblock ISSN 00129682, 14680262.

\bibitem[Owen(1990)]{owen90}
A.~Owen.
\newblock Empirical likelihood ratio confidence regions.
\newblock \emph{The Annals of Statistics}, 18\penalty0 (1):\penalty0 90--120,
  1990.
\newblock ISSN 00905364.

\bibitem[Owen(1988)]{owen88}
A.~B. Owen.
\newblock Empirical likelihood ratio confidence intervals for a single
  functional.
\newblock \emph{Biometrika}, 75\penalty0 (2):\penalty0 237--249, 1988.
\newblock ISSN 00063444.

\bibitem[Papamakarios et~al.(2021)Papamakarios, Nalisnick, Rezende, Mohamed,
  and Lakshminarayanan]{papamakarios2021normalizing}
G.~Papamakarios, E.~Nalisnick, D.~J. Rezende, S.~Mohamed, and
  B.~Lakshminarayanan.
\newblock Normalizing flows for probabilistic modeling and inference.
\newblock \emph{The Journal of Machine Learning Research}, 22\penalty0
  (1):\penalty0 2617--2680, 2021.

\bibitem[Qin and Lawless(1994)]{qin-lawless}
J.~Qin and J.~Lawless.
\newblock Empirical likelihood and general estimating equations.
\newblock \emph{The Annals of Statistics}, 22\penalty0 (1):\penalty0 300--325,
  1994.
\newblock ISSN 00905364.

\bibitem[Rahimi and Recht(2007)]{NIPS2007_013a006f}
A.~Rahimi and B.~Recht.
\newblock Random features for large-scale kernel machines.
\newblock In J.~Platt, D.~Koller, Y.~Singer, and S.~Roweis, editors,
  \emph{Advances in Neural Information Processing Systems}, volume~20. Curran
  Associates, Inc., 2007.

\bibitem[Rosenblatt(1956)]{rosenblatt1956remarks}
M.~Rosenblatt.
\newblock Remarks on some nonparametric estimates of a density function.
\newblock \emph{The annals of mathematical statistics}, pages 832--837, 1956.

\bibitem[Saengkyongam et~al.(2022)Saengkyongam, Henckel, Pfister, and
  Peters]{saengkyongam2022exploiting}
S.~Saengkyongam, L.~Henckel, N.~Pfister, and J.~Peters.
\newblock Exploiting independent instruments: Identification and distribution
  generalization, 2022.

\bibitem[Sch{\"o}lkopf and Smola(2002)]{scholkopf2002learning}
B.~Sch{\"o}lkopf and A.~J. Smola.
\newblock \emph{Learning with kernels: support vector machines, regularization,
  optimization, and beyond}.
\newblock MIT press, 2002.

\bibitem[Simon-Gabriel and Sch{{\"o}}lkopf(2018)]{JMLR:v19:16-291}
C.-J. Simon-Gabriel and B.~Sch{{\"o}}lkopf.
\newblock Kernel distribution embeddings: Universal kernels, characteristic
  kernels and kernel metrics on distributions.
\newblock \emph{Journal of Machine Learning Research}, 19\penalty0
  (44):\penalty0 1--29, 2018.
\newblock URL \url{http://jmlr.org/papers/v19/16-291.html}.

\bibitem[Smith(1997)]{smith1997alternative}
R.~J. Smith.
\newblock Alternative semi-parametric likelihood approaches to generalised
  method of moments estimation.
\newblock \emph{The Economic Journal}, 107\penalty0 (441):\penalty0 503--519,
  1997.

\bibitem[Steinwart and Christmann(2008)]{steinwart2008support}
I.~Steinwart and A.~Christmann.
\newblock \emph{Support vector machines}.
\newblock Springer Science \& Business Media, 2008.

\bibitem[Tripathi and Kitamura(2003)]{tripathi2003testing}
G.~Tripathi and Y.~Kitamura.
\newblock Testing conditional moment restrictions.
\newblock \emph{The Annals of Statistics}, 31\penalty0 (6):\penalty0
  2059--2095, 2003.

\bibitem[Xu et~al.(2021)Xu, Chen, Srinivasan, de~Freitas, Doucet, and
  Gretton]{xu2021learning}
L.~Xu, Y.~Chen, S.~Srinivasan, N.~de~Freitas, A.~Doucet, and A.~Gretton.
\newblock Learning deep features in instrumental variable regression.
\newblock In \emph{International Conference on Learning Representations}, 2021.

\bibitem[Zeidler(2012)]{zeidler2012applied}
E.~Zeidler.
\newblock \emph{Applied functional analysis: applications to mathematical
  physics}, volume 108.
\newblock Springer Science \& Business Media, 2012.

\bibitem[Zhang et~al.(2021)Zhang, Imaizumi, Schölkopf, and
  Muandet]{zhang2021maximum}
R.~Zhang, M.~Imaizumi, B.~Schölkopf, and K.~Muandet.
\newblock Maximum moment restriction for instrumental variable regression,
  2021.
\newblock arXiv 2010.07684.

\bibitem[Zhu et~al.(2021)Zhu, Jitkrittum, Diehl, and
  Sch{\"o}lkopf]{zhu2021kernel}
J.-J. Zhu, W.~Jitkrittum, M.~Diehl, and B.~Sch{\"o}lkopf.
\newblock Kernel distributionally robust optimization: Generalized duality
  theorem and stochastic approximation.
\newblock In \emph{International Conference on Artificial Intelligence and
  Statistics}, pages 280--288. PMLR, 2021.

\end{thebibliography}
\bibliographystyle{abbrvnat}

\newpage
\appendix
\onecolumn

\section{KMM for Functional Moment Restrictions}\label{sec:KMMFMR}

\subsection{Duality} \label{sec:primal-cmr}
The primal problem of the entropy regularized KMM estimator for functional moment restrictions is given by
\begin{align}
     R^\varphi_\epsilon(\theta) = & \inf_{P \in \mathcal{P}} \frac{1}{2}  \mmd(P,\hat{P}_n;\mathcal{F})^2 + \epsilon D_\varphi(P||\omega) \quad \mathrm{s.t.} \quad \|E_P[\Psi(X,Z;\theta)]\|_{\mathcal{H}^\ast} \leq \lambda_n, \quad E_P[1] = 1, \label{eq:primal-cmr}
\end{align}
where we relaxed the moment restrictions to hold only exactly for $n\rightarrow \infty$.
Note that to be precise, the dual of \eqref{eq:primal-cmr}, which can be obtained following the proof of Theorem~\ref{th:kel-duality}, contains a regularization term $- \lambda_n \| h\|_\mathcal{H}$ instead of $-\frac{1}{2} \| h\|_\mathcal{H}^2$ as in our Definition~\ref{def:kmm-cmr}. However, by Lagrangian duality the regularizer  $- \lambda_n \| h\|_\mathcal{H}$ corresponds to restricting $\mathcal{H}$ to a norm ball of some radius $\rho$, and equally $-\frac{1}{2} \| h\|_\mathcal{H}^2$ corresponds to a restriction to a norm ball of different radius $\rho'$. Therefore both formulations are practically equivalent and we use the squared version for its greater smoothness and facilitated theoretical analysis. Note that in this context that a theoretical analysis of the $-\lambda_n \|h \|_\mathcal{H}$ version would be possible by resorting to the variational formulation of the norm $\|h \|_\mathcal{H} = \sup_{h' \in \mathcal{H}, \|h' \|\leq 1} \langle h', h \rangle_\mathcal{H}$.
For an appropriate choice of reference distribution $\omega$ a solution to the KMM problem \eqref{eq:primal-cmr} at the true parameter $\theta_0 \in \Theta$ always exists as the true distribution $P_0$ is contained in an MMD ball around the empirical distribution with probability $1$. This is in stark contrast to the $\varphi$-divergence based FGEL estimator of \citet{kremer2022functional}, as a $\varphi$-divergence ball around the empirical distribution $\hat{P}_n$ generally contains the corresponding continuous true distribution with probability $0$, as the $\varphi$ divergence between a discrete distribution $\hat{P}_n$ and continuous distribution $P_0$ diverges.
However, at different parameters $\theta \in \Theta$ existence of a distribution $P \in \mathcal{P}$ for which the functional moment restrictions hold exactly cannot be guaranteed which implies $R(\theta) =\infty$ and thus gradient-based optimization over $\theta \in \Theta$ can become difficult. Therefore the role of the relaxation parameter $\lambda_n$ here is to smooth the MMD profile such that $R(\theta) < \infty$ in a neighbourhood of the true parameter to facilitate gradient-based optimization over $\theta$. Note that even for fixed values of $\lambda_n$, i.e., $\lambda_n = O_p(1)$, as $n \rightarrow \infty$ the objective has its global minimum of $0$ at $P = P_0$ as $\hat{P}_n \overset{p}{\rightarrow} P_0$ and $\omega \overset{p}{\rightarrow} P_0$ weakly and thus we will retrieve the true solution $\theta_0$. Therefore, compared to \citet{kremer2022functional} where the relaxation scheme is a fundamental necessity to restore strong duality, here the regularization parameter can be seen merely as a computational tool.

\subsection{Asymptotic Properties} \label{sec:KMM-FMR}
For the KMM estimator for functional moment restrictions (FMR) of the form \eqref{eq:fmr} based on the functional MMD profile \eqref{eq:functional-kgel} we have the following properties.
\begin{theorem}[Consistency for FMR]\label{th:consistency-fmr}
Assume that 
a)~$\theta_0 \in \Theta$ is the unique solution to $E[\Psi(X,Z;\theta)] = 0 \in \mathcal{H}^\ast$; 
b)~$\Theta \subset \mathbb{R}^p$ is compact; 
c)~$\Psi(X,Z;\theta)$ is continuous in $\theta$ at any $\theta \in \Theta$ with probability one;
d)~$E[\sup_{\theta \in \Theta} \| \Psi(X,Z;\theta) \|_{\mathcal{H}^\ast}^2] < \infty$; 
e)~$\Omega_0 = E[\Psi(X,Z;\theta_0) \otimes \Psi(X,Z;\theta_0)]$ is non-singular; 
f)~$\omega = (1-\alpha) \hat{P}_n + \alpha Q$ for $\alpha= O_p(n^{-1})$ and any distribution $Q$ such that $E_Q[\sup_{\theta \in \Theta} \| \Psi(X,Z;\theta) \|_{\mathcal{H}^\ast}^2] < \infty$; 
and
g)~$\lambda_n = O_p(n^{-\xi})$ with $0 < \xi < 1/2$.
Let $\hat{\theta}$ denote the functional KMM estimator for $\theta_0$, then $\hat{\theta} \overset{p}{\rightarrow} \theta_0$.

If additionally 
    h)~$\theta_0 \in \operatorname{int}(\Theta)$;
    i)~$\Psi(x,z;\theta)$ is continuously differentiable in a neighborhood $\bar{\Theta}$ of $\theta_0$ and $E[\sup_{\theta \in \bar{\Theta}} \| \nabla_\theta \Psi(X,Z;\theta) \|^2_{\mathcal{H}^\ast}] < \infty$; as well as
    j)~$\Sigma_0 = \left\langle E[ \nabla_\theta \Psi(X,Z;{\theta_0})], E[ \nabla_\theta \Psi(X,Z;{\theta_0})]  \right\rangle_{\mathcal{H}^\ast} \in \mathbb{R}^{p \times p}$ is non-singular,
we have $\|\hat{\theta} - \theta_0 \| = O_{p}(n^{-1/2})$.
\end{theorem}

\begin{theorem}[Asymptotic Normality for FMR]\label{th:asymptotic-normality}
Let Assumptions a)-j) of Theorem~\ref{th:consistency-fmr} be satisfied. 
Then,
\begin{align*}
    \sqrt{n} (\hat{\theta} - \theta_0)  \overset{d}{\rightarrow}  N(0, \Xi_0)
\end{align*}
where 
   $\Xi_0 = \left( E[\nabla_\theta \Psi(X,Z;\theta_0)] \Omega_0^{-1} E[\nabla_\theta \Psi(X,Z;\theta_0)] \right)^{-1}$.
\end{theorem}
\section{Asymptotic Properties of the Finite-Dimensional KMM Estimator}\label{sec:KMMfinite}
For the KMM estimator for finite dimensional moment restrictions based on \eqref{eq:dual-profile} we have the following results. 
\begin{theorem}[Consistency for MR]\label{th:consistency-mr}
Assume that 
a)~$\theta_0 \in \Theta$ is the unique solution to $E[\psi(X;\theta)] = 0 \in \mathbb{R}^m$; 
b)~$\Theta \subset \mathbb{R}^p$ is compact; 
c)~$\psi(X;\theta)$ is continuous at each $\theta \in \Theta$ with probability one; d)~$E[\sup_{\theta \in \Theta} \| \psi(X;\theta) \|_{2}^2] < \infty$; 
e)~The covariance matrix $\Omega_0 := E\left[\psi(X,\theta_0) \psi(X;\theta_0)^T\right]$ is non-singular;
and
f)~$\omega = (1-\alpha) \hat{P}_n + \alpha Q$ for $\alpha= O_p(n^{-1})$ and any distribution $Q$ such that $E_Q[\sup_{\theta \in \Theta} \| \psi(X;\theta) \|_{2}^2] < \infty$.
Let $\hat{\theta}$ denote the KMM estimator for $\theta_0$, then $\hat{\theta} \overset{p}{\rightarrow} \theta_0$. 

If additionally 
    g)~$\theta_0 \in \operatorname{int}(\Theta)$;
    h)~$\psi(x;\theta)$ is continuously differentiable in a neighborhood $\bar{\Theta}$ of $\theta_0$ and $E[\sup_{\theta \in \bar{\Theta}} \| \nabla_\theta \psi(X;\theta) \|^2] < \infty$ w.p.1 as well as
    i)~$\operatorname{rank}\left(E[\nabla_\theta \psi(X;\theta_0)] \right) = p$,
we have $\|\hat{\theta} - \theta_0 \| = O_{p}(n^{-1/2})$.
\end{theorem}

\begin{theorem}[Asymptotic Normality for MR]\label{th:asymptotic-normality-mr}
Let Assumptions a)-i) of Theorem~\ref{th:consistency-mr} be satisfied. 
Then,
\begin{align*}
    \sqrt{n} (\hat{\theta} - \theta_0)  \overset{d}{\rightarrow}  N(0, \Xi_0)
\end{align*}
where 
    $\Xi_0 = \left( E[\nabla_\theta \psi(X;\theta_0)] \ \Omega_0^{-1} \ E[\nabla_\theta \psi(X;\theta_0)] \right)^{-1}$.
\end{theorem}
\begin{remark}
    The asymptotic variance $\Xi_0$ of the KMM estimator agrees with the one of the optimally weighted GMM estimator \citep{hansen}, thus for finite dimensional moment restrictions KMM and OW-GMM are asymptotically first-order equivalent.
\end{remark}
\section{Additional Experimental Details} \label{sec:additional-details}
\subsection{Hyperparameter choices}
For the KMM estimator and the baselines we set the hyperparameters within the values described below using the setting with the minimal value for $\operatorname{HSIC}(\psi(X;\theta), Z)$ evaluated on a validation set of the same size as the training set. As for large samples the $\operatorname{HSIC}$ computation becomes increasingly expensive we partition the validation data into batches of size $n_b=2000$ and average $\operatorname{HSIC}$ over the batches.
For the variational methods we use an optimistic Adam \citep{daskalakis2018training} implementation with a mini-batch size of $n=200$ and a learning rate of $\tau_\theta = 5 \cdot 10^{-4}$ for optimization over $\theta$ and $\tau_h = 2.5 \cdot 10^{-3}$ for optimization over $h$ and $\beta = (\eta, f, h)$ respectively. The regularization parameter $\lambda$ for the instrument function $h \in \mathcal{H}$ is picked from $\lambda \in [0, 10^{-4}, 10^{-2}, 1]$.
Specific to FGEL we treat the divergence $\varphi$ as a hyperparameter which we pick from $\varphi \in [\operatorname{KL}, \operatorname{log}, \chi^2]$.
Specific to KMM we use $n_{\textrm{RF}} = 2000$ random Fourier features and for every batch of size $n_\textrm{batch} = 200$ sampled from $\hat{P}_n$ we attach $n_\text{reference}=200$ samples from a reference distribution $Q$ which we represent by a kernel density estimator with Gaussian kernel and bandwidth of $\sigma = 0.1$ trained on $\hat{P}_n$. We observed that the results are largely insensitive to the choice of bandwidth parameter $\sigma$.
The entropy regularization parameter $\epsilon$ is picked from $\epsilon \in [ 0.1, 1, 10]$.
The entropy regularizer is chosen as the Kullback-Leibler divergence as in the first part of Section~\ref{sec:choices}. In agreement with the observations of \citet{kremer2022functional} we noticed experimentally that the choice of $\varphi$-divergence has only a minor effect on the obtained estimator. 
\subsection{Choice of Validation Metric and Failure of MMR} \label{sec:validation}
\begin{figure}[h!]
    \centering
    \includegraphics[width=\linewidth]{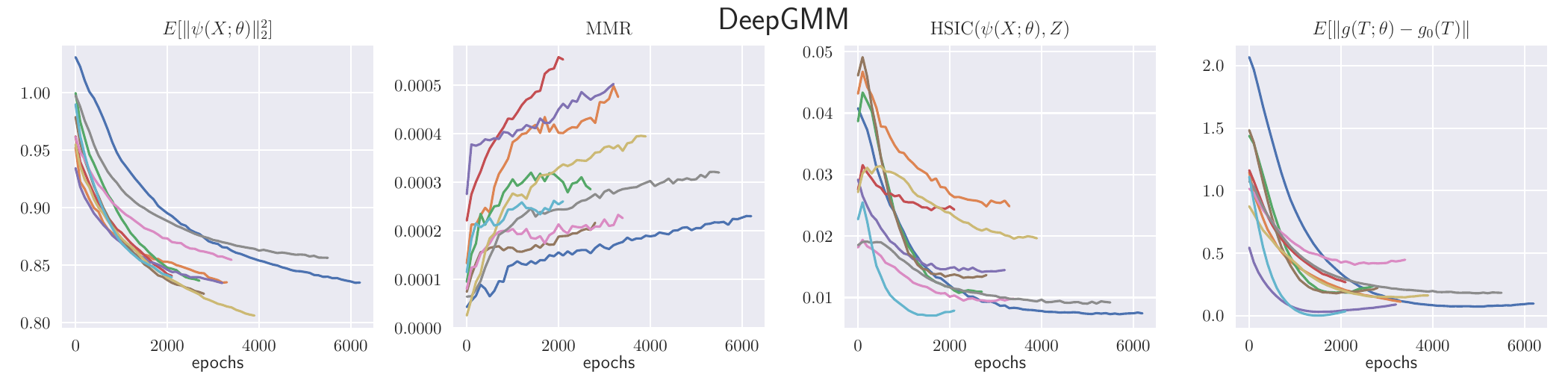}
    \includegraphics[width=\linewidth]{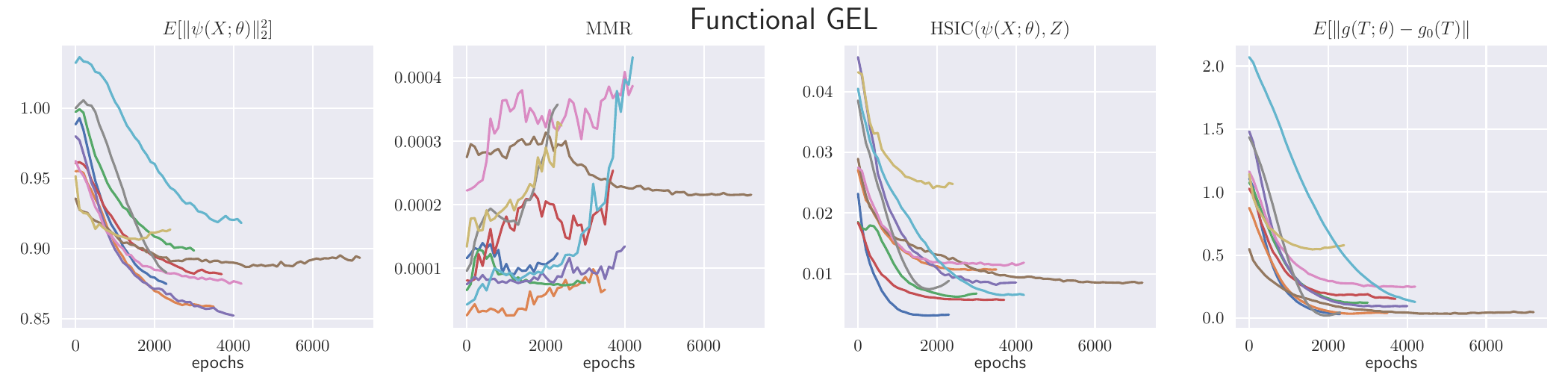}
    \includegraphics[width=\linewidth]{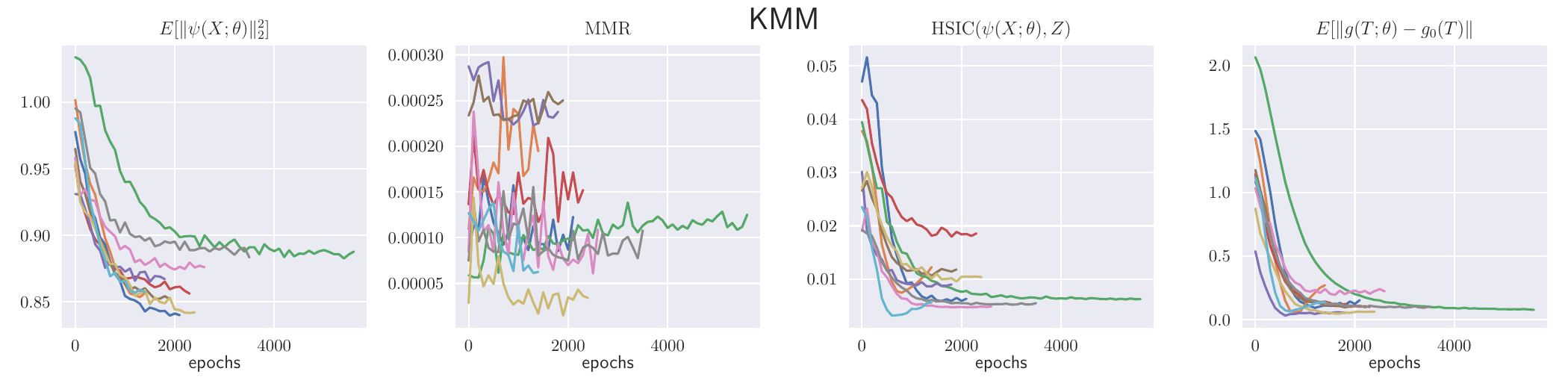}
    \caption{Effects of Validation Metrics for Early Stopping. Visualization of different validation losses for $10$ training samples and different estimators. Goal of the estimation is to minimize the error with respect to the true function $g_0$ shown on the right which is unknown in practice. We observe that among the considered validation metrics, HSIC is the only one that approximately follows the behavior of the error with respect to the true function and thus allows for effective early stopping. 
    The author's implementations of DeepGMM and FGEL use MMR as validation loss. Switching to HSIC allowed us to improve the performance of these baselines by a factor of 2-10.}
    \label{fig:validation}
\end{figure}
The computation of modern CMR estimators including DeepGMM~\citep{bennett2020deep}, Functional GEL~\citep{kremer2022functional} and our KMM estimator generally requires solving a mini-max or saddle point problem where the minimization is with respect to the model parameters and the maximization with respect to the instrument function $h$ (and the RKHS function $f$ in the case of KMM). For such problems it is not obvious how to monitor the success of the training procedure as for conditional moment restriction problems it is not clear which validation objective is supposed to be optimized, which makes tuning of hyperparameters and early stopping cumbersome and ambiguous. This is in contrast to standard supervised learning via empirical risk minimization where training and target objectives are usually aligned and thus one can simply evaluate the loss function over a suitable validation set. There exist different approaches to quantify how well the restrictions $E[\psi(X;\theta)|Z] = 0 \ P_Z\text{-a.s.}$ are satisfied. The authors of \citet{bennett2020deep} and \citet{kremer2022functional} used the maximum moment restriction objective \citep{zhang2021maximum} $\operatorname{MMR}(\theta) = E_{\hat{P}_n}[\psi(X;\theta) K(Z,Z') \psi(X';\theta)]$ which results from the variational formulation $\operatorname{MMR}(\theta) = \sup_{h \in \mathcal{H}} \left(E_{\hat{P}_n}[\psi(X;\theta)^T h(Z)] \right)^2$ with $\mathcal{H}$ corresponding to a unit ball of a reproducing kernel Hilbert space. While \citet{zhang2021maximum} show that this leads to a consistent estimator for $\theta_0$ when optimized over $\theta \in \Theta$ and thus quantifies the satisfaction of the CMR in a meaningful way, it is a priori not clear if it provides a suitable validation metric in finite samples. 
As an alternative \citet{saengkyongam2022exploiting} proposed to measure the satisfaction of $E[\psi(X;\theta)|Z] = 0 \ P_Z\text{-a.s.}$ by quantifying the independence of the random variables $\psi(X;\theta)$ and $Z$ via the Hilbert-Schmidt independence criterion~(HSIC)~\citep{gretton2005measuring}. 
We tested these two validation metrics for hyperparameter optimization and early stopping and observed that using HSIC instead of MMR as validation metric leads to improvements of the predictive MSE of the variational estimators (DeepGMM, FGEL, KMM) by a factor of $2$-$10$, when keeping all other settings (i.e. hyperparameter grids) fixed.
We exemplarily visualize the effect of different validation metrics for early stopping for the heteroskedastic IV experiment in Figure~\ref{fig:validation}. We train all estimators for $10$ different random samples and use HSIC with a loose stopping criterion as validation metric in order to train beyond the optimal validation loss for visualization. The left column shows the prediction MSE of the learned function. While we aim to optimize this quantity, in practice we do not have access to it as the true function $g_0$ is unknown. The remaining columns show the different validation metrics over the course of the optimization. We observe that using the simplistic unconditional moment violation generally leads to overfitting as the estimator would be trained beyond the minimum of the true objective of interest. Interestingly, in most cases the MMR objective does not decrease over the course of the optimization procedure and thus any early stopping strategy based on it might stop the training at random. Of the three metrics, HSIC is the only one that approximately mimics the behavior of the true objective of interest and thus allows for an effective early stopping strategy to prevent overfitting and unnecessary long training.
\section{Entropy Regularization}\label{sec:appendix-entropy}
\subsection{Effect of the Regularization Parameter}
As discussed in Section~\ref{sec:regularization}, for the case of the backward KL divergence or Burg entropy, our entropy regularization can be interpreted as a barrier function in an interior-point method, see \citep{nesterov1994interior}.
For decreasing values of $\epsilon$, the \emph{entropy-regularized MMD profile} approaches the unregularized \emph{MMD-profile}. To validate this empirically we carry out the maximization over the dual parameters $(\eta, f)$ in equation~\eqref{eq:dual-profile} while keeping $h$ and $\theta$ fixed, which preserves the convex structure of the problem.
In Figure~\ref{fig:reg_effect}(a) we observe that for smaller $\epsilon$ we get closer and closer to the original \emph{MMD-profile}, which we obtain from equation~\eqref{eq:unreg-mmd-profile} by using a sample approximation of the semi-infinite constraint in a convex solver.

\subsection{Annealing of Entropy Regularization}
Instead of keeping the regularization parameter $\epsilon$ fixed during the optimization procedure as in Figure~\ref{fig:reg_effect}(a), we study an annealing schedule in which it is gradually decreased, similar to actual interior-point methods. \citet{chizat_mean-field_2022} also studied effects of annealing in a setting where they use particle-gradient descent.
While their work also builds on an energy functional consisting of a combination of MMD and KL-divergence, the dissipation is done in the Wasserstein geometry.
In comparison, we do not carry out the optimization by moving in the Wasserstein space, but instead in the dual RKHS.
To visualize the effect of annealing, we keep $h$ and $\theta$ fixed and only maximize with respect to the remaining dual variables $(\eta, f)$, while gradually decreasing $\epsilon$ with the number of iterations. 
We empirically observe that the annealing procedure eventually leads to a solution that satisfies the (semi-)infinite constraint~(11). This is visualized in Figure~\ref{fig:reg_effect}(b) where the shaded black curves, corresponding to $f(x) + \eta$ at different iterations, are slowly pushed below the red curve in the course of the optimization.

\begin{figure}[t]
    \centering
    \subfloat[Regularization Parameters]{
        \centering
        \includegraphics[width=.49\linewidth]{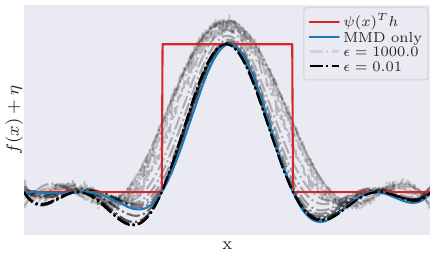}
    }
    \subfloat[Annealing to Approximate the (Semi-)Infinite Constraint]{
        \centering
        \includegraphics[width=.49\linewidth]{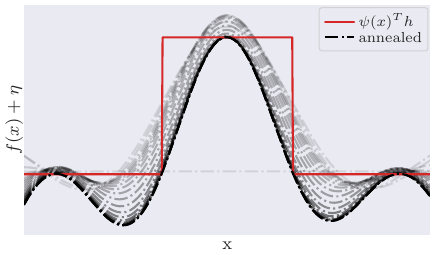}
    }
    \caption{Effect of Entropy Regularization.
    Figure a) shows the effect of entropy regularization for fixed parameters $\epsilon$. The gray lines correspond to logarithmically decreasing values of $\epsilon$ between $1000$ and $0.01$.
    Figure b) shows the annealing procedure for entropy regularization, where the shaded curves show the intermediate progress of the optimization.
    }
    \label{fig:reg_effect}
\end{figure}

\subsection{Choice of Reference Measure} \label{app:reference-measure}
The KMM estimator based on \eqref{eq:dual-profile} and \eqref{eq:functional-kgel} respectively requires a choice of distribution $Q$ that enters the reference distribution $\omega = (1 - \alpha) \hat{P}_n + \alpha Q$ to define the entropy regularizer. As the candidate distributions are required to admit a density with respect to $\omega$, the choice of $Q$ directly determines the class of distributions considered in the minimization over $P$. Optimally $Q$ should be chosen as close as possible to the population distribution $P_0$. As generally $P_0$ is unknown, in the following we discuss several (data-driven) choices for $Q$. 

\paragraph{Lebesgue Measure}
Choosing $Q$ as the Lebesgue measure or the uniform distribution over $\mathcal{X} \times \mathcal{Z}$ respectively, allows for considering arbitrary distributions on $\mathcal{X} \times \mathcal{Z}$. At the same time, this choice corresponds to an uninformative prior which discards the information contained in the sample and does not converge to the population distribution as $n\rightarrow \infty$. Empirically the Lebesgue measure did not show competitive performance.

\paragraph{Empirical Distribution}
The empirical distribution converges to the population distribution as $n \rightarrow \infty$, and therefore it provides a viable candidate as reference distribution. However, as the empirical distribution is a discrete distribution supported on the samples, the considered candidate distributions are again reweightings of the data and thus some of the competitive advantage of KMM over other method of moments estimators is lost. Note however, that MMD provides different gradient information compared to $\varphi$-divergences as used in GEL/GMM and therefore the obtained estimator will still be different.

\paragraph{Kernel Density Estimation}
In order to combine the strength of considering continuous candidate distributions in \eqref{eq:dual-profile} with the information contained in the empirical distribution, one can represent the reference distribution $Q$ by a kernel density estimator (KDE) \citep{rosenblatt1956remarks} trained on the empirical sample. This allows us to sample from a continuous distribution in Algorithm~\ref{alg:training} and thus taking into account candidate distributions with support different from the empirical distribution, while still converging to the population distribution as $n \rightarrow \infty$. Representing $Q$ by a KDE proved to be the most effective choice in practice.

\paragraph{Modern Machine Learning Models}
As a straight-forward extension of the KDE approach, one can represent $Q$ by any density estimator from which one can sample and can thus leverage the potential of modern machine learning approaches like generative adversarial networks~\citep{goodfellow2014generative}, variational auto-encoders~\citep{kingma2022autoencoding}, normalizing flows~\citep{papamakarios2021normalizing} or diffusion models~\citep{ho2020denoising}. This seems particularly promising for complex high-dimensional data, where KDE estimators become increasingly inaccurate. Note however, that while better density estimators most likely improve the finite sample performance of our estimator, the role of $Q$ is to define the class of (continuous) candidate distributions via its support. As long as $P_0 \ll Q$, we can find a $P$ arbitrarily close to $P_0$ and better choices of $Q$ (closer to $P_0$) mostly only facilitate finding these.

\paragraph{Time Evolution of $Q$ via Primal-Dual Schemes}
Instead of using a fixed choice of reference distribution $Q$, one could choose the reference distribution adaptively over the course of the optimization via primal-dual schemes. To this aim, take $P^0= \hat{P}_n$ and consider for timesteps $k=1,\ldots,T$ problem \eqref{eq:mmd-profile} as
\begin{align}
     R^\varphi_\epsilon(\theta) = & \inf_{P \in \mathcal{P}} \frac{1}{2}  \mmd(P,\hat{P}_n)^2 + \epsilon D(P||P_k) \quad  \mathrm{s.t.} \quad E_P[\psi(X;\theta)] =0, \quad E_P[1] =1.
\end{align}
Carrying out the optimization over $\mu \in \mathcal{F}$ and defining $\beta = (\eta, f , h) \in \mathcal{M}$, the Lagrangian of the MMD profile~\eqref{eq:mmd-profile} can be cast in the (semi-dual) form,
\begin{align}
    \label{lagr-semi-dual}
    L(P, \beta) = \frac{1}{n} \sum_{i=1}^n f(x_i) + \eta -\frac{1}{2} \|f \|_\mathcal{H}^2
    + \int \left( - f(x) - \eta + \psi(x;\theta)^T h \right) dP(x) + \epsilon D_\varphi(P || P_{k})
\end{align}
Now, instead of the dual approach used in our dual MMD profile, one could consider a primal dual update where one alternates between updates of the primal measure $P$ via a proximal term using a Bregman divergence $D$
\begin{align}
    \label{eq-P-update-bregman}
    P_{k+1} \in \argmin_P 
         \int \left( - f(x) - \eta + \psi(x;\theta)^T h \right) dP(x) + \epsilon D(P\| P_{k}).
\end{align}
and subsequently update the dual variables $\beta$, where we solve 
\begin{align}
    \label{eq-rkhs-update}
    \beta_{k+1}
    \in \argmax_{\beta \in \mathcal{M}}
    \frac{1}{n} \sum_{i=1}^n f(x_i) + \eta - \frac{1}{2} \|f \|_\mathcal{H}^2 + \int \left( - f(x) - \eta + \psi(x;\theta)^T h \right) dP_{k+1}(x) + \frac{1}{2t}\|\beta - \beta_k\|^2_\mathcal{M}.
\end{align}
The update rule~\eqref{eq-rkhs-update} is a standard convex optimization problem.
We leave a specific implementation of this approach to future work.
\newpage
\section{Proofs}\label{sec:proofs}
\subsection{Definitions and Preliminaries}
To simplify notation and analysis we first provide a compact formulation of the functional KMM objective \eqref{eq:functional-kgel} given by
\begin{align*}
    \widehat{G}_{\epsilon,\lambda_n}(\theta,\beta) = \frac{1}{n} \sum_{i=1}^n f(x_i, z_i) + \eta -\frac{1}{2} \|f \|_{\mathcal{F}}^2 - \epsilon \intxz{ \!\!\!\!\!\!\! \varphi^\ast \left(\frac{f(x,z) + \eta - \langle \Psi(x,z;\theta), h \rangle_\mathcal{H}}{\epsilon} \right) } - \frac{\lambda_n}{2} \|h \|_{\mathcal{H}}^2,
\end{align*}
where $\beta = (\eta, f, h)$.
As $\mathcal{F}$ is an RKHS of functions $\mathcal{X} \times \mathcal{Z} \rightarrow \mathbb{R}$, the evaluation functional in $\mathcal{F}$ is given by $k\left( (x,z),\cdot\right) : \mathcal{F} \rightarrow \mathbb{R}$, such that $\langle k( (x,z), \cdot), f \rangle_\mathcal{F} = f(x,z) \ \forall f \in \mathcal{F}$ and $(x,z) \in \mathcal{X} \times \mathcal{Z}$. Let $\mathcal{M} : = \mathbb{R} \times \mathcal{F} \times \mathcal{H}$. For $\beta = (\eta, f, h) \in \mathcal{M}$ define a norm on $\mathcal{M}$ as $\| \beta \|_\mathcal{M} = \sqrt{|\eta |^2 +  \| f \|^2_\mathcal{F} +  \| h \|^2_\mathcal{H}}$. Define for $i=1,\ldots,n$,
\begin{align}
    b_i = \begin{pmatrix}
    1 \\
    k((x_i,z_i),\cdot) \\
    0
    \end{pmatrix} \in \mathcal{M}, \quad
    a(x, z ;\theta) = \begin{pmatrix}
    1 \\
    k((x,z), \cdot) \\
    - \Psi(x, z;\theta)
    \end{pmatrix} \in \mathcal{M}, \quad
    R_{\lambda} = \begin{pmatrix}
    0 & & \\
    & I & \\
    & & \lambda I
    \end{pmatrix}
    \in \mathcal{M \times M}.
\end{align}
where we used that we can identify $\mathcal{M}^\ast$ with $\mathcal{M}$ by the self-duality property of Hilbert spaces.
Then the functional KMM objective~\eqref{eq:functional-kgel} can be written in the compact form
\begin{align}
    G_{\epsilon,\lambda_n}(\theta, \beta) = \frac{1}{n} \sum_{i=1}^n b_i^T \beta - \epsilon \intxz{ \varphi^\ast\left( \frac{1}{\epsilon} a(x,z;\theta)^T \beta \right) } - \frac{1}{2} \beta^T R_{\lambda_n} \beta. \label{eq:compact}
\end{align}
Analogously for the objective of the finite dimensional KMM estimator \eqref{eq:dual-profile} we have $\mathcal{H} = \mathbb{R}^m$ and $\mathcal{M} = \mathbb{R} \times \mathcal{F} \times \mathbb{R}^m$ and further define for $i=1,\ldots,n$,
\begin{align}
    b_i = \begin{pmatrix}
        1 \\
        k(x_i,\cdot) \\
        0
    \end{pmatrix} \in \mathcal{M}, \quad
    a(x;\theta) = \begin{pmatrix}
        1 \\
        k(x, \cdot) \\
        - \psi(x;\theta)
    \end{pmatrix} \in \mathcal{M}, \quad
    R = \begin{pmatrix}
        0 & & \\
        & I & \\
        & & 0
    \end{pmatrix}
    \in \mathcal{M \times M}.
\end{align}
Then the unconditional KMM objective~\eqref{eq:dual-profile} can be written in the compact form
\begin{align}
    \widehat{G}_{\epsilon}(\theta, \beta) = \frac{1}{n} \sum_{i=1}^n b_i^T \beta - \epsilon \intx{ \varphi^\ast\left( \frac{1}{\epsilon} a(x;\theta)^T \beta \right) } - \frac{1}{2} \beta^T R \beta.
    \label{eq:compact-finite}
\end{align}
In the proofs we will consider derivatives of the KMM objective with respect to the dual parameters $\beta \in \mathcal{M}$, the second and the third component of which live in function spaces $\mathcal{F}$ and $\mathcal{H}$ respectively. We define the corresponding functional derivative as follows.
\begin{definition}[Functional Derivative] \label{def:func-dev}
    Let $\mathcal{H}$ be a vector space of functions. For a functional $G: \mathcal{H} \rightarrow \mathbb{R}$ and a pair of functions $h_0, h_1 \in \mathcal{H}$, we define the derivative operator $\frac{\partial}{\partial h} G(h_0)$ at $h_0$ via $\frac{\partial}{\partial h} G(h_0)(h_1) = \left.\frac{d}{d t} G(h_0+t h_1)\right|_{t=0}$. Likewise, we define the $k$-th functional derivative $\frac{\partial^k}{(\partial h)^k}G(h_0)$ at $h_0$ via
    \begin{align}
    \frac{\partial^k}{(\partial h)^k}G(h_0)\left(h_{1}, \ldots, h_{k}\right) 
    = \left.\frac{\partial^{k}}{\partial t_{1} \ldots \partial t_{k}} G\left(h_0 + t_{1} h_{1}+\ldots+t_{k} h_{k}\right)\right|_{t_{1}=\cdots=t_{k}=0}. 
    \end{align}
    Moreover, we write $\frac{\partial^k}{(\partial h)^k}G(h_0) = 0$ as a shorthand for $\frac{\partial^k}{(\partial h)^k}G(h_0)\left(h_{1}, \ldots, h_{k}\right) =0$ for all $h_{1}, \ldots, h_{k} \in \mathcal{H}$.
    Similarly, when considering a vector-valued function of a vector-valued parameter, $G: \Theta \subseteq \mathbb{R}^p \rightarrow \mathbb{R}^m$, we denote the $k$-th standard directional derivative at $\theta_0 \in \Theta$ as $\frac{\partial^k}{(\partial \theta)^k} G(\theta_0) \in \mathbb{R}^{p\times m}$ and in the case $k=1$ we write the Jacobian as $\frac{\partial}{(\partial \theta)} G(\theta_0) = \left( \nabla_\theta G \right)(\theta_0) =: \nabla_\theta G(\theta_0)$.
\end{definition} 
Additionally we will make use of the functional version of Taylor's theorem with Lagrange remainder, which we state here for completeness.
\begin{proposition}[Taylor's Theorem] \label{prop:taylor}
    Let $G: \mathcal{H} \rightarrow \mathbb{R}$, where $\mathcal{H}$ is a vector space of functions. For any $h, h^{\prime} \in \mathcal{H}$, if $t \mapsto G\left(t  h + (1-t)  h^{\prime}\right)$ is $(k+1)$-times differentiable over an open interval containing $[0,1]$, then there exists $\bar{h} \in \operatorname{conv}\left(\left\{h, h^{\prime}\right\}\right)$ such that
    \begin{align}
        G\left(h^{\prime}\right) = & \  G(h) + \sum_{i=1}^{k} \frac{1}{i !} \frac{\partial^i}{(\partial h)^i} G(h)(\underbrace{h^{\prime}-h, \ldots, h^{\prime}-h}_{i \text { times }} \\ 
        &+ \frac{1}{(k+1) !} \frac{\partial^{k+1}}{(\partial h)^{k+1}} G(\bar{h})(\underbrace{h^{\prime}-h, \ldots, h^{\prime}-h}_{k+1 \text { times }}) .
    \end{align}
\end{proposition}
\newpage
\subsection{Duality Results}
\subsubsection{Proof of Theorem~\ref{th:unreg-duality}}
\begin{proof}
Let $\mu_{\hat{P}_n} = E_{\hat{P}_n}[k(X,\cdot)] = \avg k(x_i,\cdot)$ denote the kernel mean embedding of the empirical distribution.
Instead of working with the measure $P$ directly, we introduce an auxiliary variable $\mu\in\mathcal F$, which serves as the kernel mean embedding of $P$, so we can write the MMD profile as
\begin{align}
    R(\theta) = \inf_{P \in \mathcal{P}, \mu \in \mathcal{H}} \frac{1}{2} \|\mu - \mu_{\hat{P}_n} \|_\mathcal{F}^2 \quad \mathrm{s.t.} \quad \int k(x, \cdot) dP(x) = \mu, \quad \int dP(x) =1, \quad \int \psi(x;\theta) dP(x) =0.
\end{align}
Introducing Lagrange parameters $\eta \in \mathbb{R}$, $f \in \mathcal{F}$ and $h \in \mathbb{R}^m$ we can define the Lagrangian as
\begin{align}
    L(P, \mu, \eta, f, h) &= \frac{1}{2} \|\mu - \mu_{\hat{P}_n} \|_\mathcal{H}^2 + \langle f, \mu - \int k(x, \cdot) dP(x) \rangle_{\mathcal{F}} + \eta \left( 1 - \int dP(x) \right) + \langle h, \int \psi(x;\theta) dP(x) \rangle_{\mathbb{R}^m} \\
    & = \frac{1}{2} \|\mu - \mu_{\hat{P}_n} \|_\mathcal{F}^2 + \int \left( - f(x) - \eta + \langle h, \psi(x;\theta) \rangle_{\mathbb{R}^m}\right) dP(x) + \langle f, \mu \rangle_{\mathcal{F}} + \eta.
\end{align}
Now as we minimize the Lagrangian with respect to all positive measures $P$, this only yields a finite expression as long as $- f(x) - f_0 + \langle h, \psi(x;\theta) \rangle_{\mathbb{R}^m} \geq 0$ $\forall x \in \mathcal{X}$. This directly translates into a semi-infinite constraint and the problem becomes
\begin{align}
    &\sup_{f_0,f,h} \inf_{\mu \in \mathcal{F}} \frac{1}{2} \|\mu - \mu_{\hat{P}_n} \|_\mathcal{F}^2 + \langle f, \mu \rangle_{\mathcal{F}} + f_0 \\
    &\mathrm{s.t.} \quad \langle h, \psi(x;\theta) \rangle_{\mathbb{R}^m} \geq f(x) + f_0 \quad \forall x \in \mathcal{X}.
 \end{align}
Now, the first order optimality conditions for $\mu$ yield (see subsection Optimality Condition in $\mu$ in the following for details)
\begin{align}
    \mu = \mu_{\hat{P}_n} - f,
\end{align}
and reinserting yields the final dual problem
\begin{align}
    &\sup_{f_0,f,h} \frac{1}{n} \sum_{i=1}^n f(x_i) -\frac{1}{2} \|f \|_\mathcal{H}^2 + f_0 \\
    &\mathrm{s.t.} \quad \langle h, \psi(x;\theta) \rangle_{\mathbb{R}^m} \geq f(x) + f_0 \quad \forall x \in \mathcal{X}.
    \label{eq:constrained-dual}
    \end{align}
Strong duality holds trivially as the primal problem only contains equality constraints.
\end{proof}
\subsubsection{Proof of Theorem~\ref{th:kel-duality}}
\begin{proof}
Let again $\mu_{\hat{P}_n} = E_{\hat{P}_n}[k(X,\cdot)] = \avg k(x_i,\cdot)$ denote the KME of the empirical distribution.
Following the proof of Theorem~\ref{th:unreg-duality} we can write the entropy regularized MMD profile as
\begin{align*}
    R_\epsilon(\theta) = \inf_{P \ll \omega, \mu \in \mathcal{H}} \frac{1}{2} \|\mu - \mu_{\hat{P}_n} \|_\mathcal{F}^2 + \epsilon D_\varphi(P||\omega) \quad \mathrm{s.t.} \quad \int k(x, \cdot) dP(x) = \mu, \quad \int dP(x) =1, \quad \int \psi(x;\theta) dP(x) = 0.
\end{align*}
Let $p(x)$ denote the density of $P$ with respect to the reference measure $\omega$.
Introducing dual variables $\eta \in \mathbb{R}$, $f \in \mathcal{F}$ and $h \in \mathbb{R}^m$, the Lagrangian of the problem can be obtained as
\begin{align}
    L(p, \mu, f, h) =& \frac{1}{2} \|\mu - \mu_{\hat{P}_n} \|_\mathcal{F}^2 + \epsilon \intx{ \varphi(p(x)) }  + \langle f, \mu  - \intx{ k(x, \cdot) p(x) } \rangle_{\mathcal{F}} \\
    &+ \eta \left(1 - \intx{ p(x) } \right) + \langle h, \intx{ \psi(x;\theta) p(x) } \rangle_{\mathbb{R}^m}.
\end{align}
Now, collecting terms containing $p$ we get
\begin{align}
    L(p, \mu, f, h) =& \frac{1}{2} \|\mu - \mu_{\hat{P}_n} \|_\mathcal{F}^2  + \langle f, \mu \rangle_{\mathcal{F}} + \eta - \epsilon \intx{ \left( \frac{f(x) + \eta - \langle h, \psi(x;\theta) \rangle_{\mathbb{R}^m}}{\epsilon} p(x)  - \varphi(p(x)) \right) } .
\end{align}
The dual formulation follows from minimizing the Lagrangian with respect to the primal variables $\mu \in\mathcal{F}$ and $p \in \Pi(\omega)$, where $\Pi(\omega)$ denotes the set of all densities with respect to $\omega$. Taking the infimum of $L$ with respect to $p$ we obtain
\begin{align}
    & \inf_{p \in \Pi(\omega)} \frac{1}{2} \|\mu - \mu_{\hat{P}_n} \|_\mathcal{F}^2  + \langle f, \mu \rangle_{\mathcal{F}} + \eta - \epsilon \intx{ \left( \frac{f(x) + \eta - \langle h, \psi(x;\theta) \rangle_{\mathbb{R}^m}}{\epsilon} p(x)  - \varphi(p(x)) \right) }  \\
    =& \frac{1}{2} \|\mu - \mu_{\hat{P}_n} \|_\mathcal{F}^2  + \langle f, \mu \rangle_{\mathcal{F}} + \eta -  \epsilon  \sup_{p \in \Pi(\omega)}\intx{ \left( \frac{f(x) + \eta - \langle h, \psi(x;\theta) \rangle_{\mathbb{R}^m}}{\epsilon} p(x)  - \varphi(p(x)) \right) } \\
    =& \frac{1}{2} \|\mu - \mu_{\hat{P}_n} \|_\mathcal{F}^2  + \langle f, \mu \rangle_{\mathcal{F}} + \eta -  \epsilon  \intx{ \sup_{t \in \mathbb{R}_+} \left( \frac{f(x) + \eta - \langle h, \psi(x;\theta) \rangle_{\mathbb{R}^m}}{\epsilon} t  - \varphi(t) \right) } \\
    =& \frac{1}{2} \|\mu - \mu_{\hat{P}_n} \|_\mathcal{F}^2  + \langle f, \mu \rangle_{\mathcal{F}} + \eta - \epsilon \intx{ \varphi^{\ast}\left(\frac{f(x) + \eta - \langle h, \psi(x;\theta) \rangle_{\mathbb{R}^m}}{\epsilon}\right)},
\end{align}
where we used the definition of the Fenchel conjugate function $\varphi^\ast(q) = \sup_{p} \langle q, p \rangle - \varphi(p)$.
In the third line we used that as $p: \mathcal{X} \rightarrow \mathbb{R}_{+}$ is an arbitrary function, we can swap the supremum outside the integral for a pointwise supremum over $t := p(x)$ for each $x\in \mathcal{X}$ in the integral.
Now, the first order optimality conditions for the function $\mu$ yield (refer to the following subsection for details)
\begin{align}
    \mu = \mu_{\hat{P}_n} - f.
\end{align}
Inserting this back into the Lagrangian we get
\begin{align}
L(f,\eta, h) & = \frac{1}{n} \sum_{i=1}^n f(x_i) + \eta -\frac{1}{2} \|f \|_\mathcal{H}^2 
- \epsilon \intx{ \varphi^{\ast}\left(\frac{f(x) + \eta - \langle h, \psi(x;\theta) \rangle_{\mathbb{R}^m}}{\epsilon}\right) }, \label{eq:ent_reg_dual_obj}
\end{align}
from which the dual program follows. Strong duality follows trivially as the primal problem only contains equality constraints. In order to show convexity, consider the compact notation \eqref{eq:compact-finite},
\begin{align}
    G_{\epsilon}(\theta, \beta) = \frac{1}{n} \sum_{i=1}^n b_i^T \beta - \epsilon \intx{ \varphi^\ast \left( \frac{1}{\epsilon} a(x)^T \beta \right) } - \frac{1}{2} \beta^T R \beta.
\end{align}
The first term is linear in $\beta$ and thus trivially concave. The second term is concave as by definition the Fenchel conjugate of any function is convex (and thus its negative concave) and the composition of a concave function with a linear function yields a concave function. Finally the third term is a negative semi-definite quadratic form for any $\lambda_n \geq 0$ and thus concave. As an unconstrained maximization over a jointly concave objective the optimization over the dual parameters is a convex program.
\end{proof}
\paragraph{Optimality condition in $\mu$}
It is easily verified that the functional
\begin{align}
    \frac{1}{2} \|\mu - \mu_{\hat{P}_n} \|_\mathcal{F}^2 + \langle f, \mu \rangle_{\mathcal{F}} 
\end{align}
is (strongly) convex in $\mu$.
In fact, its minimizer can be seen by a straightforward manipulation of the terms
\begin{align}
    \frac{1}{2} \|\mu - \mu_{\hat{P}_n} \|_\mathcal{F}^2 + \langle f, \mu \rangle_{\mathcal{F}}
    =&\frac{1}{2} \|\mu - \mu_{\hat{P}_n} \|_\mathcal{F}^2  
    +
    \langle f, \mu  - \mu_{\hat{P}_n} \rangle_{\mathcal{F}}
    +
    \frac12 \|f\|_\mathcal{F}^2  
    +
    \langle f, \mu_{\hat{P}_n} \rangle_{\mathcal{F}}
    -
    \frac12 \|f\|_\mathcal{F}^2  
    \\
    =&
    \frac{1}{2} \|\mu - \mu_{\hat{P}_n} +f\|_\mathcal{F}^2    
    +
    \langle f, \mu_{\hat{P}_n} \rangle_{\mathcal{F}}
    -
    \frac12 \|f\|_\mathcal{F}^2  
     \\
    \geq&
    \langle f, \mu_{\hat{P}_n} \rangle_{\mathcal{F}}
    -
    \frac12 \|f\|_\mathcal{F}^2,
\end{align}
where the optimum is attained at
$\mu = \mu_{\hat{P}_n} - f$.
Alternatively, we can also characterize the optimality condition via the differentiability structure.
Since $\mathcal{F}$ is a normed space, we use $\nabla G(\mu)$ to denote the Fr\'echet derivative of a functional $G$.
Suppose $\mu$ is a minimizer of the problem
\begin{align}
    \min_{\mu'} 
    \
    \left\{
        G(\mu'):=
    \langle f, \mu' \rangle_{\mathcal{F}}
   + \frac{1}{2} \|\mu' - \mu_{\hat{P}_n} \|_\mathcal{F}^2  
   \right\}.
\end{align}
Then 
\(
    \nabla 
    G(\mu)
    = 0
\) and a
straightforward calculation yields 
$\mu = \mu_{\hat{P}_n} - f$.
\newpage
\subsection{Asymptotic Properties of KMM for Conditional Moment Restrictions} \label{sec:proofs-cmr}
The asymptotic properties of the KMM estimator for conditional moment restrictions follow from phrasing the conditional moment restrictions \eqref{eq:conditional-mr} as functional moment restrictions of the form \eqref{eq:continuum-mr} over a sufficiently rich Hilbert space of functions. In the following we show that the assumptions of Theorem~\ref{th:consistency-cmr} suffice to fulfill the assumptions of the theorems for the functional KMM estimator (Theorems~\ref{th:consistency-fmr} and \ref{th:asymptotic-normality}) from which the results follow. The proofs for the functional case are deferred to Section~\ref{sec:proofs-fmr}.

\subsubsection{Proof of Theorem~\ref{th:consistency-cmr} (Consistency for CMR)}
\begin{lemma}\label{lemma:invertible}
    For a moment function $\psi(x;\theta)$ taking values in $\mathbb{R}^m$ define the \emph{conditional covariance matrix} $V_0(Z) = E[\psi(X;\theta_0) \psi(X;\theta_0)^T | Z]$ as a function of the conditioning random variable $Z$ taking values in $\mathcal{Z}$. Let $\mathcal{H}$ be a Hilbert space of square integrable functions equipped with the norm $h \mapsto \|h \|_{L^2(\mathcal{H}, P_0)}= \left(\int_\mathcal{Z} \| h(z)\|_2^2 \ P_0(\mathrm{d}z)\right)^{1/2}$. Define the moment functional $\Psi(x,z;\theta): \mathcal{H} \rightarrow \mathbb{R}$ such that $\Psi(x,z;\theta)(h) = \psi(x;\theta)^T h(z)$ for any $(x,z) \in \mathcal{X} \times \mathcal{Z}$, $\theta \in \Theta$ and $h \in \mathcal{H}$. Then the covariance operator $\Omega_0: \mathcal{H} \rightarrow \mathcal{H}$ defined as
    \begin{align}
        \Omega_0 = E[\Psi(X,Z;\theta_0) \otimes \Psi(X,Z;\theta_0)]
    \end{align}
    is non-singular if $V_0(Z)$ is non-singular with probability 1.
\end{lemma}
\begin{proof}
    Note that $\Omega_0$ is non-singular if $\| \Omega_0 h \|_{L^2(\mathcal{H}, P_0)} > 0$ for any $h \in \mathcal{H}$ with $\| h \|_{L^2(\mathcal{H}, P_0)} > 0$, or equivalently if $\langle h, \Omega_0 h \rangle \neq 0$.
    Consider any $h \in \mathcal{H}$ with $\|h \|_{L^2(\mathcal{H}, P_0)} > 0$, then by the law of iterated expectation we have
    \begin{align}
        \langle h, \Omega_0 h \rangle_\mathcal{H} &= E\left[ \langle \Psi(X,Z;\theta_0)(h), \Psi(X,Z;\theta_0)(h) \rangle \right] \\
        &= E\left[\left(\psi(X;\theta_0)^T h(Z) \right)^T \left( \psi(X;\theta_0)^T h(Z)\right) \right] \\
        &= E \left[ h(Z)^T E[\psi(X;\theta_0) \psi(X;\theta_0)^T | Z] h(Z) \right]   \\
        &= \int h(z)^T V_0(z) h(z) \mathrm{d}P_0(z)
    \end{align}
    Now, $V_0(Z)$ is a positive-semi definite matrix by construction and non-singular $P_0\text{-a.s.}$ by assumption and thus its smallest eigenvalue $C$ is bounded away from zero. Therefore we have
    \begin{align}
         \langle h, \Omega_0 h \rangle_\mathcal{H} \geq C \int \|h(z) \|_2^2 \mathrm{d}P_0(z) = C \|h \|_{L^2(\mathcal{H}, P_0)}^2 > 0
    \end{align}
    and thus $\Omega_0$ is non-singular with smallest eigenvalue bounded away from zero. 
\end{proof}

\begin{lemma}\label{lemma:non-singularity}
Let the assumptions of Theorem~\ref{th:consistency-cmr} be satisfied and define for any $(x,z,\theta) \in \mathcal{X} \times \mathcal{Z} \times \Theta$ the moment functional $\Psi(x,z;\theta): \mathcal{H} \rightarrow \mathbb{R}$ with $\Psi(x,z;\theta)(h) = \psi(x;\theta)^T h(z)$.
Then the matrix $\Sigma_0 = \left\langle E[ \nabla_\theta \Psi(X,Z;\theta_0)], E[ \nabla_\theta \Psi(X,Z;\theta_0)] \right\rangle_{\mathcal{H}^\ast} \in \mathbb{R}^{p \times p}$ is strictly positive definite and non-singular with smallest eigenvalue bounded away from zero.
\end{lemma}
\begin{proof}
    By definition we have $\mathcal{H} = \bigoplus_{i=1}^m \mathcal{H}_i$ and thus $\mathcal{H}^\ast = \bigoplus_{i=1}^m \mathcal{H}_i^\ast$. For each $i \in \{1,\ldots, m\}$ let $\{h_j^i\}_{j=1}^\infty$ denote an orthonormal basis of $\mathcal{H}^\ast_i$ such that $\langle h_i^k, h_j^l \rangle = \delta_{ij} \delta_{kl}$. Then the identity operator in $\mathcal{H}^\ast$ can be expressed as $I_{\mathcal{H}^\ast} = \sum_{i=1}^m \sum_{j=1}^\infty h_j^i \left(h_j^i\right)^\ast$, where $\left(h_j^i\right)^\ast \in \mathcal{H}^{\ast\ast}$ can be uniquely identified with an element in $\mathcal{H}$ by the property of Hilbert spaces. In the following, we overload notation and denote with $h_j^i$ also the Riesz representer of $h_j^i \in \mathcal{H}^\ast$ in $\mathcal{H}$ which is uniquely identified by the self-duality property of Hilbert spaces.
    Consider any $\theta \in \Theta$ with $0 < \|\theta \| < \infty$ then
    \begin{align}
        \theta^T \Sigma_0 \theta = & \left\langle E[ \theta^T \nabla_\theta \Psi(X,Z;\theta_0)], E[ \theta^T \nabla_\theta \Psi(X,Z;\theta_0)] \right\rangle_{\mathcal{H}^\ast} \\
        =&  \left\langle E[\theta^T \nabla_\theta \Psi(X,Z;\theta_0)], 
        \left( \sum_{i=1}^m \sum_{j=1}^\infty h_j^i \left( h_j^i \right)^\ast \right)        
        E[\theta^T \nabla_\theta \Psi(X,Z;\theta_0)] \right\rangle_{\mathcal{H}^\ast} \\
        =& \sum_{i=1}^m \sum_{j=1}^\infty \left(E[\theta^T  \nabla_\theta \psi_i(X;\theta_0) h^i_j(Z)]\right)^2  \\
        =& \sum_{i=1}^m \sum_{j=1}^\infty \left( E[  \theta^T D^i_0(Z) h^i_j(Z) ] \right)^2, \label{eq:non-sing}
    \end{align}
    where $D_0^i(z) = E[\nabla_\theta \psi_i(X;\theta_0)|Z=z] \in \mathbb{R}^p$ denotes the columns of $D_0(z) = E[\nabla_\theta \psi(X;\theta_0)|Z=z] \in \mathbb{R}^{p \times m}$. Now as $\operatorname{rank}(D_0(Z)) = p$ w.p.1 by Assumption~j), the $p$ rows of $D_0(Z)$ are linearly independent w.p.1 which means that for any $\theta \in \Theta$ with $ 0 < \| \theta \| < \infty$ there exists $s \in \{1,\ldots,m\}$ such that $\theta^T D_0^s(Z) \neq 0$ w.p.1.
    Now, by assumption the function space $\mathcal{H}$ is chosen such that we have equivalence between the conditional and variational/functional forms of the moment restrictions, i.e., for any continuous function $\rho$ we have $E[\rho(X;\theta)^T h(Z)] = 0\ \forall h \in \mathcal{H}$ if and only if $E[\rho(X;\theta)|Z]=0$ w.p.1. In particular this implies $E[\theta^T \nabla_\theta \psi_s(X;\theta) h_s(Z)] = 0\ \forall h_s \in \mathcal{H}_s$ if and only if $E[\theta^T \nabla_\theta \psi_s(X;\theta)|Z]= \theta^T D_0^s(Z) = 0$ w.p.1. 
    As $\theta^T D_0^s(Z) \neq 0$ w.p.1 this means there must exist $h^s \in \mathcal{H}_s$ such that $E[\theta^T D_0^s(Z) h^s(Z)] \neq 0$. As we can expand any $h^s \in \mathcal{H}_s$ in terms of an orthonormal basis $\{h^s_k\}_{k=1}^\infty$ of $\mathcal{H}_s$ as $h = \sum_{k=1}^\infty \alpha_k h_k^s$, there must exist at least one $r \in \mathbb{N}$ with $\alpha_r \neq 0$ and $E[D_0^s(Z) h_r^s(Z)] \neq 0$.
    Inserting this back into \eqref{eq:non-sing} we get
    \begin{align}
         \theta^T \Sigma_0 \theta = &\sum_{i=1}^m \sum_{j=1}^\infty \left( E[  \theta^T D^i_0(Z) h^i_j(Z) ] \right)^2 \\
         \geq& \left( E[\theta^T D_0^s(Z) h_r^s(Z)] \right)^2 > 0.
    \end{align}
   From this it follows that $\Sigma_0$ is non-singular with probability 1.
\end{proof}

\paragraph{Proof of Theorem~\ref{th:consistency-cmr}}
\begin{proof}
    By definition the function space $\mathcal{H}$ is expressive enough such that we can express the conditional moment restriction $E[\psi(X;\theta)|Z]=0 \ P_{Z}\text{-a.s.}$ in functional form as
    \begin{align}
        E[\Psi(X,Z;\theta)] = 0 \in \mathcal{H}^\ast.
    \end{align}
    It remains to be shown that the assumptions imposed on $\psi$ are sufficient for $\Psi$ to fulfill the conditions of Theorem~\ref{th:consistency-fmr}.
    Assumptions a) and b) directly translate to the corresponding assumptions in Theorem~\ref{th:consistency-fmr}. 
    Assumption~c) of Theorem~\ref{th:consistency-cmr} follows directly from Assumption~c) as $\Psi(X,Z;\theta)(h) = \psi(X;\theta)^T h(Z)$ is continuous in $\theta$ for any $\theta \in \Theta$ if $\psi(X;\theta)$ is continuous in $\theta$ for any $\theta \in \Theta$. As this holds for any $h \in \mathcal{H}$ continuity of $\Psi(X,Z;\theta)$ in $\theta$ follows.
    Assumption~d) for Theorem~\ref{th:consistency-fmr} follows as
    \begin{align}
        &E[\sup_{\theta \in \Theta} \|\Psi(X,Z;\theta) \|_{\mathcal{H}^\ast}^2] \\
        &= E[\sup_{\theta \in \Theta} \sup_{h \in \mathcal{H}, \|h \|\leq 1} \| \psi(X;\theta)^T h(Z) \|^2 \\ 
        &\leq E[ E[\sup_{\theta \in \Theta} \|\psi(X;\theta)\|_2^2 \ | Z] \sup_{h \in \mathcal{H}, \|h \|\leq 1} \| h(Z) \|_2^2] \\ 
        &\leq C \int_{\mathcal{Z}} \sup_{h \in \mathcal{H}, \| h \| \leq 1} \| h(z)\|_2^2 \ P_0(\mathrm{d}z)
    \end{align}
    where we used that $E[\sup_{\theta \in \Theta} \|\psi(X;\theta)\|_2^2 |Z] \leq C$ with probability $1$ by Assumption~d).
    Now for any function $h \in \mathcal{H}$ with $\|h \|_{L^2(\mathcal{H}, P_0)} \leq 1$ we must have that $h(Z) < \infty$ w.p.1 and thus by the local Lipschitz property it follows $h(z) \leq M < \infty$ for any $z \in \operatorname{supp}(P_0)$. As this holds for any $h \in \mathcal{H}$, in particular it also holds for the supremum over $\mathcal{H}$ and thus $\sup_{h \in \mathcal{H}, \| h\|\leq 1} \| h(z)\|_2^2 \leq M$ $\forall z \in \mathcal{Z}$. Therefore, we obtain
    \begin{align}
        E[\sup_{\theta \in \Theta} \|\Psi(X,Z;\theta) \|_{\mathcal{H}^\ast}^2] \leq C M \int_{\mathcal{Z}} P_0(\mathrm{d}z) = CM < \infty.
    \end{align}
    Assumption~e) of Theorem~\ref{th:consistency-fmr} follows from Assumption~e) and Lemma~\ref{lemma:invertible}.
    Assumption f) is identical to the corresponding Assumption~f) in Theorem~\ref{th:consistency-fmr} using the same argument as for Assumption~d) for the integrability condition. 
    Finally, Assumption~g) of Theorem~\ref{th:consistency-fmr} is identical to Assumption~g).
    Therefore Assumptions a)-g) of Theorem~\ref{th:consistency-fmr} are fulfilled and it follows that $\hat{\theta} \overset{p}{\rightarrow} \theta_0$.
    Now further, Assumption~h) of Theorem~\ref{th:consistency-fmr} is identical with Assumption~h). Assumption~i) of Theorem~\ref{th:consistency-fmr} follow from Assumption~i) by the same argument presented earlier for Assumption~c) and d) of Theorem~\ref{th:consistency-fmr}.
    Finally Assumption~j) of Theorem~\ref{th:consistency-fmr} follows from Assumption~j) by Lemma~\ref{lemma:non-singularity}.
    Therefore, Assumptions h)-j) of Theorem~\ref{th:consistency-fmr} are fulfilled and we have $\|\hat{\theta} - \theta_0 \| = O_p(n^{-1/2})$.
\end{proof}
\subsubsection{Proof of Theorem~\ref{th:asymptotic-normality-cmr} (Asymptotic Normality for CMR)}
The asymptotic normality of the KMM estimator for conditional moment restrictions follows directly from the result for functional moment restrictions Theorem~\ref{th:asymptotic-normality} using that by Theorem~\ref{th:consistency-cmr} the assumptions of Theorem~\ref{th:consistency-cmr} are sufficient to satisfy the assumptions of Theorem~\ref{th:consistency-fmr}. 
What remains to be shown is that we can translate the asymptotic covariance of the KMM estimator for functional moment restrictions into an expression containing the conditional quantities.
To this aim, first, we show that we can express the asymptotic covariance of the KMM estimator for FMR in a variational form following Lemma~15 of~\citet{bennett2020variational}.
\begin{lemma} \label{lemma:variational}
    Let the assumptions of Theorem~\ref{th:consistency-fmr} be fulfilled. Then we have
    \begin{align}
        E[\nabla_\theta \Psi(X,Z;\theta_0) \Omega_0^{-1} \nabla_\theta \Psi(X,Z;\theta_0)] = \sup_{h \in \mathcal{H}} E[\nabla_\theta \psi(X;\theta_0)^T h(Z)] - \frac{1}{4} E\left[ \left(  \psi(X;\theta_0)^T h(Z) \right)^2 \right]
    \end{align}
\end{lemma}
\begin{proof}
    By Lemma~14 of \citet{bennett2020variational} we have for any Hilbert space $\mathcal{H}$, and element $h \in \mathcal{H}$, that
    \begin{align}
        \| h\|_\mathcal{H}^2 = \sup_{h' \in \mathcal{H}} \langle h ,h' \rangle - \frac{1}{4} \|h' \|^2. 
    \end{align} 
    Moreover, as the dual space of $\mathcal{H}$, $\mathcal{H}^\ast$ is a Hilbert space itself, we can write for any $(x,z) \in \mathcal{X} \times \mathcal{Z}$, 
    \begin{align}
        \nabla_\theta \Psi(x,z;\theta_0) \Omega_0^{-1} \nabla_\theta \Psi(x,z;\theta_0) =&  \|\Omega_0^{-1/2} \nabla_\theta \Psi(x,z;\theta) \|^2_{\mathcal{H}^\ast} \\ 
        =& \sup_{h' \in \mathcal{H}^\ast} \langle \Omega_0^{-1/2} \nabla_\theta \Psi(x,z;\theta), h' \rangle_{\mathcal{H}^\ast}  - \frac{1}{4} \|h' \|^2_{\mathcal{H}^\ast} \\
       =& \sup_{h' \in \mathcal{H}^\ast} \langle \nabla_\theta \Psi(x,z;\theta_0) , \Omega_0^{-1/2} h' \rangle_{\mathcal{H}^\ast} - \frac{1}{4} \|h' \|^2_{\mathcal{H}^\ast} \\
       =& \sup_{h' \in \operatorname{Range}(\Omega^{-1/2})} \langle \nabla_\theta \Psi(x,z;\theta_0), h' \rangle - \frac{1}{4} \langle \Omega^{1/2} h' , \Omega^{1/2} h' \rangle_{\mathcal{H}^\ast} \\
       =& \sup_{h'\in \mathcal{H}^\ast} \langle \nabla_\theta \Psi(x,z;\theta_0), h' \rangle - \frac{1}{4} \langle  h' , \Omega h' \rangle_{\mathcal{H}^\ast} \\
       =& \sup_{h \in \mathcal{H}} \nabla_\theta \psi(x;\theta_0)^T h(z)  - \frac{1}{4} h(z)^T \psi(x;\theta_0) \psi(x;\theta_0)^T h(z)
    \end{align}
    where we used that $\operatorname{Range}(\Omega_0^{-1/2}) = \mathcal{H}^\ast$. This follows as $\Omega_0$ is defined on all of $\mathcal{H}$ and invertible which immediately implies $\Omega_0^{1/2}$ is defined on all of $\mathcal{H}$ and invertible. This means that $\Omega_0^{1/2}$ is injective and thus $\operatorname{Range}(\Omega_0^{-1/2}) = \mathcal{H}^\ast$.
    The result follows by taking the expectation over $(x,z)$ on both sides.
\end{proof}
With the variational formulation at hand we can translate the expression of the covariance of the KMM estimator for FMR into an expression for CMR. The following result is a special case of Lemma~25 of \citet{bennett2020variational}.
\begin{lemma} \label{lemma:variance}
    Let the assumptions of Theorem~\ref{th:consistency-fmr} be fulfilled.
    Then, if $V_0(Z)= E[\psi(X;\theta_0) \psi(X;\theta_0)^T | Z]$ is non-singular with probability $1$, we have
    \begin{align}
        E[\nabla_\theta \Psi(X,Z;\theta_0) \Omega_0^{-1} \Psi(X,Z;\theta_0)] = E\left[  E[\nabla_\theta \psi(X;\theta_0)|Z] \ V_0^{-1}(Z) \ E[\nabla_\theta \psi(X;\theta_0)|Z] \right].
    \end{align} 
\end{lemma}
\begin{proof}
    Using Lemma~\ref{lemma:variational} we can write
    \begin{align}
        E[\nabla_\theta \Psi(X,Z;\theta_0) \Omega_0^{-1} \Psi(X,Z;\theta_0)] = \sup_{h \in \mathcal{H}} E[\nabla_\theta \psi(X;\theta_0)^T h(Z)] - \frac{1}{4} E\left[ \left(  \psi(X;\theta_0)^T h(Z) \right)^2 \right] =: L(h) \label{eq:variational}
    \end{align}
    The functional derivative of $L$ at $h^\ast \in \mathcal{H}$ in direction $\epsilon \in \mathcal{H}$ is given by
    \begin{align}
        \left(\frac{\partial}{\partial h}L (h^\ast) \right) (\epsilon) &= E[\nabla_\theta \psi(X;\theta_0)^T \epsilon(Z)] - \frac{1}{2} E[\epsilon(Z)^T \psi(X;\theta_0) \psi(X;\theta_0)^T h^\ast(Z)] \\
        &= E\left[\epsilon(Z)^T \left( \nabla_\theta \psi(X;\theta_0) - \frac{1}{2} \psi(X;\theta_0) \psi(X;\theta_0)^T h^\ast(Z) \right)\right] \\
        &= E\left[ \epsilon(Z)^T \left( E[\nabla_\theta \psi(X;\theta_0)|Z] - \frac{1}{2} V_0(Z) h^\ast(Z) \right)\right]
    \end{align}
    Now, by assumption $V_0(Z)$ is non-singular and thus invertible, moreover $L$ is a concave functional in $h$ and thus the global maximizer is given for any $z \in \mathcal{Z}$ by
    \begin{align}
        h^\ast(z) = 2 V_0(z)^{-1} E[\nabla_\theta \psi(X;\theta_0) | Z=z].
    \end{align}
    Inserting back into equation~\eqref{eq:variational} and denoting $D_0(z) := E[\nabla_\theta \psi(X;\theta_0)|Z=z]$ we have
    \begin{align}
        E[\nabla_\theta \Psi(X,Z;\theta_0) \Omega_0^{-1} \Psi(X,Z;\theta_0)] &= 2 E[ D_0(Z)^T V_0(Z)^{-1} D_0(Z)] - E[ D_0(Z)^T V_0(Z)^{-1} V_0(Z) V_0(Z)^{-1} D_0(Z)] \\
        &= E[ D_0(Z)^T V_0(Z)^{-1} D_0(Z)].
    \end{align}
\end{proof}
\paragraph{Proof of Theorem~\ref{th:asymptotic-normality-cmr}}
\begin{proof}
The conditions of Theorem~\ref{th:asymptotic-normality} are fulfilled by the conditions of Theorem~\ref{th:asymptotic-normality-cmr} by the proof of Theorem~\ref{th:consistency-cmr}.
We can translate the expression for the asymptotic variance in terms of the moment functional into the conditional counterpart by applying Lemma~\ref{lemma:variance} whose conditions are fulfilled by Assumption~e) of Theorem~\ref{th:consistency-cmr}.
\end{proof}
\subsubsection{Proof of Corollary~\ref{efficiency-cmr} (Efficiency for CMR)}
\begin{proof}
    This is a direct implication of Theorem~\ref{th:asymptotic-normality-cmr} as the asymptotic variance of the KMM estimator achieves the semi-parametric efficiency bound of \citet{CHAMBERLAIN1987305}.
\end{proof}
\newpage
\subsection{Asymptotic Properties of KMM for Functional Moment Restrictions} \label{sec:proofs-fmr}
The consistency proofs roughly follow the general idea laid out in the seminal paper by \citet{newey04} with the adaption to functional moment restrictions by \citet{kremer2022functional}.
The proof for the finite dimensional case is mostly a special case of the proof of the functional version. Therefore, we provide a detailed proof for the arguably more interesting functional case and a short version for the finite dimensional case, emphasizing the differences to the former. 
\subsubsection{Proof of Theorem~\ref{th:consistency-fmr}}
\begin{lemma}\label{lemma:d2}
Let $\mathcal{A}$ denote a $\sigma$-algebra on $\mathcal{X} \times \mathcal{Z}$ and let ($\mathcal{X} \times \mathcal{Z}, \mathcal{A}, \omega)$ be a probability space with measure $\omega$. For any functional $\Psi: (\mathcal{X} \times \mathcal{Z}) \times \Theta \times \mathcal{H} \rightarrow \mathbb{R}$ with $\intxz{ \sup _{\theta \in \Theta}\|\Psi(x,z; \theta)\|_{\mathcal{H}^\ast}^{2}} <\infty$, 
it follows that $\sup _{\theta \in \Theta}\left\|\Psi\left(X,Z; \theta\right)\right\|_{\mathcal{H}^\ast} \leq C \asxz$ for some constant $C < \infty$.
\end{lemma}
\begin{proof}
    The proof is trivially implied by the definition of the almost surely property. If the event $\mathcal{E} = \left\{(x,z) \in \mathcal{X} \times \mathcal{Z} \ : \ \sup _{\theta \in \Theta}\left\|\Psi\left(x,z; \theta\right)\right\|_{\mathcal{H}^\ast} = \infty\right\}$ has non-zero measure, i.e., $\omega\left[\mathcal{E} \right] \neq 0$, then
    $\intxz{ \sup _{\theta \in \Theta}\|\Psi(x,z; \theta)\|_{\mathcal{H}^\ast}} = \infty$ and thus $\intxz{ \sup _{\theta \in \Theta}\|\Psi(x,z; \theta)\|_{\mathcal{H}^\ast}^{2}} = \infty $. Therefore we must have $\omega[\mathcal{E}]=0$ and there exists some constant $C$ such that $\sup _{\theta \in \Theta}\left\|\Psi\left(X,Z; \theta\right)\right\|_{\mathcal{H}^\ast} \leq C \asxz$. 
\end{proof}
\begin{lemma}\label{mmdlemma}
    For two distributions $Q_1$ and $Q_2$ on $\mathcal{X} \times \mathcal{Y}$ define the mixing distribution $\omega = (1-\alpha) Q_1 + \alpha Q_2$, with $\alpha = O_p(n^{-\zeta})$ and $\zeta > 0$. Then $\mmd(Q_1, \omega; \mathcal{F}) = O_p(n^{-\zeta})$ for any RKHS $\mathcal{F}$ of functions $\mathcal{X} \times \mathcal{Z} \rightarrow \mathbb{R}$.
    In particular 
    it follows for any distribution $Q$ and $\omega = (1-\alpha) \hat{P}_n + \alpha Q$ that $\mmd(\hat{P}_n, \omega; \mathcal{F}) = O_p(n^{-\zeta})$.
\end{lemma}
\begin{proof}
The proof follows directly by using the definition of MMD,
\begin{align}
    \mmd(Q_1, \omega; \mathcal{F}) =& \sup_{f \in \mathcal{F}, \| f\|_{\mathcal{F}} = 1} \left( \int_{\mathcal{X} \times \mathcal{Z}} f(x,z) Q_1(\mathrm{d}x \otimes \mathrm{d}z) - \intxz{f(x,z)} \right) \\
    =&  \alpha \sup_{f \in \mathcal{F}, \| f\|_\mathcal{F} = 1} \left( \int_{\mathcal{X} \times \mathcal{Z}} f(x,z) Q_1(\mathrm{d}x \otimes \mathrm{d}z) - \int_{\mathcal{X} \times \mathcal{Z}} f(x,z) Q_2(\mathrm{d}x \otimes \mathrm{d}z) \right) \\
    =& \alpha \mmd(Q_1, Q_2;\mathcal{F}) \\
    =& \alpha C \\
    =& O_p(n^{-\zeta}),
\end{align}
where we used that $\operatorname{MMD}(Q_1, Q_2;\mathcal{F})$ can be bounded by some positive constant $C$ for any $Q_1$, $Q_2$ and $\mathcal{F}$ which directly follows from the fact that by definition of an RKHS the evaluation functional in $\mathcal{F}$ is bounded and $Q_1$, $Q_2$ are finite measures normalized to $1$. 
The second statement is a direct application of the former.
\end{proof}

\begin{lemma}\label{lemma1}
Let the assumptions of Theorem~\ref{th:consistency-fmr} be satisfied. For any $\zeta$ with $0 < \zeta < 1/2$ define the magnitude constrained set of dual variables $\mathcal{M}_n = \{\beta =(\eta, f, h) \in \mathcal{M}: \| \beta \|_{\mathcal{M}} \leq n^{-\zeta} \}$. Then as $n \rightarrow \infty$,
\begin{align}
    &\sup_{\theta \in \Theta, \ (\eta, f, h) \in \mathcal{M}_n} |\Psi(X,Z;\theta)(h)| = O_p(n^{-\zeta}) \asxz, \\
    & \sup_{\theta \in \Theta, \ \beta \in \mathcal{M}_n} |a(X,Z;\theta)^T \beta| = O_p(n^{-\zeta})  \asxz.
\end{align}
\end{lemma}
\begin{proof}
Using the Cauchy-Schwarz inequality and Assumption~d) with Lemma~\ref{lemma:d2},
\begin{align}
    &  \sup_{\theta \in \Theta, \beta \in \mathcal{M}_n} |\Psi(X,Z;\theta)(h)| \\
    \leq &   \sup_{\theta \in \Theta, \beta \in \mathcal{M}_n} \left( \| h \|_\mathcal{H} \cdot \| \Psi(X,Z;\theta) \|_{\mathcal{H}^\ast}  \right) \\
    \leq &  \sup_{\theta \in \Theta, \beta \in \mathcal{M}_n} \left( \| \beta \|_\mathcal{M} \cdot \| \Psi(X,Z;\theta) \|_{\mathcal{H}^\ast}  \right) \\
    \leq & n^{-\zeta}  \sup_{\theta \in \Theta}  \|\Psi(X,Z;\theta) \|_{\mathcal{H}^\ast}.
\end{align}
Now, by Assumptions~d) and f) of Theorem~\ref{th:consistency-fmr} and Lemma~\ref{lemma:d2} we have that $\sup_{\theta \in \Theta} \|\Psi(X,Z;\theta) \|_{\mathcal{H}^\ast} < C \ P_{0}\text{-a.s.}$ and $\sup_{\theta \in \Theta} \|\Psi(X,Z;\theta) \|_{\mathcal{H}^\ast} < C \ Q\text{-a.s.}$ respectively and as $\omega \rightarrow P_0$ weakly we have w.p.a.1 that $\sup_{\theta \in \Theta} \|\Psi(X,Z;\theta) \|_{\mathcal{H}^\ast} < C \ \asxz$ and thus w.p.a.1 $\sup_{\theta \in \Theta, \beta \in \mathcal{M}_n} |\Psi(X,Z;\theta)(h)| = 0 \asxz$.
For the second part note that if $\| \beta \|_\mathcal{M} \leq n^{-\zeta}$, we must have that $|\eta|, \| f\|_\mathcal{F}, \| h\|_\mathcal{H} \leq n^{-\zeta}$.
Then we have
\begin{align}
    & \sup_{\theta \in \Theta, \ \beta \in \mathcal{M}_n} |a(X,Z;\theta)^T \beta|   \\
    = &  \sup_{\theta \in \Theta, \ (\eta, f, h) \in \mathcal{M}_n} | \eta + \langle k( (X,Z),\cdot) , f \rangle_\mathcal{F} + \Psi(X,Z;\theta)(h) | \\
    \leq & \sup_{(\eta, f, h) \in \mathcal{M}_n} |\eta| +  \sup_{(\eta, f, h) \in \mathcal{M}_n} |\langle k((X,Z),\cdot) , f \rangle_\mathcal{F}|  +  \sup_{\theta \in \Theta, \beta \in \mathcal{M}_n} |\Psi(X,Z;\theta)(h)| \\
    \leq & n^{-\zeta} + \sup_{(\eta, f, h) \in \mathcal{M}_n} \| f \|_\mathcal{F}   \| k((X,Z),\cdot) \|_\mathcal{F} +\sup_{\theta \in \Theta, \beta \in \mathcal{M}_n} |\Psi(X,Z;\theta)(h)| \\
    \leq & n^{-\zeta} + C n^{-\zeta} + \sup_{\theta \in \Theta, \beta \in \mathcal{M}_n} |\Psi(X,Z;\theta)(h)| \\
    \leq & \left( n^{-\zeta} + C n^{-\zeta} + C n^{-\zeta}  \right) \asxz \\
    \leq & O(n^{-\zeta}) \asxz, 
\end{align}
where for the second term in the fourth line we applied the Cauchy-Schwarz inequality and used the fact that in an RKHS the evaluation functional is bounded by some constant $C > 0$.
\end{proof}
\begin{lemma} \label{lemma:omega-empirical}
    Under the assumptions of Theorem~\ref{th:consistency-fmr} we have for any $\theta \in \Theta$,
    \begin{align}
        \intxz{a(x,z;\theta)} =  \avg a(x_i,z_i;\theta) + O_p(n^{-1})
    \end{align}
    and for any $\beta \in \mathcal{M}$ with $\|\beta \|_\mathcal{M} < \infty$,
    \begin{align}
    \intxz{ a(x,z;\theta) a(x,z;\theta)^T \beta} = \avg  a(x_i,z_i;\theta) a(x_i,z_i;\theta)^T \beta + O_p(n^{-1}).
    \end{align}
\end{lemma}
\begin{proof}
    For the first statement note that
    \begin{align}
        \intxz{ a(x,z;\theta)} &= \intxz{ a(x,z;\theta)} + \int a(x,z;\theta) \mathrm{d}\hat{P}_n  - \int a(x,z;\theta) \mathrm{d}\hat{P}_n \\
        & = \avg a(x_i,z_i;\theta) + \alpha \int a(x,z;\theta) (\mathrm{d}Q - \mathrm{d}\hat{P}_n)
    \end{align}
    Now $a(x,z;\theta_0) = \left(1, k\left((x,z), \cdot\right), \Psi(x,z;\theta) \right)^T$ is trivially integrable in the first component and second component with respect to any probability distribution as the evaluation functional $k((x,z),\cdot)$ in $\mathcal{F}$ is bounded by definition of an RKHS. For the third component integrability with respect to $Q$ follows by Assumption~f) of Theorem~\ref{th:consistency-fmr}. Moreover, by Assumption~d) of Theorem~\ref{th:consistency-fmr} and Lemma~\ref{lemma:d2} we have $\sup_{\theta \in \Theta} \|\Psi(X,Z;\theta)\| \leq C$ w.p.1 with respect to $P_0$. And thus as $\hat{P}_n \overset{p}{\rightarrow} P_0$ weakly, we have $\int \sup_{\theta \in \Theta} \|\Psi(X,Z;\theta)\| \hat{P}_n(\mathrm{d}x \otimes \mathrm{d}z) < \infty$ w.p.a.1. In conclusion we have $\int a(x,z;\theta) (\mathrm{d}Q - \mathrm{d}\hat{P}_n) < \infty$ w.p.a.1 and as $\alpha = O_p(n^{-1})$ we finally get
    \begin{align}
        \intxz{ a(x,z;\theta)} =  \avg a(x_i,z_i;\theta) + O_p(n^{-1}).
    \end{align}
    For the second statement consider any $\beta \in \mathcal{M}$ with $\|\beta \|_\mathcal{M} < \infty$, 
    \begin{align}
        \intxz{ a(x,z;\theta) a(x,z;\theta)^T \beta} =& \avg  a(x_i,z_i;\theta) a(x_i,z_i;\theta)^T \beta \\
        &+ \alpha \int_{\mathcal{X} \times \mathcal{Z}} a(x,z;\theta) a(x,z;\theta)^T \beta \left( \mathrm{d}Q - \mathrm{d}\hat{P}_n \right). \label{eqxx}
        \end{align}
    Now for the second term we have
    \begin{align}
        &\|\int_{\mathcal{X} \times \mathcal{Z}} a(x,z;\theta) a(x,z;\theta)^T \beta \left( \mathrm{d}Q - \mathrm{d}\hat{P}_n \right)  \|_{\mathcal{M}^\ast}^2 \\
        \leq& \int_{\mathcal{X} \times \mathcal{Z}} \|  a(x,z;\theta) a(x,z;\theta)^T \beta \|^2 \left( \mathrm{d}Q + \mathrm{d}\hat{P}_n \right) \\
        =& \int_{\mathcal{X} \times \mathcal{Z}}   | a(x,z;\theta)^T \beta|^2 \|a(x,z;\theta) \|^2 \left( \mathrm{d}Q + \mathrm{d}\hat{P}_n \right) \\
        \leq& \int_{\mathcal{X} \times \mathcal{Z}}   | a(x,z;\theta)^T \beta|^2 \left( \mathrm{d}Q + \mathrm{d}\hat{P}_n \right) \int_{\mathcal{X} \times \mathcal{Z}}  \|a(x,z;\theta) \|^2 \left( \mathrm{d}Q + \mathrm{d}\hat{P}_n \right) \\
        \leq& \|\beta \|_\mathcal{M}^2 \left(\int_{\mathcal{X} \times \mathcal{Z}}  \|a(x,z;\theta) \|^2 \left( \mathrm{d}Q + \mathrm{d}\hat{P}_n \right) \right)^2 \\
        \leq & \|\beta \|_\mathcal{M}^2 \left(\int_{\mathcal{X} \times \mathcal{Z}} 1 + \|k((x,z),\cdot) \|_\mathcal{F}^2 + \|\Psi(x,z;\theta) \|_{\mathcal{H}^\ast}^2 \left( \mathrm{d}Q + \mathrm{d}\hat{P}_n \right) \right)^2.
    \end{align}
    The first term is trivially bounded, the second bounded as $\mathcal{F}$ is an RKHS and thus its evaluation functional $k((x,z),\cdot)$ is bounded. The third term is bounded as $E_Q[\sup_{\theta \in \Theta} \|\Psi(X,Z;\theta) \|_{\mathcal{H}^\ast}^2 < \infty$ by Assumption~f) of Theorem~\ref{th:consistency-fmr} and $E_{\hat{P}_n}[\sup_{\theta \in \Theta} \|\Psi(X,Z;\theta) \|_{\mathcal{H}^\ast}^2 < \infty$ w.p.a.1 as $E[\sup_{\theta \in \Theta} \|\Psi(X,Z;\theta) \|_{\mathcal{H}^\ast}^2 < \infty$ by Assumption~d) of Theorem~\ref{th:consistency-fmr} and $\hat{P}_n \rightarrow P_0$ weakly. In conclusion the norm of the integral in equation~\eqref{eqxx} is bounded by some constant $C$ and as $\alpha = O_p(n^{-1})$ the statement follows.
\end{proof}
\begin{lemma}\label{lemma:bounded-covariance}
    Let the assumptions of Theorem~\ref{th:consistency-fmr} be satisfied and consider $\bar{\theta} \in \Theta$ such that $\bar{\theta} \overset{p}{\rightarrow} \theta_{0}$. Further let $\beta_\zeta := \argmax_{\beta \in \mathcal{M}_n} \widehat{G}(\bar{\theta}, \beta)$, where $\mathcal{M}_n = \{\beta \in \mathcal{M}: \| \beta \|_\mathcal{M} \leq n^{-\zeta} \}$ with $0 < \zeta < 1/2$.
Define the operator $\Lambda_{n}(\beta,\theta): \mathcal{M} \rightarrow \mathcal{M}$ as
\begin{align}
    \Lambda_{n}(\beta, \theta) := \intxz{ \frac{1}{\epsilon} a(x,z;\theta) a(x,z;\theta)^T \varphi^\ast_2 \left( \frac{1}{\epsilon} a(x,z;\theta)^T \beta \right)}  +  R_{\lambda_n}. \label{eq:covariance}
\end{align}
Then w.p.a.1 for any $\bar{\beta} \in \operatorname{conv}(\{0, \beta_\zeta\})$, $\Lambda_{n}(\bar{\beta}, \bar{\theta})$ is strictly positive definite and its smallest eigenvalue is bounded away from zero. Moreover, for any $\theta \in \Theta$ the largest eigenvalue of $\Lambda_n(\bar{\beta}, \theta)$ is bounded from above by a positive constant $M$.
\end{lemma}
\begin{proof}
    As $\bar{\beta} \in \operatorname{conv}(\{0, \beta_\zeta \})$ we have $\bar{\beta} \in \mathcal{M}_n$, and hence Lemma~\ref{lemma1} implies that $\sup_{\theta \in \Theta} |a(X,Z;\theta)^T \bar{\beta}| \overset{n\rightarrow \infty}{\longrightarrow} 0 \asxz$, which implies for every fixed value of $\epsilon >0$, $\varphi_2^\ast\left( \frac{1}{\epsilon} a(X,Z;\theta)^T \bar{\beta} \right) \overset{n\rightarrow \infty}{\longrightarrow} \varphi_2(0) = 1 \asxz$ by the continuous mapping theorem. This means that for every value of $(x,z)$ that provides a non-vanishing contribution to the integral, we have $\varphi_2^\ast\left( \frac{1}{\epsilon} a(x,z;\theta)^T \bar{\beta} \right) \overset{n\rightarrow \infty}{\longrightarrow} 1$ and so as $n\rightarrow \infty$ the first term is equivalent to $\intxz{\frac{1}{\epsilon} a(x,z;\theta)a(x,z;\theta)^T}$ which clearly is a positive semi-definite operator. In the following we will show that its smallest eigenvalue is bounded away from zero w.p.a.1.
    First note that for any vector $\beta = (\eta, f, h) \in \mathcal{M}$ with $f \neq 0$ we have
    \begin{align}
        \beta^T \Lambda_{n} \beta &= \intxz{ \underbrace{(a(x,z;\theta)^T \beta)^2}_{\geq 0}} + \| f \|^2 + \lambda_n \| h\|^2 \\
        &\geq \| f \|^2 + \lambda_n \| h\|^2 \\
        &> 0,
    \end{align}
    and thus such vector cannot correspond to an eigenvalue of $0$. Therefore consider any vector $\beta = (\eta, 0, h) \in \mathcal{M}$, then if such vector corresponds to an eigenvalue of zero we must have 
    \begin{align}
    \begin{pmatrix}
        0 \\ 0 \\0 
    \end{pmatrix} =&
        \Lambda_{n}(\bar{\beta}, \bar{\theta}) 
        \begin{pmatrix}
            \eta \\ 0 \\ h
        \end{pmatrix} \\
        =& \intxz{ \frac{1}{\epsilon} a(x,z;\bar{\theta}) a(x,z;\bar{\theta})^T 
        \begin{pmatrix}
            \eta \\ 0 \\ h
        \end{pmatrix}} + 
        \begin{pmatrix}
            0 & 0 & 0 \\
            0 & I &0 \\
            0& 0& \lambda_n I
        \end{pmatrix}
        \begin{pmatrix}
            \eta \\ 0 \\ h
        \end{pmatrix}
        \\
    =& \intxz{ \frac{1}{\epsilon}
    \begin{pmatrix}
        \eta - \Psi(x,z;\bar{\theta})(h) \\
        k((x,z),\cdot) \eta - k((x,z),\cdot) \Psi(x,z;\bar{\theta})(h) \\
        -\eta \Psi(x,z;\bar{\theta}) + \Psi(x,z;\bar{\theta}) \Psi(x,z;\bar{\theta}) (h) + \lambda_n h
    \end{pmatrix}}\\
    =&
    \frac{1}{\epsilon}
    \avg
    \begin{pmatrix}
        \eta - \Psi(x_i,z_i;\bar{\theta})(h) \\
        k((x_i,z_i),\cdot) \eta - k((x_i,z_i),\cdot) \Psi(x_i,z_i;\bar{\theta})(h) \\
        -\eta \Psi(x_i,z_i;\bar{\theta}) + \Psi(x_i,z_i;\bar{\theta}) \Psi(x_i,z_i;\bar{\theta}) (h) + \lambda_n h
    \end{pmatrix} + O_p(n^{-1}),
    \end{align}
    where we used Lemma~\ref{lemma:omega-empirical} to express the integral term in terms of the empirical average.
    Now the first row gives $\eta = E_{\hat{P}_n} [\Psi(X,Z;\bar{\theta})(h)] + O_p(n^{-1})$ which inserted in the last row gives
    \begin{align}
        0 = \left( \underbrace{E_{\hat{P}_n}[\Psi(x_i,z_i;\bar{\theta}) \otimes \Psi(x_i,z_i;\bar{\theta})] - E_{\hat{P}_n}[\Psi(x_i,z_i;\bar{\theta})] \otimes E_{\hat{P}_n}[\Psi(x_i,z_i;\bar{\theta})]}_{=: \hat{\Omega}(\bar{\theta})} + \lambda_n \right) h + O_p(n^{-1}).
    \end{align}
     As $n\rightarrow \infty$ we have $\bar{\theta} \rightarrow \theta_0$ and by the uniform weak law of large numbers and the continuous mapping theorem (as $\Psi$ is continuous in $\theta$ by Assumption~c) of Theorem~\ref{th:consistency-fmr}) $\hat{\Omega}(\bar{\theta}) \rightarrow E[\Psi(X,Z;\theta_0) \otimes \Psi(X,Z;\theta_0)] - E[\Psi(X,Z;\theta_0)] \otimes E[\Psi(X,Z;\theta_0)] = E[\Psi(X,Z;\theta_0) \otimes \Psi(X,Z;\theta_0)] = \Omega_0$.
     Thus as $n\rightarrow \infty$ we have
     \begin{align}
         0 = \left( \Omega_0 + \lambda_n \right) h + O_p(n^{-1})
     \end{align}
     From Assumption~e) of Theorem~\ref{th:consistency-fmr} it follows that $\Omega_0$ is non-singular and thus $0$ is not in its spectrum. Moreover, by Assumption~g) of Theorem~\ref{th:consistency-fmr} $\lambda_n = O_p(n^{-\xi})$ with $0 < \xi < 1/2$, so the RHS is $\neq 0$ w.p.a.1 and the eigenvalue equations can only be fulfilled with $h =0$ which implies $\eta = 0$ and thus $\beta =0$. Therefore it follows that the smallest eigenvalue of $\Lambda_n(\bar{\beta}, \bar{\theta})$ is bounded away from zero w.p.a.1.
    In order to bound the largest eigenvalue of $\Lambda_{n}(\bar{\beta}, \theta)$ for any $\theta \in \Theta$ recall that for the second term we have $\operatorname{eig}(R_{\lambda_n}) = \{0, 1, \lambda_n \}$ where $\lambda_n \rightarrow 0$. Therefore, the boundedness depends on the eigenvalues of the first term. For any $\beta \in \mathcal{M}$ we have
    \begin{align}
        \beta^T \Lambda_{n} \beta &= \intxz{ \frac{1}{\epsilon} \beta^T a(x,z;\theta) a(x,z;\theta)^T \beta} + \beta^T R_{\lambda_n} \beta \\
        &\leq \intxz{ \frac{1}{\epsilon} \| a(x,z;\theta)^T \beta \|^2 }  + \| \beta \|^2 \\
        & = \intxz{ \frac{1}{\epsilon} \| \eta + \langle(k((x,z),\cdot), f \rangle_\mathcal{F} - \Psi(x,z;\theta)(h) \|^2 } + \| \beta \|^2 \\
        &\leq \intxz{ \frac{1}{\epsilon} \left( \| \eta \|^2 + \| \langle(k((x,z),\cdot), f \rangle_\mathcal{F} \|^2 +  \| \Psi(x,z;\theta)(h) \|^2 \right)}  + \| \beta \|^2 \\
        &\leq  \intxz{ \frac{1}{\epsilon} \left( \| \eta \|^2 + \|f \|^2 \| k((x,z),\cdot) \|^2 + \|h \|^2_\mathcal{H} \| \Psi(x,z;\theta) \|_{\mathcal{H}^\ast}^2 \right)} + \| \beta \|^2.
    \end{align}
    Now, as $\mathcal{F}$ is an RKHS, the evaluation functional $k((x,z),\cdot)$ can be bounded by a constant $C_1$. 
    Moreover, by Assumption~d) and f) of Theorem~\ref{th:consistency-fmr}, we have $\int \sup_{\theta \in \Theta} \| \Psi(x,z;\theta)\|_{\mathcal{H}^\ast}^2 \mathrm{d}P_0 < \infty$ and $\int \sup_{\theta \in \Theta} \| \Psi(x,z;\theta)\|_{\mathcal{H}^\ast}^2 \mathrm{d}Q < \infty$ and thus as $\omega = (1-\alpha) \hat{P}_n + \alpha Q \overset{p}{\rightarrow} (1-\alpha) P_0 + \alpha Q$ it follows $\sup_{\theta \in \Theta} \int \| \Psi(x,z;\theta)\|_{\mathcal{H}^\ast}^2 \mathrm{d}\omega \leq \int   \sup_{\theta \in \Theta} \| \Psi(x,z;\theta)\|_{\mathcal{H}^\ast}^2 \mathrm{d}\omega < C_2$ 
    for some $C_2 > 0$ w.p.a.1. 
    Inserting this back we obtain
    \begin{align}
        \beta^T \Lambda_{\epsilon, \lambda_n} \beta &\leq \frac{1}{\epsilon} \left( \| \eta \|^2 + C_1 \|f \|^2 + C_2 \|h \|^2_\mathcal{H} \right) + \| \beta \|^2 \\
        &\leq \left( \frac{C_3}{\epsilon} + 1\right) \| \beta \|^2,
    \end{align}
    where $C_3=\max(1, C_1, C_2)$. It follows w.p.a.1 that the largest eigenvalue of $\Lambda_{n}$ can be bounded by some constant $M = \frac{C_3}{\epsilon} + 1$ for any finite value of $\epsilon > 0$.
\end{proof}
\begin{lemma}\label{lemma2}
  Let the assumptions of Theorem~\ref{th:consistency-fmr} be satisfied. Additionally let $\bar{\theta} \in \Theta$, $\bar{\theta} \overset{p}{\rightarrow} \theta_{0}$, and $\|E_{\hat{P}_n}[\Psi(X,Z;\bar{\theta})] \|_{\mathcal{H}^\ast} = O_{p}\left(n^{-1 / 2}\right)$.  
  Then for $\bar{\beta}=$ $\argmax_{\beta \in \mathcal{M}} \widehat{G}_{\epsilon, \lambda_n}(\bar{\theta}, \beta)$ we have $\| \bar{\beta}\|_\mathcal{M}=O_{p}\left(n^{-1/2}\right)$, and $\widehat{G}_{\epsilon, \lambda_n}(\bar{\theta},\bar{\beta}) \leq -\epsilon \varphi^\ast(0) + O_p\left(n^{-1}\right)$.
\end{lemma}
\begin{proof}
    Define $\bar{\Psi}_i := \Psi(x_i,z_i;\bar{\theta})$ and $\bar{\Psi} = \frac{1}{n} \sum_{i=1}^n \bar{\Psi}_i$. For simplicity of notation let $\widehat{G}(\theta, \beta) := \widehat{G}_{\epsilon, \lambda_n}(\theta, \beta)$.
    The first and second derivative of $\widehat{G}(\bar{\theta}, \beta)$ with respect to $\beta$ are given by
    \begin{align}
    \frac{\partial \widehat{G}}{\partial \beta} (\bar{\theta}, \beta) &= \avg b_i - \intxz{  a(x,z;\bar{\theta}) \varphi_1^\ast \left( \frac{1}{\epsilon} a(x,z;\bar{\theta})^T \beta \right)} - R_{\lambda_n} \beta \\
    \frac{\partial^2 \widehat{G}}{(\partial \beta)^2} (\bar{\theta}, \beta) &= - \intxz{ \frac{1}{\epsilon} a(x,z;\bar{\theta}) a(x,z;\bar{\theta})^T \varphi_2^\ast \left( \frac{1}{\epsilon} a(x,z;\bar{\theta})^T \beta \right)} - R_{\lambda_n}.
    \end{align}
    Consider the optimal dual parameter within the magnitude constrained set $\mathcal{M}_n = \{\beta \in \mathcal{M}: \| \beta \|_\mathcal{M} \leq n^{-\zeta} \}$, i.e., $\beta_\zeta := \argmax_{\beta \in \mathcal{M}_n} \widehat{G}(\bar{\theta}, \beta)$ with $\beta_\zeta = (\eta_\zeta, f_\zeta, h_\zeta)$. Later on, we will show that this maximizer can be identified with the maximizer over the original set $\mathcal{M}$.
    Using Taylor's theorem we can expand the empirical KMM objective about $\beta = 0$,
    \begin{align}
        \widehat{G}(\bar{\theta}, \beta_\zeta) =& \widehat{G}(\bar{\theta}, 0) + \frac{\partial \widehat{G}}{\partial \beta} (\bar{\theta}, 0) \beta_\zeta + \frac{1}{2} \beta_\zeta^T \frac{\partial^2 \widehat{G}}{(\partial \beta)^2} (\bar{\theta}, \Dot{\beta}) \beta_\zeta \\
        =& -\epsilon \varphi^\ast(0) + \intxz{ \Psi(x,z,\bar{\theta})(h_\zeta) } \label{lemma12taylor}\\
        &+ \avg f_\zeta(x_i, z_i) - \intxz{ f_\zeta(x,z)} \\
        &- \frac{1}{2} \beta_\zeta^T \underbrace{\left( \intxz{ \frac{1}{\epsilon} a(x, z;\bar{\theta}) a(x, z;\bar{\theta})^T \varphi_2^\ast \left( \frac{1}{\epsilon} a(x, z;\bar{\theta})^T \Dot{\beta} \right)}  +  R_{\lambda_n} \right)}_{:= \Lambda_{n}(\Dot{\beta},\bar{\theta})} \beta_\zeta
    \end{align}
        for some $\Dot{\beta} \in \operatorname{conv}(\{0 , \beta_\zeta \})$. Now adding and subtracting the empirical expectation of the moment functional $\Psi$ we get
        \begin{align}
         \widehat{G}(\bar{\theta}, \beta_\zeta) =& -\epsilon \varphi^\ast(0)  - \frac{1}{2} \beta_\zeta^T \Lambda_{n}(\Dot{\beta},\bar{\theta}) \beta_\zeta +  \int_{\mathcal{X} \times \mathcal{Z}} \Psi(x,z;\bar{\theta})(h_\zeta) \hat{P}_n(\mathrm{d}x \otimes \mathrm{d}z)  \\
        &+ \intxz{ \Psi(x,z;\bar{\theta})(h_\zeta) } - \int_{\mathcal{X} \times \mathcal{Z}} \Psi(x,z;\bar{\theta})(h_\zeta) \hat{P}_n(\mathrm{d}x \otimes \mathrm{d}z) \\
        &+ \int_{\mathcal{X} \times \mathcal{Z}} f_\zeta(x, z) \hat{P}_n(\mathrm{d}x \otimes \mathrm{d}z) - \intxz{ f_\zeta(x,z)} \\
        \leq & -\epsilon \varphi^\ast(0)  - \frac{1}{2} \beta_\zeta^T \Lambda_{n}(\Dot{\beta},\bar{\theta}) \beta_\zeta 
        + \| h_\zeta\|_\mathcal{H} \left\| \int_{\mathcal{X} \times \mathcal{Z}} \Psi(x,z;\bar{\theta}) \hat{P}_n(\mathrm{d}x \otimes \mathrm{d}z) \right\|_{\mathcal{H}^\ast}  \\
        &+ \| h_\zeta\|_\mathcal{H} \left\| \int_{\mathcal{X} \times \mathcal{Z}} \Psi(x,z;\bar{\theta}) \left(\hat{P}_n(\mathrm{d}x \otimes \mathrm{d}z) -  \omega(\mathrm{d}x \otimes \mathrm{d}z)\right)  \right\|_{\mathcal{H}^\ast} \\
        &+ \| f_\zeta \|_\mathcal{F} \sup_{f \in \mathcal{F}, \| f\|_\mathcal{F} = 1}\int_{\mathcal{X} \times \mathcal{Z}} f(x, z) \left(\hat{P}_n(\mathrm{d}x \otimes \mathrm{d}z) - \omega(\mathrm{d}x \otimes \mathrm{d}z)\right) \\
        \leq &  -\epsilon \varphi^\ast(0)  - \frac{1}{2} \beta_\zeta^T \Lambda_{n}(\Dot{\beta},\bar{\theta}) \beta_\zeta 
        + \| h_\zeta\|_\mathcal{H}  \| \bar{\Psi}\|_{\mathcal{H}^\ast}  \\
        &+ \alpha \| h_\zeta\|_\mathcal{H} \left( \int_{\mathcal{X} \times \mathcal{Z}} \left\| \Psi(x,z;\bar{\theta}) \right\|_{\mathcal{H}^\ast} Q(\mathrm{d}x \otimes \mathrm{d}z) + \| \bar{\Psi}\|_{\mathcal{H}^\ast} \right)  \\
        &+ \| f_\zeta \|_\mathcal{F} \sup_{f \in \mathcal{F}, \| f\|_\mathcal{F} = 1}\int_{\mathcal{X} \times \mathcal{Z}} f(x, z) \left(\hat{P}_n(\mathrm{d}x \otimes \mathrm{d}z) - \omega(\mathrm{d}x \otimes \mathrm{d}z)\right) \\
        \leq & -\epsilon \varphi^\ast(0)  - \frac{1}{2} \beta_\zeta^T \Lambda_{n}(\Dot{\beta}, \bar{\theta}) \beta_\zeta 
        +  \| \bar{\Psi}\|_{\mathcal{H}^\ast} \| h_\zeta \|_\mathcal{H} \\
        &+ \alpha \| h_\zeta\|_\mathcal{H} \left( C_Q + \| \bar{\Psi} \|_{\mathcal{H}^\ast} \right) + \| f_\zeta \|_\mathcal{F} \mmd(\hat{P}_n, \omega; \mathcal{F}) \\
        \leq & -\epsilon \varphi^\ast(0)  - \frac{1}{2} \beta_\zeta^T \Lambda_{n}(\Dot{\beta}, \bar{\theta}) \beta_\zeta 
        +  \|\beta_\zeta \|_\mathcal{M} \left( \| \bar{\Psi}\|_\mathcal{H} + \alpha (C_Q + \| \bar{\Psi}\|_{\mathcal{H}^\ast}) + \mmd(\hat{P}_n, \omega; \mathcal{F})  \right)
    \end{align}
    where we repeatedly used the Cauchy-Schwarz inequality and the fact that $\int_{\mathcal{X} \times \mathcal{Z}} \left\| \Psi(x,z;\bar{\theta}) \right\|_{\mathcal{H}^\ast} Q(\mathrm{d}x \otimes \mathrm{d}z) < \int_{\mathcal{X} \times \mathcal{Z}}  \sup_{\theta \in \Theta} \left\| \Psi(x,z;\theta) \right\|_{\mathcal{H}^\ast} Q(\mathrm{d}x \otimes \mathrm{d}z) =: C_Q < \infty$ by Assumption~f) of Theorem~\ref{th:consistency-fmr} and Lemma~\ref{lemma:d2}.
Lemma~\ref{lemma:bounded-covariance} states that
the smallest eigenvalue $C$ of $\Lambda_{n}(\Dot{\beta}, \bar{\theta})$ is bounded away from zero w.p.a.1.
As $\beta_\zeta$ is a global maximizer of $\widehat{G}(\bar{\theta}, \beta)$ over $\mathcal{M}_n$ we have that $\widehat{G}(\bar{\theta}, \beta_\zeta) \geq \widehat{G}(\bar{\theta}, \beta)$ for any $\beta \in \mathcal{M}_n$ and therefore,
\begin{align}
-\epsilon \varphi^\ast(0) &= \widehat{G}(\bar{\theta}, 0) \\
&\leq 
\widehat{G}(\bar{\theta}, \beta_\zeta) \\
&\leq -\epsilon \varphi^\ast(0)  - \frac{1}{2} \beta_\zeta^T \Lambda_{n}(\Dot{\beta}, \bar{\theta}) \beta_\zeta 
    +  \|\beta_\zeta \|_\mathcal{M} \left( \| \bar{\Psi}\|_\mathcal{H} + \alpha (C + \| \bar{\Psi}\|_{\mathcal{H}^\ast}) + \mmd(\hat{P}_n, \omega; \mathcal{F}) \right) \\ 
&\leq -\epsilon \varphi^\ast(0)  -  C \| \beta_\zeta \|_\mathcal{M}^2
    +  \|\beta_\zeta \|_\mathcal{M} \left( \| \bar{\Psi}\|_\mathcal{H} + \alpha (C + \| \bar{\Psi}\|_{\mathcal{H}^\ast}) + \mmd(\hat{P}_n, \omega; \mathcal{F}) \right)
\end{align}
Now, adding $-\epsilon \varphi^\ast(0)$ on both sides and dividing by $\| \beta_\zeta\|_\mathcal{M}$, we have 
\begin{align}
 C \| \beta_\zeta \|_\mathcal{M} \leq \| \bar{\Psi} \|_{\mathcal{H}^\ast} + \alpha (C + \| \bar{\Psi}\|_{\mathcal{H}^\ast}) + \mmd(\hat{P}_n, \omega; \mathcal{F}) .
\end{align} 
As $\| \bar{\Psi} \|_{\mathcal{H}^\ast} = O_p(n^{-1/2})$ by assumption and by Assumption~f) $\alpha = O_p(n^{-1})$ as well as $\mmd(\hat{P}_n, \omega;\mathcal{F}) = O(n^{-1}) = o_p(n^{-1/2})$ by Lemma~\ref{mmdlemma}, we thus obtain $\| \beta_\zeta \|_\mathcal{M} = O_p(n^{-1/2})$. 
So far, we have restricted the analysis to the maximizer $\beta_\zeta$ over the magnitude constrained set of dual variables $\mathcal{M}_n$. In the following we will show that this maximizer agrees with the maximizer over the unconstrained (original) set of dual variables $\mathcal{M}$.
First, note that with $\| \beta_\zeta \|_\mathcal{M} = O_p(n^{-1/2})$ and $\zeta < 1/2$ we have that $n^{-\zeta} > n^{-1/2}$, which means that asymptotically $\beta_\zeta$ is contained in the interior of $\mathcal{M}_n$, i.e., $\beta_{\zeta} \in \operatorname{int}(\mathcal{M}_n)$.
As $\beta_\zeta$ is a maximizer contained in the interior of the domain, it must correspond to a stationary point of $\widehat{G}$, i.e.,  $\frac{\partial \widehat{G}}{\partial \beta}(\bar{\theta}, \beta_\zeta) = 0$. Clearly $\mathcal{M}_n \subset \mathcal{M}$, so the stationary point is contained also in $\mathcal{M}$. As the empirical objective $\widehat{G}$ is concave with respect to $\beta$, this means we must have that $\widehat{G}(\bar{\theta}, \beta_\zeta) = \sup_{\beta \in \mathcal{M}} \widehat{G}(\bar{\theta}, \beta)$ and thus $\bar{\beta} = \beta_\zeta$, where again $\bar{\beta} = \argmax_{\beta \in \mathcal{M}} \widehat{G}(\bar{\theta}, \beta)$.
As $\bar{\beta} = \beta_\zeta$, and $\|\beta_\zeta \|_\mathcal{M} = O_p(n^{-1/2})$ it directly follows that $\| \bar{\beta} \|_\mathcal{M} = O_p(n^{-1/2})$. Finally by assumption we have $\| \bar{\Psi}\|_{\mathcal{H}^\ast} = O_p(n^{-1/2})$ and thus $\widehat{G}(\bar{\theta}, \bar{\beta}) \leq -\epsilon \varphi^\ast(0) + \left(\| \bar{\Psi} \|_{\mathcal{H}^\ast} + o_p(n^{-1/2}) \right)) \| \bar{\beta} \|_\mathcal{M} - C \| \bar{\beta} \|_\mathcal{M}^2  = -\epsilon \varphi^\ast(0) + O_p(n^{-1})$.
\end{proof}
\begin{lemma}\label{lemma3}
Let the assumptions of Theorem~\ref{th:consistency-fmr} be satisfied and denote the KMM estimator as $\hat{\theta}= \argmin_{\theta \in \Theta}\sup_{\beta \in \mathcal{M}} \widehat{G}_{\epsilon, \lambda_n}({\theta}, \beta)$.
Then $\| E_{\hat{P}_n} [ \Psi(X,Z;\hat{\theta}) ] \|_{\mathcal{H}^\ast} = O_p\left(n^{-1/2}\right)$.
\end{lemma}
\begin{proof}
    Define $\hat{\Psi}_i := \Psi(x_i,z_i;\hat{\theta})$ and $\hat{\Psi} = \frac{1}{n} \sum_{i=1}^n \hat{\Psi}_i$. For simplicity of notation let $\widehat{G}(\theta, \beta) := \widehat{G}_{\epsilon, \lambda_n}(\theta, \beta)$.
    For any $\eta \in \mathbb{R}$ and $f \in \mathcal{F}$ consider the dual variable $\bar{\beta} = (\eta, f, \phi(\hat{\Psi}))$ and its normalized version $\bar{\beta}_\zeta = n^{-\zeta} \bar{\beta} / \| \bar{\beta}\|$, where $\phi(\hat{\Psi})$ denotes the Riesz representer of $\hat{\Psi} \in \mathcal{H}^\ast$ in $\mathcal{H}$ and $0< \zeta < 1/2$ as in Lemma~\ref{lemma1}. 
    Taylor expanding the KMM objective about $\beta=0$ again yields
    \begin{align}
        \widehat{G}(\hat{\theta}, \bar{\beta}_\zeta) 
        =& -\epsilon \varphi^\ast(0) - \frac{1}{2} \bar{\beta}_\zeta^T \Lambda_{n}(\Dot{\beta}, \hat{\theta})\bar{\beta}_\zeta + \frac{n^{-\zeta}}{\| \bar{\beta}\|} \intxz{\Psi(x,z;\hat{\theta})(\phi(\Psi))  } \\
        &+ \frac{n^{-\zeta}}{\| \bar{\beta}\|} \int_{\mathcal{X} \times \mathcal{Z}} f(x, z) \hat{P}_n(\mathrm{d}x \otimes \mathrm{d}z) - \frac{n^{-\zeta}}{\| \bar{\beta}\|} \intxz{ f(x,z)} \\
        = & -\epsilon \varphi^\ast(0) - \frac{1}{2} \bar{\beta}_\zeta^T \Lambda_{n}(\Dot{\beta}, \hat{\theta})\bar{\beta}_\zeta 
        +  \frac{n^{-\zeta}}{\| \bar{\beta}\|} \int_{\mathcal{X} \times \mathcal{Z}} \Psi(x,z;\hat{\theta})(\phi(\hat{\Psi})) \hat{P}_n(\mathrm{d}x \otimes \mathrm{d}z) \\
        &+ \frac{n^{-\zeta}}{\| \bar{\beta}\|} \intxz{\Psi(x,z;\hat{\theta})(\phi(\hat{\Psi})) }
        -  \frac{n^{-\zeta}}{\| \bar{\beta}\|} \int_{\mathcal{X} \times \mathcal{Z}} \Psi(x,z;\hat{\theta})(\phi(\hat{\Psi})) \hat{P}_n(\mathrm{d}x \otimes \mathrm{d}z) \\
        &+ \frac{n^{-\zeta}}{\| \bar{\beta}\|} \int_{\mathcal{X} \times \mathcal{Z}} f(x, z) \hat{P}_n(\mathrm{d}x \otimes \mathrm{d}z) - \frac{n^{-\zeta}}{\| \bar{\beta}\|} \intxz{ f(x,z)} \\
        \geq & -\epsilon \varphi^\ast(0) - \frac{1}{2} \bar{\beta}_\zeta^T \Lambda_{n}(\Dot{\beta}, \hat{\theta})\bar{\beta}_\zeta 
        +  \frac{n^{-\zeta}}{\| \bar{\beta}\|} \|\hat{\Psi} \|_{\mathcal{H}^\ast}^2 - \frac{n^{-\zeta}}{\| \bar{\beta}\|} \| f\|_\mathcal{F} \mmd(\hat{P}_n, \omega;\mathcal{F}) \\
        &- \frac{\alpha n^{-\zeta}}{\| \bar{\beta}\|} \left(  \| \hat{\Psi}\|_{\mathcal{H}^\ast} \int_{\mathcal{X} \times \mathcal{Z}} \| \Psi(x,z;\hat{\theta}) \|_{\mathcal{H}^\ast} Q(\mathrm{d}x \otimes \mathrm{d}z) + \|\hat{\Psi} \|^2_{\mathcal{H}^\ast}  \right) \\
        \geq & -\epsilon \varphi^\ast(0) - \frac{1}{2} \bar{\beta}_\zeta^T \Lambda_{n}(\Dot{\beta}, \hat{\theta})\bar{\beta}_\zeta 
        +  C_\psi n^{-\zeta} \|\hat{\Psi} \|_{\mathcal{H}^\ast} \\
        &- \alpha n^{-\zeta} C_\psi \left( C_Q + \| \hat{\Psi}\|_{\mathcal{H}^\ast} \right) - n^{-\zeta} C_f  \mmd(\hat{P}_n, \omega;\mathcal{F}),
    \end{align}
    where $C_\psi, C_f \in [0, 1]$ as $\| \hat{\Psi} \|_{\mathcal{H}^\ast}/ \| \bar{\beta} \|_\mathcal{M} \leq 1$ and $\| f \|_{\mathcal{F}}/ \| \bar{\beta} \|_\mathcal{M} \leq 1$ by definition of $\bar{\beta}$ and $\alpha = O_p(n^{-1})$ as well as $\mmd(\hat{P}_n, \omega;\mathcal{F}) = O_p(n^{-1})$ by Lemma~\ref{mmdlemma} and Assumption~f).
    Using Lemma~\ref{lemma:bounded-covariance} we can bound the largest eigenvalue of $\Lambda_{n}(\Dot{\beta}, \hat{\theta})$ by some positive constant $M$ which is independent of $n$, so we obtain
    \begin{align}
        \widehat{G}(\hat{\theta}, \bar{\beta}_\zeta) \geq -\epsilon \varphi^\ast(0) -  M n^{-2\zeta} + C_\psi n^{-\zeta} \| \hat{\Psi} \|_{\mathcal{H}^\ast} + O_p(n^{-1 -\zeta}),
    \end{align}
    Now as $(\hat{\theta}, \hat{\beta})$ is a saddle point of the \emph{empirical} KMM objective, we have $\widehat{G}(\hat{\theta}, \bar{\beta}_\zeta) \leq \widehat{G}(\hat{\theta}, \hat{\beta}) \leq \max_{\beta \in \mathcal{M}} \widehat{G}(\theta_0, \beta) $.
    Putting this together with the previous inequality we have
    \begin{align}
        -\epsilon \varphi^\ast(0) + C_\psi n^{-\zeta} \| \hat{\Psi} \|_{\mathcal{H}^\ast} -  M n^{-2\zeta} + O_p(n^{-1 - \zeta}) &\leq \widehat{G}(\hat{\theta}, \bar{\beta}_\zeta) \label{eq:20} \\
        &\leq \widehat{G}(\hat{\theta}, \hat{\beta}) \\
        &\leq \max_{\beta \in \mathcal{M}} \widehat{G}(\theta_0, \beta) \\
        &\leq -\epsilon \varphi^\ast(0) + O_p(n^{-1}),
    \end{align}
    where in the last line we used Lemma~\ref{lemma2} with $\bar{\theta}=\theta_0$ which fulfills the corresponding conditions as $\| E[\Psi(X,Z;\theta_0)] \|_{\mathcal{H}^\ast} =0 $ by definition and thus by the central limit theorem $\| E_{\hat{P}_n}[\Psi(X,Z;\theta_0)] \|_{\mathcal{H}^\ast} = O_p(n^{-1/2})$.
    Adding $\epsilon \varphi^\ast(0)$ on both sides and solving for $\| \hat{\Psi} \|_{\mathcal{H}^\ast}$ yields 
    \begin{align}
        \| \hat{\Psi} \|_{\mathcal{H}^\ast} \leq O_p(n^{-1 + \zeta}) + O_p(n^{-\zeta}) 
        = O_p(n^{-\zeta}), \label{eq:rate-zeta}
    \end{align}
    where the last step follows as $\zeta < 1/2$ by definition and thus $-1+\zeta < - \zeta$. 
    Equation~\eqref{eq:rate-zeta} provides an upper bound on the convergence rate for $\| \hat{\Psi} \|_{\mathcal{H}^\ast}$. To further refine this rate define $\tilde{\beta} := (0, 0, \phi(\hat{\Psi}))$ and for any sequence $\kappa_n \rightarrow 0$ consider $\kappa_n \tilde{\beta}$. Then as $\|\tilde{\beta} \|_\mathcal{M} = \|\hat{\Psi} \|_{\mathcal{H}^\ast} \leq O_p(n^{-\zeta})$ we immediately have $\| \kappa_n \tilde{\beta} \|_\mathcal{M} = o_p(n^{-\zeta})$ which implies $\kappa_n \tilde{\beta} \in \mathcal{M}_n$ w.p.a.1 and 
    \begin{align}
         -\epsilon \varphi^\ast(0) + \kappa_n \| \hat{\Psi} \|_{\mathcal{H}^\ast}^2 - M \kappa_n^2 \| \hat{\Psi} \|_{\mathcal{H}^\ast}^2 + O_p(n^{-1 -\zeta})
        & \leq \widehat{G}(\hat{\theta}, \kappa_n \tilde{\beta}) \\
        &\leq \widehat{G}(\hat{\theta}, \hat{\beta}) \\
        &\leq \max_{\beta \in \mathcal{M}} \widehat{G}(\theta_0, \beta) \\
        &\leq -\epsilon \varphi^\ast(0) + O_p(n^{-1}).
    \end{align}
    This implies 
    $(1- \kappa_n M) \kappa_n \| \hat{\Psi} \|_{\mathcal{H}^\ast}^2 \leq O_p(n^{-1})$ and as $(1- \kappa_n M)$ is bounded away from zero for all sufficiently large $n$, we get $\kappa_n \| \hat{\Psi} \|_{\mathcal{H}^\ast}^2 = O_p(n^{-1})$. As this holds for all $\kappa_n \rightarrow 0$, we finally obtain $\| \hat{\Psi} \|_{\mathcal{H}^\ast} = O_p(n^{-1/2})$.    
\end{proof}
\paragraph{Proof of Theorem~\ref{th:consistency-fmr}}
\begin{proof}
Define $\hat{\Psi}_i = \Psi(x_i,z_i;\hat{\theta})$ and $\hat{\Psi} = \frac{1}{n}\sum_{i=1}^n \hat{\Psi}_i$. As $\hat{\Psi}$ is the average of $n$ i.i.d. random variables $\hat{\Psi}_i$, by the central limit theorem and absolute homogeneity of the dual norm, we have $\| \hat{\Psi}(\theta) - E [\Psi(X,Z;\theta)] \|_{\mathcal{H}^\ast} = O_p(n^{-1/2})$ for any $\theta \in \Theta$. From Lemma~\ref{lemma3} we also have $\| \hat{\Psi} \|_{\mathcal{H}^\ast} = O_p(n^{-1/2})$ and thus using the triangle inequality we get
\begin{align}
    \left\| E[\Psi(X,Z;\hat{\theta})] \right\|_{\mathcal{H}^\ast} &= \left\| E[\Psi(X,Z;\hat{\theta})] -\hat{\Psi}  + \hat{\Psi} \right\|_{\mathcal{H}^\ast} \\
    & \leq \left\| E[\Psi(X,Z;\hat{\theta})] -\hat{\Psi}  \right\|_{\mathcal{H}^\ast} + \left\|  \hat{\Psi} \right\|_{\mathcal{H}^\ast} \\
    &= O_p(n^{-1/2}) \overset{p}{\rightarrow} 0.
\end{align}
As by assumption $\theta = \theta_0$ is the unique parameter for which $ \theta \mapsto \| E[\Psi(X,Z;\theta)] \|_{\mathcal{H}^\ast} = 0$ it follows that $\hat{\theta} \overset{p}{\rightarrow} \theta_0$.
To derive a convergence rate for $\hat{\theta}$ note that by the mean  value theorem, there exists $\bar{\theta} \in \operatorname{conv}(\{\theta_0, \hat{\theta} \})$ such that 
\begin{align}
    \Psi(X,Z;\hat{\theta}) = \Psi(X,Z;\theta_0) + (\hat{\theta} - \theta_0)^T \nabla_\theta \Psi(X,Z;\bar{\theta}).
\end{align}
Using this we have
\begin{align}
    \|E[\Psi(X,Z;\hat{\theta})] \|^2_{\mathcal{H}^\ast} & = \| \underbrace{E[\Psi(X,Z;\theta_0)]}_{=0} + (\hat{\theta} - \theta_0)^T E[ \nabla_\theta \Psi(X,Z;\bar{\theta})]\|^2_{\mathcal{H}^\ast} \\
    &= \left\langle (\hat{\theta} - \theta_0)^T E[ \nabla_\theta \Psi(X,Z;\bar{\theta})], (\hat{\theta} - \theta_0)^T E[ \nabla_\theta \Psi(X,Z;\bar{\theta})] \right\rangle_{\mathcal{H}^\ast} \\
    &= (\hat{\theta} - \theta_0)^T \underbrace{\left\langle  E[ \nabla_\theta \Psi(X,Z;\bar{\theta})], E[ \nabla_\theta \Psi(X,Z;\bar{\theta})] \right\rangle_{\mathcal{H}^\ast}}_{=: \Sigma(\bar{\theta})} (\hat{\theta} - \theta_0) \\
    &\geq \lambda_\text{min}\left(\Sigma(\bar{\theta})\right) \| \hat{\theta} - \theta_0 \|^2_2
\end{align}
Now as $\hat{\theta} \overset{p}{\rightarrow} \theta_0$ and $\bar{\theta} \in \operatorname{conv}(\{\theta_0, \hat{\theta} \})$ we have $\bar{\theta} \overset{p}{\rightarrow} \theta_0$ and thus $\Sigma(\bar{\theta}) \overset{p}{\rightarrow} \Sigma(\theta_0) =: \Sigma_0$ by the continuous mapping theorem. By the non-negativity of the norm $\Sigma_0$ is positive-semi definite and non-singular by Assumption~j), thus the smallest eigenvalue of $\Sigma(\bar{\theta})$, $\lambda_\text{min}(\Sigma(\bar{\theta}))$, is positive and bounded away from zero w.p.a.1. Finally as $\|E[\Psi(X,Z;\hat{\theta})] \| = O_p(n^{-1/2})$ taking the square-root on both sides we have $\|\hat{\theta} - \theta_0 \| = O_p(n^{-1/2 })$.
\end{proof}
\subsubsection{Proof of Theorem~\ref{th:asymptotic-normality}}
To show asymptotic normality we linearize the first order conditions for $(\hat{\theta}, \hat{\beta})$ about the true parameters $(\theta_0, 0)$ and solve for the KMM estimates. This involves the inversion of a blockmatrix for whose invertibility we require the following Lemma.
\begin{lemma}\label{lemma:schur-complement}
    For a moment functional $\Psi(X,Z;\theta): \mathcal{H} \rightarrow \mathbb{R}$, continuously differentiable in $\theta$, define the covariance operator $\Omega_0 := E[\Psi(X,Z;\theta) \otimes \Psi(X,Z;\theta)]$. Further define the matrix $\Sigma_0 := \langle E[\nabla_\theta \Psi(X,Z;\theta_0)], E[\nabla_\theta \Psi(X,Z;\theta_0)] \rangle_{\mathcal{H}^\ast} \in \mathbb{R}^{p \times p}$, where the inner product is only taken with respect to the $\mathcal{H}^\ast$ index. If $\Omega_0$ and $\Sigma_0$ are non-singular with smallest eigenvalue bounded away from zero, then the matrix
    \begin{align}
        \Gamma := - E[\nabla_\theta \Psi(X,Z;\theta_0)] \Omega_{0}^{-1} E[\nabla_\theta \Psi(X,Z;\theta_0)]
    \end{align}
    is non-singular with smallest eigenvalue bounded away from zero.
\end{lemma}
\begin{proof}
    Let $\mathcal{H} = \bigoplus_{i=1}^m \mathcal{H}_i$ and for $i=1,\ldots,m$ let $\{h^i_j\}_{j=1}^{\infty}$ denote a orthonormal basis of $\mathcal{H}_i^\ast$ such that $\langle h_i^k, h_j^l \rangle = \delta_{ij} \delta_{kl}$. Then we can write the identity operator in $\mathcal{H}^\ast$ as $I_{\mathcal{H}^\ast} = \sum_{i=1}^m \sum_{j=1}^\infty h_j^i \left(h_j^i\right)^\ast$, where $\left(h_j^i\right)^\ast$ is the Riesz representer of $h_j^i \in \mathcal{H}^\ast$ in $\mathcal{H}^{\ast\ast}$ which can be uniquely identified with an element in $\mathcal{H}$ by the property of Hilbert spaces.
    Further, let $\nabla_\theta \Psi_0 := E[\nabla_\theta \Psi(X,Z;\theta_0)]$.
    Then we can write for any $\theta \in \Theta$ with $0 < \|\theta \| < \infty$,
    \begin{align}
        - \theta^T \Gamma \theta &= \theta^T \nabla_\theta \Psi_0 \Omega_0^{-1} \nabla_\theta \Psi_0 \theta \\
        &= \theta^T \nabla_\theta \Psi_0 \sum_{i=1}^m \sum_{j=1}^\infty h_j^i \left(h_j^i\right)^\ast \Omega_0^{-1} \sum_{k=1}^m \sum_{l=1}^\infty h_l^k \left(h_l^k\right)^\ast \left(\nabla_\theta \Psi_0 \right)^T \theta \\
        &= \theta^T \nabla_\theta \Psi_0 \left(  \sum_{i,k=1}^m \sum_{j,l=1}^\infty  h^i_j  \ \langle h_j^i, \Omega_0^{-1} h_l^k \rangle_{\mathcal{H}^\ast} \  \left( h^k_l \right)^\ast \right) \left(\nabla_\theta \Psi_0\right)^T \theta\\
         & \geq \lambda_\textrm{min}(\Omega_0^{-1})\theta^T \nabla_\theta \Psi_0 \left(  \sum_{i=1}^m \sum_{j=1}^\infty  h^i_j  \left(h_j^i\right)^\ast \right) \left(\nabla_\theta \Psi_0\right)^T \theta\\
        &= \lambda_\textrm{min}(\Omega_0^{-1})\theta^T \langle \nabla_\theta \Psi_0, \nabla_\theta \Psi_0 \rangle_{\mathcal{H}^\ast} \theta \\
        &= \lambda_\textrm{min}(\Omega_0^{-1})\theta^T \Sigma_0 \theta \\
        &\geq \lambda_\textrm{min}(\Omega_0^{-1}) \lambda_\textrm{min}(\Sigma_0) \| \theta\|^2  > 0,
    \end{align}
    where we used that $\Omega_0$ is positive semi-definite by construction and non singular by Assumption~e) and thus being the inverse of a strictly positive definite operator the smallest eigenvalue $\lambda_\textrm{min}(\Omega_0^{-1})$ of $\Omega_0^{-1}$ is positive and bounded away from zero. Moreover, $\Sigma_0$ is positive semi-definite by construction and non singular by Assumptions~j) and therefore its smallest eigenvalue $\lambda_\textrm{min}(\Sigma_0)$ positive and bounded away from zero. From this it immediately follows that $\Gamma$ is strictly negative definite and thus non-singular.
\end{proof}
\paragraph{Proof of Theorem~\ref{th:asymptotic-normality}}
\begin{proof}
    The KMM estimator $\hat{\theta}$ and the optimal Lagrange parameter $\hat{\beta}$ are determined via the first order optimality conditions for $(\theta, \beta)$ which are given by
    \begin{align}
        0 = \frac{\partial \widehat{G}}{\partial \theta} (\hat{\theta}, \hat{\beta}) &= - \intxz{  \varphi_1^\ast \left( \frac{1}{\epsilon} a(x,z;\hat{\theta})^T \hat{\beta }\right) \nabla_\theta \left(a(x,z;\hat{\theta})^T \hat{\beta} \right) } \label{eq:1st} \\
        0 = \frac{\partial \widehat{G}}{\partial \beta} (\hat{\theta}, \hat{\beta}) &= \avg b_i - \intxz{  \varphi_1^\ast \left( \frac{1}{\epsilon} a(x,z;\hat{\theta})^T \hat{\beta } \right) a(x,z;\hat{\theta})} - R_{\lambda_n} \beta \label{eq:2nd}
    \end{align}
    Linearizing the first condition \eqref{eq:1st} about the true parameters $(\theta_0, 0)$ yields
    \begin{align}
        0 =& \frac{\partial \widehat{G}}{\partial \theta} (\theta_0, 0) + \left( \frac{\partial^2 \widehat{G}}{\partial \theta \ \partial \theta} (\bar{\theta}, \bar{\beta}) \right) \left( \hat{\theta} - \theta_0 \right) + \left( \frac{\partial^2 \widehat{G}}{\partial \theta \ \partial \beta} (\bar{\theta}, \bar{\beta}) \right) \hat{\beta} \\
        =& -\left(\intxz{\frac{1}{\epsilon} \varphi_2^\ast \left( \frac{1}{\epsilon} a(x,z;\bar{\theta})^T \bar{\beta }\right)  \left(\nabla_\theta a(x,z;\bar{\theta})^T \bar{\beta} \right)   \left( \nabla_{\theta} a(x,z;\bar{\theta})^T \bar{\beta} \right)^T  } \right) \left( \hat{\theta} - \theta_0 \right) \\
        &- \left( \intxz{  \varphi_1^\ast \left( \frac{1}{\epsilon} a(x,z;\bar{\theta})^T \bar{\beta }\right)\left(\nabla_\theta^2 a(x,z;\bar{\theta})^T \bar{\beta} \right) } \right) \left( \hat{\theta} - \theta_0 \right) \\
        &- \left( \intxz{  \frac{1}{\epsilon} \varphi_2^\ast \left( \frac{1}{\epsilon} a(x,z;\bar{\theta})^T \bar{\beta }\right)  \left(\nabla_\theta a(x,z;\bar{\theta})^T \bar{\beta} \right)  a(x,z;\bar{\theta})^T } \right) \hat{\beta} \\
        &- \left(\intxz{ \varphi_1^\ast \left( \frac{1}{\epsilon} a(x,z;\bar{\theta})^T \bar{\beta }\right) \nabla_\theta a(x,z;\bar{\theta})^T } \right) \hat{\beta}
    \end{align}
    for some $(\bar{\theta}, \bar{\beta})$ on the line between $(\hat{\theta}, \hat{\beta})$ and $(\theta_0, 0)$. Analogously the linearization of the second condition \eqref{eq:2nd} is given by
    \begin{align}
        0 =& \frac{\partial \widehat{G}}{\partial \beta} (\theta_0, 0) + \left( \frac{\partial^2 \widehat{G}}{\partial \theta \ \partial \beta} (\Dot{\theta}, \Dot{\beta}) \right) \left( \hat{\theta} - \theta_0 \right) + \left( \frac{\partial^2 \widehat{G}}{\partial \beta \ \partial \beta} (\Dot{\theta}, \Dot{\beta}) \right) \hat{\beta} \\
        =& - \avg b_i + \intxz{ a(x,z;\theta_0)} \\
        & - \left( \intxz{  \frac{1}{\epsilon} \varphi_2^\ast \left( \frac{1}{\epsilon} a(x,z;\Dot{\theta})^T \Dot{\beta }\right) a(x,z;\Dot{\theta})  \left(\nabla_\theta a(x,z;\Dot{\theta})^T \Dot{\beta} \right) } \right) \left(\hat{\theta} - \theta_0 \right)  \\
        & - \left( \intxz{ \varphi_1^\ast \left( \frac{1}{\epsilon} a(x,z;\Dot{\theta})^T \Dot{\beta }\right) \nabla_\theta a(x,z;\Dot{\theta}) } \right) \left(\hat{\theta} - \theta_0 \right) \\
        & - \left( \intxz{ \frac{1}{\epsilon} \varphi_2^\ast \left( \frac{1}{\epsilon} a_i(\Dot{\theta})^T \Dot{\beta }\right) a(x,z;\Dot{\theta}) a(x,z;\Dot{\theta})^T} + R_{\lambda_n} \right) \hat{\beta}
    \end{align}
    Now, as $\bar{\beta}, \Dot{\beta}$ are on the line between $\hat{\beta}$ and $0$ and $\hat{\beta} = O_p(n^{-1/2})$ by Lemma~\ref{lemma2}, we have that all derivative terms of $\widehat{G}$ involving $\bar{\beta}, \Dot{\beta}$ linearly are $O_p(n^{-1})$, as each term additionally gets multiplied by $\hat{\theta} - \theta$ or $\hat{\beta}$ which both are $O_p(n^{-1/2})$ in the respective norms. Further it follows from Lemma~\ref{lemma1} that $ \varphi_j(a(X,Z;\theta)^T \tilde{\beta})  \overset{p}{\rightarrow} 1 \ \asxz$ for any $\theta \in \Theta$, $\tilde{\beta} \in \{\bar{\beta}, \Dot{\beta} \}$ and $j=1,2$. Therefore, the first linearized first order condition \eqref{eq:1st} reduces to
    \begin{align}
        0 =& \left( \intxz{ \nabla_\theta a(x,z;\bar{\theta})^T } \right) \hat{\beta} + O_p(n^{-1}).
    \end{align}
    For the second condition note that by Lemma~\ref{lemma:omega-empirical}
    \begin{align}
        \intxz{ a(x,z;\theta_0)} &= \avg a(x_i,z_i;\theta_0) + O_p(n^{-1}).
    \end{align}
    Now inserting this into the second linearized first order condition we obtain
    \begin{align}
        0 =& 
        \avg \left( b_i - a(x_i,z_i;\theta_0) \right)  + \intxz{\nabla_\theta a(x,z;\Dot{\theta})^T} \left(\hat{\theta} - \theta_0\right) \\
        &+ \underbrace{\left(\intxz{ \frac{1}{\epsilon} a(x,z;\Dot{\theta}) a(x,z;\Dot{\theta})^T } 
        + R_{\lambda_n} \right)}_{=: \Lambda_n(\Dot{\theta})} \hat{\beta} + O_p(n^{-1}),
    \end{align}
    where $\avg (b_i - a(x_i,z_i;\theta_0)) = \left( 0, 0,\hat{\Psi}(\theta_0) \right)^T$.
    Writing the two linearized first order conditions in matrix-vector form we obtain
    \begin{align}
        \begin{pmatrix}
            0 \\ 0
        \end{pmatrix}
        =& 
        \begin{pmatrix}
            0 \\ \avg \left( b_i - a(x_i,z_i;\theta_0) \right)
        \end{pmatrix}
        \\ 
        & +  
        \underbrace{
        \begin{pmatrix}
            0 & \intxz{ \nabla_\theta a(x,z;\bar{\theta})^T } \\
            \intxz{ \nabla_\theta a(x,z;\Dot{\theta})^T} & \Lambda_n(\Dot{\theta})
        \end{pmatrix}}_{:=M_n}
        \begin{pmatrix}
            \hat{\theta} - \theta_0 \\
            \hat{\beta} - \beta_0
        \end{pmatrix},
        \label{eq:first-order}
    \end{align}
    where $\beta_0 = 0$.
    Now $\hat{\theta} \overset{p}{\rightarrow} \theta_0$ and $\Dot{\theta}$ and $\bar{\theta}$ are on the line between $\hat{\theta}$ and $\theta$, and $\omega \overset{p}{\rightarrow} P_0$ weakly by definition. Moreover, $\nabla_\theta \Psi(X,Z;\bar{\theta})$ is continuous for any $\bar{\theta}$ in a neighborhood $\bar{\Theta}$ of $\theta_0$ by Assumption~i) and thus by the continuous mapping theorem we have
    $\intxz{ \nabla_\theta a(x,z;\bar{\theta}) } \overset{p}{\rightarrow} E[\nabla_\theta a(X,Z;\theta_0)] =:\nabla_\theta a_0$ and the same holds for the other off-diagonal entry. In addition, the off-diagonal entries are bounded as Assumption~i) with Lemma~\ref{lemma:d2} implies $E[\sup_{\theta \in \bar{\Theta}} \| \nabla_\theta \Psi(X,Z;\theta)\|] < \infty$. Finally, by the uniform weak law of large numbers and the continuous mapping theorem we have ${\Lambda}_n(\Dot{\theta}) \overset{p}{\rightarrow} {\Lambda}(\theta_0) = E[a(X,Z;\theta_0) a(X,Z;\theta_0)^T] + R_{\lambda=0} =: \Lambda$. Let correspondingly $M$ denote the limit operator for $M_n$, i.e., $M_n \overset{p}{\rightarrow} M$.
    From Lemma~\ref{lemma:bounded-covariance} it follows that the smallest eigenvalue of $\Lambda$ is bounded away from zero and thus it is invertible. Now suppose the Schur complement of $\Lambda$ in $M$,  
    \begin{align}
        \Gamma := M / \Lambda= - \left( \nabla_\theta a_0^T \right) \Lambda^{-1} \left( \nabla_\theta a_0^T \right) 
    = - \nabla_\theta \Psi_0 \left( \Lambda^{-1} \right)_{3,3} \nabla_\theta \Psi_0
    \end{align}
    is invertible, then it follows from standard blockmatrix algebra (see e.g.\ \citet{bernstein2009matrix}), that the inverse of $M$ is given by
    \begin{align}
        M^{-1} = 
        \begin{pmatrix}
             \Gamma^{-1} &  \Gamma^{-1} \left( \nabla_\theta a_0
             \right) \Lambda^{-1} \\
             \Lambda^{-1} \left( \nabla_\theta a_0 \right) \Gamma^{-1} &  \Lambda^{-1} +  \Lambda^{-1} \left( \nabla_\theta a_0 \right) \Gamma^{-1}  \left( \nabla_\theta a_0 \right)  \Lambda^{-1}
        \end{pmatrix}.
    \end{align}
    Now it remains to find an explicit expression for $(\Lambda^{-1})_{3,3}$ and to show that $\Gamma$ is indeed invertible.
    To this aim we write out the outer product over $\mathcal{M} \times \mathcal{M}$ in $\Lambda$ which yields
    \begin{align}
        \Lambda =& \frac{1}{\epsilon} 
    \begin{pmatrix}
        1 & E[1 \otimes k((X,Z),\cdot)] & - E[1 \otimes \Psi(X,Z;\theta_0)] \\
        E[k((X,Z),\cdot) \otimes 1] & E[k((X,Z),\cdot) \otimes k((X,Z),\cdot)] + I & - E[k((X,Z),\cdot) \otimes \Psi(X,Z;\theta_0)] \\
        - E[\Psi(X,Z;\theta_0) \otimes 1] & - E[\Psi(X,Z;\theta_0) \otimes k((X,Z),\cdot)] & E[\Psi(X,Z;\theta_0) \otimes \Psi(X,Z;\theta_0)]
    \end{pmatrix} \\
    =& \frac{1}{\epsilon} 
    \begin{pmatrix}
        1 & E[1 \otimes k((X,Z),\cdot)] & 0 \\
        E[k((X,Z),\cdot) \otimes 1] & E[k((X,Z),\cdot) \otimes k((X,Z),\cdot)] + I & 0 \\
        0 & 0 & \Omega_0
    \end{pmatrix}
    \end{align}
    where we used that $\| E[\Psi(X,Z;\theta_0)] \|_{\mathcal{H}^\ast} = 0$ by definition and as $\mathcal{F}$ is an RKHS, the evaluation functional $k((x,z), \cdot)$ can be bounded by some constant $C$ and thus we have $\| E[k((X,Z), \cdot) \otimes \Psi(X,Z;\theta_0)] \|_{\mathcal{F} \times \mathcal{H}^\ast} \leq C \| E[\Psi(X,Z;\theta_0)]\|_{\mathcal{H}^\ast} =0$.
    Now $\Lambda$ is of blockdiagonal form and for the upper block $B$ we have
    \begin{align}
        B = \begin{pmatrix}
            1 & E[1 \otimes k((X,Z), \cdot)] \\
            E[k(X,Z),\cdot) \otimes 1] & E[k(X,Z),\cdot) \otimes k((X,Z),\cdot) ]
        \end{pmatrix} 
        + 
        \begin{pmatrix}
            1 & 0 \\
            0 & I
        \end{pmatrix}.
    \end{align}
    The first term is symmetric and thus positive semi-definite and the second term diagonal with positive entries, thus B is a strictly positive definite operator and thus invertible.
    Moreover, 
    by Assumption~e) of Theorem~\ref{th:consistency-fmr} $\Omega_0$ is invertible and thus we can conclude
    \begin{align}
        \Lambda^{-1} = \begin{pmatrix}
            B^{-1} & 0 \\
            0 & \Omega_0^{-1}
        \end{pmatrix}
    \end{align}
    and $\left( \Lambda^{-1}\right)_{3,3} = \Omega_0^{-1}$.
    Now, from invertibility of $\Omega_0$ we directly obtain that $\Omega_0^{-1}$ is non-singular and thus by Assumption~e) and Lemma~\ref{lemma:schur-complement} we have that $\Gamma$ is non-singular and invertible which legitimates the inversion of $M$.
    With this at hand, we can solve equation \eqref{eq:first-order} for $\hat{\theta} - \theta_0$ and obtain
    \begin{align}
        \sqrt{n} \left(\hat{\theta} - \theta_0 \right) &= \left( \Gamma^{-1} \nabla_\theta \Psi_0 \left( \Lambda^{-1} \right)_{3,3} \right) \sqrt{n} \hat{\Psi}(\theta_0) \\ 
        &= - \left( \left( \nabla_\theta \Psi_0 \left( \Lambda^{-1} \right)_{3,3} \nabla_\theta \Psi_0 \right)^{-1}  \nabla_\theta \Psi_0 \left( \Lambda^{-1}\right)_{3,3} \right) \sqrt{n} \hat{\Psi}(\theta_0). \label{eq:asymp}
    \end{align}
    By the central limit theorem we have 
    $\sqrt{n} E_{\hat{P}_n}[\Psi(X,Z;\theta_0)] \sim \mathcal{N}(0, \Omega_0)$ where as before $\Omega_0 = E[\Psi(X,Z;\theta_0) \otimes \Psi(X,Z;\theta_0)]$ and thus inserting into equation~\eqref{eq:asymp} we get
    \begin{align}
        \sqrt{n} \left(\hat{\theta} - \theta_0 \right) &=  - \left( \left( \left( \nabla_\theta \Psi_0 \right) \Omega_0^{-1} \left(\nabla_\theta \Psi_0 \right) \right)^{-1} \left( \nabla_\theta \Psi_0 \right)  \Omega_0^{-1} \right) \sqrt{n} \hat{\Psi}(\theta_0) \sim \mathcal{N}(0, \Xi)
    \end{align}
    with
    \begin{align}
        \Xi &= \left( \left( \left( \nabla_\theta \Psi_0 \right) \Omega_0^{-1} \left(\nabla_\theta \Psi_0 \right) \right)^{-1} \left( \nabla_\theta \Psi_0 \right)  \Omega_0^{-1} \right) \Omega_0 \left( \left( \left( \nabla_\theta \Psi_0 \right) \Omega_0^{-1} \left(\nabla_\theta \Psi_0 \right) \right)^{-1} \left( \nabla_\theta \Psi_0 \right)  \Omega_0^{-1} \right)^T \\
        &= \left( \left( \nabla_\theta \Psi_0 \right) \Omega_0^{-1} \left(\nabla_\theta \Psi_0 \right) \right)^{-1}.
    \end{align}
\end{proof}
\newpage
\subsection{Asymptotic Properties of KMM for Finite-Dimensional Moment Restrictions}
\subsubsection{Proof of Theorem~\ref{th:consistency-mr} (Consistency for MR)}
The consistency for the finite dimensional case follows as a special case of the consistency result for the functional case (Theorem~\ref{th:consistency-fmr}) by identifying $\mathcal{H} = \mathbb{R}^m$, $\Psi(x,z;\theta) = \psi(x;\theta) \in \mathbb{R}^m$ and $\lambda_n = 0$. 
For a finite dimensional version of Theorem~\ref{th:consistency-fmr} we need finite dimensional versions of Lemmas~\ref{lemma1}-\ref{lemma3}, which we will state in the following and describe the differences in the proofs compared to the functional case. Refer to the proof of Theorem~\ref{th:consistency-fmr} for details.
\begin{lemma}\label{lemma:d2-finite}
Let $\mathcal{A}$ denote a $\sigma$-algebra on $\mathcal{X}$ and let ($\mathcal{X}, \mathcal{A}, \omega)$ be a probability space with measure $\omega$. For any function $\psi: \mathcal{X} \times \Theta \rightarrow \mathbb{R}^m$ with $\intx{ \sup _{\theta \in \Theta}\|\psi(x; \theta)\|_2^2 \ } <\infty$,
it follows that $\sup _{\theta \in \Theta}\left\|\psi\left(X; \theta\right)\right\|_2 \leq C \asxz$ for some constant $C < \infty$.
\end{lemma}
\begin{proof}
    The proof follows immediately from the one for Lemma~\ref{lemma:d2} by exchanging $\Psi(X,Z;\theta) \rightarrow \psi(X;\theta)$ and $\mathcal{H} \rightarrow \mathbb{R}^m$.
\end{proof}
\begin{lemma}\label{lemma1-finite}
Let the assumptions of Theorem~\ref{th:consistency-mr} be satisfied, then for any $\zeta$ with $0 < \zeta < 1/2$ define the magnitude constrained set of dual variables $\mathcal{M}_n = \{\beta =(\eta, f, h) \in \mathcal{M}: \| \beta \|_{\mathcal{M}} \leq n^{-\zeta} \}$. Then as $n \rightarrow \infty$,
\begin{align}
    &\sup_{\theta \in \Theta, \ (\eta, f, h) \in \mathcal{M}_n} |\psi(X;\theta)^T h| = O_p(n^{-\zeta}) \asxz, \\
    & \sup_{\theta \in \Theta, \ \beta \in \mathcal{M}_n} |a(X;\theta)^T \beta| = O_p(n^{-\zeta})  \asxz.
\end{align}
\end{lemma}
\begin{proof}
    The proof follows immediately from the one for Lemma~\ref{lemma1} by exchanging $\Psi(X,Z;\theta) \rightarrow \psi(X;\theta)$ and $\mathcal{H} \rightarrow \mathbb{R}^m$ and using Lemma~\ref{lemma:d2-finite} instead of Lemma~\ref{lemma:d2} to bound $\sup _{\theta \in \Theta}\left\|\psi\left(X; \theta\right)\right\|_2 \leq C \asxz$.
\end{proof}
\begin{lemma}\label{lemma:bounded-covariance-finite}
    Let the assumptions of Theorem~\ref{th:consistency-mr} be satisfied and consider $\bar{\theta} \in \Theta$ such that $\bar{\theta} \overset{p}{\rightarrow} \theta_{0}$. Further let $\beta_\zeta := \argmax_{\beta \in \mathcal{M}_n} \widehat{G}(\bar{\theta}, \beta)$, where $\mathcal{M}_n = \{\beta \in \mathcal{M}: \| \beta \|_\mathcal{M} \leq n^{-\zeta} \}$ with $0 < \zeta < 1/2$.
Define the operator $\Lambda_{n}(\beta,\theta): \mathcal{M} \rightarrow \mathcal{M}$ as
\begin{align}
    \Lambda_{n}(\beta, \theta) := \intx{ \frac{1}{\epsilon} a(x;\theta) a(x;\theta)^T \varphi^\ast_2 \left( \frac{1}{\epsilon} a(x;\theta)^T \beta \right)}  +  R. \label{eq:covariance-finite}
\end{align}
Then w.p.a.1 for any $\bar{\beta} \in \operatorname{conv}(\{0, \beta_\zeta\})$, $\Lambda_{n}(\bar{\beta}, \bar{\theta})$ is strictly positive definite and its smallest eigenvalue is bounded away from zero. Moreover, for any $\theta \in \Theta$ the largest eigenvalue of $\Lambda_n(\bar{\beta}, \theta)$ is bounded from above by a positive constant $M$.
\end{lemma}
\begin{proof}
    The proof follows from the one for the functional case Lemma~\ref{lemma:bounded-covariance} with the difference that we directly impose non-singularity of the covariance matrix $\Omega_0 = E[\psi(X;\theta_0) \psi(X;\theta_0)^T]$ by Assumption~e) of Theorem~\ref{th:consistency-mr}. 
\end{proof}
\begin{lemma}\label{lemma2-finite}
  Let the assumptions of Theorem~\ref{th:consistency-mr} be satisfied. Additionally let $\bar{\theta} \in \Theta$, $\bar{\theta} \overset{p}{\rightarrow} \theta_{0}$, and $\|E_{\hat{P}_n}[\psi(X;\bar{\theta})] \|_2 = O_{p}\left(n^{-1 / 2}\right)$.  
  Then for $\bar{\beta}=$ $\argmax_{\beta \in \mathcal{M}} \widehat{G}_{\epsilon, \lambda_n}(\bar{\theta}, \beta)$ we have $\| \bar{\beta}\|_\mathcal{M}=O_{p}\left(n^{-1/2}\right)$, and $\widehat{G}_{\epsilon, \lambda_n}(\bar{\theta},\bar{\beta}) \leq -\epsilon \varphi^\ast(0) + O_p\left(n^{-1}\right)$.
\end{lemma}
\begin{proof}
    The proof follows immediately from the one for the functional case (Lemma~\ref{lemma2}) with the usual substitutions.
\end{proof}
\begin{lemma}\label{lemma3-finite}
Let the assumptions of Theorem~\ref{th:consistency-mr} be satisfied and denote the KMM estimator as $\hat{\theta}= \argmin_{\theta \in \Theta}\sup_{\beta \in \mathcal{M}} \widehat{G}_{\epsilon, \lambda_n}({\theta}, \beta)$.
Then $\| E_{\hat{P}_n} [ \psi(X;\hat{\theta}) ] \|_{\mathcal{H}^\ast} = O_p\left(n^{-1/2}\right)$.
\end{lemma}
\begin{proof}
    The proof follows immediately from the one for the functional case (Lemma~\ref{lemma3}) with the usual substitutions.
\end{proof}
\begin{lemma} \label{lemma:sigma-non-singular-finite}
    Consider a moment function $\psi: \mathcal{X} \times \Theta \rightarrow \mathbb{R}^m$ with $\nabla_\theta \psi(x;\theta) \in \mathbb{R}^{p\times m}$. Then if $\operatorname{rank}\left(E[\nabla_\theta \psi(X;\theta_0)] \right)=p$ the matrix $\Sigma_0 = E[\nabla_\theta \psi(X;\theta_0)] E[\nabla_\theta \psi(X;\theta_0)]^T \in \mathbb{R}^{p\times p}$ is non-singular with smallest eigenvalue positive and bounded away from zero.
\end{lemma}
\begin{proof} 
    As $\operatorname{rank}(E[\nabla_\theta \psi(X;\theta_0)] = p$, its rows are linearly independent and thus for any $\theta \in \Theta$ with $0<\| \theta\|<\infty$ there exists $j \in \{1, \ldots, m \}$ such that $\theta^T  E[\nabla_\theta \psi_j(X;\theta_0)] \neq 0$. This means that $\theta^T \Sigma_0 \theta = \sum_{i=1}^m \left( \theta^T  E[\nabla_\theta \psi_i(X;\theta_0)] \right)^2 \geq \left( \theta^T  E[\nabla_\theta \psi_j(X;\theta_0)] \right)^2 > 0$, which follows as any term in the sum is non-negative and there exists at least one term of index $j$ which is positive. Therefore, the smallest eigenvalue of $\Sigma_0$ is positive and bounded away from zero.
\end{proof}
\paragraph{Proof of Theorem~\ref{th:consistency-mr}}
\begin{proof}
Define $\hat{\psi}_i = \psi(x_i;\hat{\theta})$ and $\hat{\psi} = \frac{1}{n}\sum_{i=1}^n \hat{\psi}_i$. As $\hat{\psi}$ is the average of $n$ i.i.d. random variables $\hat{\psi}_i$, by the central limit theorem and absolute homogeneity of the dual norm, we have $\| \hat{\psi}(\theta) - E [\psi(X;\theta)] \|_2 = O_p(n^{-1/2})$ for any $\theta \in \Theta$. From Lemma~\ref{lemma3-finite} we also have $\| \hat{\psi} \| = O_p(n^{-1/2})$ and thus using the triangle inequality we get
\begin{align}
    \left\| E[\psi(X;\hat{\theta})] \right\|_2 &= \left\| E[\psi(\hat{\theta})] -\hat{\psi}  + \hat{\psi} \right\|_2 \\
    & \leq \left\| E[\psi(X;\hat{\theta})] -\hat{\psi}  \right\|_2 + \left\|  \hat{\psi} \right\|_2 \\
    &= O_p(n^{-1/2}) \overset{p}{\rightarrow} 0.
\end{align}
As by assumption $\theta_0$ is the unique parameter for which $ E[\psi(X;\theta)] = 0$ it follows that $\hat{\theta} \overset{p}{\rightarrow} \theta_0$.
To derive a convergence rate for $\hat{\theta}$ note that by the mean  value theorem, there exists $\bar{\theta} \in \operatorname{conv}(\{\theta_0, \hat{\theta} \})$ such that 
\begin{align}
    \psi(X;\hat{\theta}) = \psi(X;\theta_0) + (\hat{\theta} - \theta_0)^T \nabla_\theta \psi(X;\bar{\theta}),
\end{align}
where $\nabla_\theta \psi(x;\theta) \in \mathbb{R}^{p \times m}$.
Using this we have
\begin{align}
    \|E[\psi(X;\hat{\theta})] \|^2_2 & = \| \underbrace{E[\psi(X;\theta_0)]}_{=0} + (\hat{\theta} - \theta_0)^T E[ \nabla_\theta \psi(X;\bar{\theta})]\|^2_2 \\
    &= \left\langle (\hat{\theta} - \theta_0)^T E[ \nabla_\theta \psi(X;\bar{\theta})], (\hat{\theta} - \theta_0)^T E[ \nabla_\theta \psi(X;\bar{\theta})] \right\rangle \\
    &= (\hat{\theta} - \theta_0)^T \underbrace{E[ \nabla_\theta \psi(X;\bar{\theta})]  E[ \nabla_\theta \psi(X;\bar{\theta})]^T }_{=: \Sigma(\bar{\theta})} (\hat{\theta} - \theta_0) \\
    &\geq \lambda_\text{min}\left(\Sigma(\bar{\theta})\right) \| \hat{\theta} - \theta_0 \|^2_2
\end{align}
Now as $\hat{\theta} \overset{p}{\rightarrow} \theta_0$ and $\bar{\theta} \in \operatorname{conv}(\{\theta_0, \hat{\theta} \})$ we have $\bar{\theta} \overset{p}{\rightarrow} \theta_0$ and thus $\Sigma(\bar{\theta}) \overset{p}{\rightarrow} \Sigma(\theta_0) =: \Sigma_0$ by the continuous mapping theorem. 
Further by Assumption~i) of Theorem~\ref{th:consistency-mr} and Lemma~\ref{lemma:sigma-non-singular-finite} it follows that $\Sigma_0$ is positive definite and thus as $\Sigma(\bar{\theta}) \overset{p}{\rightarrow} \Sigma_0$ it follows that $\Sigma(\bar{\theta})$ is positive definite with smallest eigenvalue $\lambda_{\text{min}}\left(\Sigma(\bar{\theta})\right)$ positive and bounded away from zero w.p.a.1.
Finally as $\|E[\psi(X;\hat{\theta})] \| = O_p(n^{-1/2})$ taking the square-root on both sides we have $\|\hat{\theta} - \theta_0 \| = O_p(n^{-1/2})$.
\end{proof}
\newpage
\subsubsection{Proof of Theorem~\ref{th:asymptotic-normality-mr} (Asymptotic Normality for MR)}
\paragraph{Proof of Theorem~\ref{th:asymptotic-normality-mr}}
\begin{proof}
    The proof follows directly from the one for functional moment restrictions Theorem~\ref{th:asymptotic-normality} by identifying $\Psi(x,z;\theta) = \psi(x;\theta)$, $\mathcal{H} = \mathbb{R}^m$, setting $\lambda_n =0$ and using Lemma~\ref{lemma:sigma-non-singular-finite} to translate the rank condition, Assumption~i), into non-singularity of $\Sigma_0 = E[\nabla_\theta \psi(X;\theta_0)] E[\nabla_\theta \psi(X;\theta_0)]^T \in \mathbb{R}^{p \times p}$.
\end{proof}
\end{document}